\documentclass[letterpaper,11pt]{article}
\usepackage[margin=1in]{geometry}
\usepackage[bookmarks, colorlinks=true, plainpages = false, citecolor = blue,linkcolor=red,urlcolor = blue, filecolor = blue,pagebackref]{hyperref}
\usepackage{url}
\usepackage{amsmath,amsfonts,amsthm,amssymb,bm, verbatim,dsfont,mathtools}
\usepackage{color,graphicx,appendix}
\usepackage{subfigure}
\usepackage{etoolbox}

\usepackage{tikz}
\usepackage{amsmath}
\usetikzlibrary{arrows,calc}

\usepackage{pgfplots}
\pgfplotsset{
    standard/.style={
        axis x line=middle,
        axis y line=middle,
        every axis x label/.style={at={(current axis.right of origin)},anchor=west},
        every axis y label/.style={at={(current axis.above origin)},anchor=south}
    }
}

\usepackage{xr,xspace}
\usepackage{todonotes}
\usepackage{paralist}
\usepackage{caption,soul}
\usepackage{algorithm}% http://ctan.org/pkg/algorithms
\usepackage{algorithmic}% http://ctan.org/pkg/algorithms
\makeatletter
\theoremstyle{plain}
\newtheorem{theorem}{Theorem}
\newtheorem{lemma}{Lemma}

\newtheorem{corollary}{Corollary}
\theoremstyle{definition}
\newtheorem{definition}{Definition}

\newtheorem{remark}{Remark}
\newtheorem*{remark*}{Remark}

\newcommand{\argmax}{\mathop{\arg\max}}

\newcommand{\dKS}{d_{\rm KS}}

\usepackage{xspace,prettyref}

\newcommand{\diverge}{\to\infty}

\newcommand{\iiddistr}{{\stackrel{\text{\iid}}{\sim}}}
\newcommand{\ones}{\mathbf 1}

\newcommand{\reals}{{\mathbb{R}}}

\newcommand{\naturals}{{\mathbb{N}}}

\newcommand{\eexp}{e}

\newcommand{\identity}{\mathbf I}

\newcommand{\diff}{{\rm d}}

\newcommand{\Expect}{\mathbb{E}}
\newcommand{\expect}[1]{\mathbb{E}\left[ #1 \right]}
\newcommand{\eexpect}[1]{\mathbb{E}[ #1 ]}

\newcommand{\Prob}{\mathbb{P}}
\newcommand{\pprob}[1]{ \mathbb{P}\{ #1 \} }
\newcommand{\prob}[1]{ \mathbb{P}\left\{ #1 \right\} }

\newcommand{\var}{\mathsf{var}}

\newcommand{\Binom}{{\rm Binom}}

\newcommand{\eg}{e.g.\xspace}
\newcommand{\ie}{i.e.\xspace}
\newcommand{\iid}{i.i.d.\xspace}
\newrefformat{eq}{(\ref{#1})}
\newrefformat{chap}{Chapter~\ref{#1}}
\newrefformat{sec}{Section~\ref{#1}}
\newrefformat{alg}{Algorithm~\ref{#1}}
\newrefformat{fig}{Fig.~\ref{#1}}
\newrefformat{tab}{Table~\ref{#1}}
\newrefformat{rmk}{Remark~\ref{#1}}
\newrefformat{clm}{Claim~\ref{#1}}
\newrefformat{def}{Definition~\ref{#1}}
\newrefformat{cor}{Corollary~\ref{#1}}
\newrefformat{lmm}{Lemma~\ref{#1}}
\newrefformat{prop}{Proposition~\ref{#1}}
\newrefformat{app}{Appendix~\ref{#1}}
\newrefformat{hyp}{Hypothesis~\ref{#1}}
\newrefformat{thm}{Theorem~\ref{#1}}

\newcommand{\pth}[1]{\left( #1 \right)}

\newcommand{\sth}[1]{\left\{ #1 \right\}}

\newcommand{\iprod}[2]{\left \langle #1, #2 \right\rangle}

\newcommand{\indc}[1]{{\mathbf{1}_{\left\{{#1}\right\}}}}
\newcommand{\Indc}{\mathbf{1}}

\newcommand{\calE}{{\mathcal{E}}}
\newcommand{\calF}{{\mathcal{F}}}
\newcommand{\calG}{{\mathcal{G}}}

\newcommand{\calN}{{\mathcal{N}}}

\newcommand{\calT}{{\mathcal{T}}}

\renewcommand{\hat}{\widehat}
\renewcommand{\tilde}{\widetilde}
\renewcommand{\check}[1]{#1'}
\begin{document}

%\title{On Message Passing for Submatrix Localization}
%\title{Finding a submatrix in Gaussian noise via message passing}
\title{Submatrix localization via message passing}
\date{\today}
\author{
Bruce Hajek \and Yihong Wu \and Jiaming Xu\thanks{
%This research was supported by the National Science Foundation under
%Grant ECCS 10-28464, IIS-1447879, and CCF-1423088, and
%Strategic Research
%Initiative on Big-Data Analytics of the College of Engineering
%at the University of Illinois, and DOD ONR Grant N00014-14-1-0823, and Grant 328025 from the Simons Foundation.
B. Hajek and Y. Wu are with
the Department of ECE and Coordinated Science Lab, University of Illinois at Urbana-Champaign, Urbana, IL, \texttt{\{b-hajek,yihongwu\}@illinois.edu}.
J. Xu is with Department of Statistics, The Wharton School, University of Pennsylvania, Philadelphia, PA, \texttt{jiamingx@wharton.upenn.edu}.
}}

\maketitle

The principal submatrix localization problem deals with recovering a $K\times K$
principal  submatrix of elevated mean $\mu$  in a large $n\times n$  symmetric matrix
subject to additive standard Gaussian noise.   This problem serves as a  prototypical example for community detection, in which the community corresponds to the support of the submatrix.   The main result of this paper is that in the regime
$\Omega(\sqrt{n}) \leq K \leq o(n)$, the support of the submatrix
can be weakly recovered (with $o(K)$ misclassification errors on
average) by an optimized message passing algorithm if
$\lambda = \mu^2K^2/n$, the signal-to-noise ratio, exceeds $1/e$.   This extends a result by
Deshpande and Montanari previously obtained for $K=\Theta(\sqrt{n}).$
In addition, the algorithm can be extended to provide exact recovery whenever information-theoretically possible and achieve the information limit of exact recovery as long as $K \geq \frac{n}{\log n} (\frac{1}{8e} + o(1))$.
The total running time of the algorithm is $O(n^2\log n)$.

Another version of the submatrix localization problem, known as noisy biclustering,  aims to recover a $K_1\times K_2$ submatrix  of elevated mean $\mu$ in a large
$n_1\times n_2$ Gaussian matrix. The optimized message passing algorithm and its analysis are adapted to the bicluster
problem assuming  $\Omega(\sqrt{n_i}) \leq K_i \leq o(n_i)$ and $K_1\asymp K_2.$
A sharp information-theoretic condition for the weak recovery of both clusters is also identified.

\section{Introduction}
The problem of \emph{submatrix detection} and \emph{localization}, also known as \emph{noisy biclustering} \cite{Hartigan72,shabalin2009submatrix,kolar2011submatrix,butucea2013,Butucea2013sharp,ma2013submatrix,ChenXu14,CLR15},
 deals  with
finding a submatrix with an elevated mean in a large noisy matrix, which arises in
many applications such as social network analysis and gene expression data analysis.
A widely studied statistical model is the following:
\begin{equation}
	W = \mu \Indc_{C_1^*} \Indc_{C_2^*}^\top + Z,
	\label{eq:model}
\end{equation}
where $\mu > 0$, $\Indc_{C_1^*}$ and $\Indc_{C_2^*}$ are indicator vectors of the row and column support sets $C_1^* \subset [n_1]$  and $C_2^* \subset [n_2]$ of cardinality $K_1$ and $K_2$, respectively, and $Z$ is an $n_1\times n_2$ matrix consisting of independent standard normal entries.
The objective is to accurately locate the submatrix by estimating the row and column support based on the large matrix $W$.

For simplicity we start by considering the symmetric version of this problem, namely, locating a principal submatrix,
and later extend our theoretic and algorithmic findings to the asymmetric case. To this end, consider
\begin{equation}
	W = \mu \Indc_{C^*} \Indc_{C^*}^\top + Z,
	\label{eq:models}
\end{equation}
where $C^* \subset [n]$ has cardinality $K$ and $Z$ is an $n\times n$ symmetric matrix with $\{Z_{ij}\}_{1 \leq i \leq j \le n}$ being independent standard normal.
Given the data matrix $W$, the problem of interest is to recover $C^*$.
This problem has been investigated in \cite{Deshpande12,MontanariReichmanZeitouni14,HajekWuXu_one_info_lim15}
as a prototypical example of the {\em hidden community problem},\footnote{A slight variation of the model in \cite{Deshpande12,HajekWuXu_one_info_lim15} is
 that the data matrix therein is assumed to have zero diagonal. As shown in
\cite{HajekWuXu_one_info_lim15}, the absence of the diagonal has no impact on the statistical
limit of the problem as long as $K\diverge$, which is the case considered in the present paper.}
because the distribution of the entries
exhibits a community structure, namely,
$W_{i,j}\sim \calN(\mu, 1)$ if both $i$ and $j$ belong to $C^*$
and $W_{i,j}\sim \calN(0, 1)$ if otherwise.

Assuming that
%the model parameters $(n,K, \mu)$ are known to the  estimator and
$C^*$ is drawn from all subsets of $[n]$ of cardinality $K$ uniformly at random,
we focus on the following two types of recovery guarantees.\footnote{Exact and weak recovery are
called strong consistency and weak consistency in \cite{Mossel14}, respectively.}
Let $\xi = \Indc_{C^*} \in  \{0,1\}^n$ denote the indicator of the community.
% such that $\supp(\xi)=C^*$.
Let $\hat\xi=\hat\xi(A) \in \{0,1\}^n$ be  an estimator.
\begin{itemize}
	\item We say that $\hat \xi$  {\em exactly recovers} $\xi,$ if,
as $n \to \infty$, $\Prob[\xi \neq \hat\xi] \to 0$.
	\item We say that $\hat \xi$ {\em weakly recovers} $\xi$ if, as
$n \to \infty$,  $\ d(\xi, \hat\xi) / K \to 0$ in probability,  where $d$ denotes the Hamming distance.
\end{itemize}
%\begin{definition}[Exact Recovery]
%\label{def:exact}
%We say $\hat \xi$ {\em exactly recovers} $\xi,$ if,
%as $n \to \infty$,
%$
%\Prob[\xi \neq \hat\xi] \to 0,
%$
%where the probability is with respect to the randomness of $\xi$ and $A$.
%\end{definition}
%
%\begin{definition}[Weak Recovery]   \label{def:weak_recovery}
%We say $\hat \xi$ {\em weakly recovers} $\xi$ if, as
%$n \to \infty$,  $\ d(\xi, \hat\xi) / K \to 0$ in probability,  where $d$ denotes the Hamming distance.
%\end{definition}
The weak recovery guarantee is phrased in terms of convergence in probability,
which turns out to be equivalent to convergence in mean.
Indeed, the existence of an estimator satisfying $d(\xi, \hat\xi) / K \to 0$ is equivalent
 to the existence of an estimator such that  $\Expect[d(\xi, \hat\xi)] = o(K)$ (see \cite[Appendix A]{HajekWuXu_one_info_lim15} for a proof).
 Clearly, any estimator achieving exact recovery also achieves weak recovery; for bounded $K$,  these two criteria are equivalent.

Intuitively, for a fixed matrix size $n$, as either the submatrix size $K$ or the signal strength $\mu$ decreases,
it becomes more difficult to locate the submatrix. A key role is played by the parameter
\[
\lambda = \frac{\mu^2 K^2}{n},
\]
which is the signal-to-noise ratio for classifying an index $i$ according to the statistic $\sum_j W_{i,j}$, which is distributed according to $\calN(\mu K, n)$ if $i\in C^*$ and $\calN(0, n)$
if $i\not\in C^*.$
As shown in \prettyref{app:degreethreshold},
it turns out that if the submatrix size $K$ grows  linearly with $n$, the information-theoretic limits of both weak and exact recovery are easily attainable via thresholding.
To see this, note that in the case of $K \asymp n$ simply thresholding the row sums can provide weak recovery in $O(n^2)$ time
 provided that $\lambda \to \infty$, which coincides with the information-theoretic conditions of weak recovery as proved
in \cite{HajekWuXu_one_info_lim15}.
Moreover,  in this case, one can show that this thresholding algorithm followed by a linear-time voting procedure achieves exact
recovery whenever information-theoretically possible.
Thus, this paper concentrates on weak and exact recovery in the sublinear regime of
\begin{equation}
\Omega(\sqrt{n}) \leq K \leq o(n).
	\label{eq:focus}
\end{equation}
We show that  an optimized message passing algorithm provides weak recovery in nearly linear -- $O(n^2 \log n)$ -- time
if $\lambda > 1/\eexp$. This extends the sufficient conditions obtained in \cite{Deshpande12}  for the regime $K=\Theta(\sqrt{n}).$
Our algorithm is the same as the message passing algorithm proposed in \cite{Deshpande12}, except that we find the polynomial that maximizes the signal-to-noise ratio via Hermite polynomials instead of using the truncated Taylor series as in \cite{Deshpande12}.
%\nbr{JX: old version. We show that an optimized message passing algorithm proposed in \cite{Deshpande12}  provides weak recovery in nearly linear -- $O(n^2 \log n)$ -- time
%if $\lambda > 1/\eexp$.  This extends the sufficient conditions obtained in \cite{Deshpande12}  for the regime $K=\Theta(\sqrt{n}).$
%}
The  proofs follow closely those
 in \cite{Deshpande12}, with the most essential differences described in \prettyref{rmk:differences}.
 We observe that $\lambda > 1/\eexp$ is much more stringent than $\lambda > \frac{4K}{n} \log \frac{n}{K}$,
 the information-theoretic weak recovery threshold established in  \cite{HajekWuXu_one_info_lim15}.
It is an open problem whether any polynomial-time
algorithm can provide weak recovery for $\lambda \leq 1/e$.
In addition, we show that if $\lambda > 1/\eexp$, the message passing algorithm followed by a linear-time voting procedure can provide exact recovery whenever information theoretically possible.
This procedure achieves the optimal exact recovery threshold determined in  \cite{HajekWuXu_one_info_lim15} if $K \geq (\frac{1}{8 \eexp} + o(1)) \frac{n}{\log n}$.
See \prettyref{sec:compare} for a detailed comparison with information-theoretic limits.

%if both $\lambda > 1/\eexp$ and condition \prettyref{eq:submat-mle-suff}, ensuring the information theoretic possibility of exact recovery,
%hold.  We find that if $\liminf_{n\to \infty} K\log n/n \ge 1/(8\eexp)$,
%then the information-theoretic condition implies $\lambda > 1/\eexp$, and hence the sufficient condition
%of the algorithm matches the information-theoretic condition; if $\limsup_{n \to \infty} K \log n/n \le 1/(8\eexp)$,
%then $\lambda > 1/\eexp$ implies the information-theoretic condition, and thus $\lambda > 1/\eexp$ alone is sufficient for
%the algorithm to achieve exact recovery.
%%A key ingredient of the proof is to show that the message passing algorithm is able to achieve weak recovery if $\lambda > 1/\eexp$,
%%even though $C^\ast$ is random and unknown
%%to the algorithm, with the assumption $|C^*|/K\to 1$ in probability as $n\to \infty.$

%In the case of weak recovery the community size $|C^*|$ can be random and unknown
%to the algorithm, with the assumption $|C^*|/K\to 1$ in probability as $n\to \infty.$      For exact recovery it
%is assumed that $|C^*| \equiv K,$   although the
%belief propagation phase of the algorithm is followed by a cleanup phase, and the
%statistical problem faced in the cleanup phase involves $|C^*|$ being random.
%By definition, an estimator $\hat{C}$ is said to provide {\em weak recovery}
%if $\Expect[ |\hat C  \triangle C^*| ] = o(K)$ as $n\to \infty.$
%An estimator $\hat{C}$ is said to provide {\em exact recovery}
%if $\prob{\hat C = C^*} \to 1$ as $n \to \infty.$

The message passing algorithm is simpler to formulate and analyze for the principal submatrix recovery problem; nevertheless,
we show in  \prettyref{sec:Gaussian_bi_cluster} how to adapt the message passing algorithm and
its analysis to the biclustering problem.     Butucea et al. \cite{Butucea2013sharp} obtained sharp conditions
for exact recovery for the bicluster problem.   We show that calculations in \cite{Butucea2013sharp} with
minor adjustments provide information theoretic conditions for weak recovery as well.   The connection between
weak and exact recovery via the voting procedure described in  \cite{HajekWuXu_one_info_lim15}
carries over to the biclustering problem.

\paragraph{Notation}
For any positive integer $n$, let $[n]=\{1, \ldots, n\}$.
For any set $T \subset [n]$, let $|T|$ denote its cardinality and $T^c$ denote its complement.
For an $m\times n$ matrix $M$, let $\|M\|$ and $\|M\|_{\rm F}$  denote its spectral and Frobenius norm, respectively.
Let $\sigma_i(M)$ denote its singular values ordered decreasingly.
For any $S\subset [m], T \subset [n]$, let $M_{ST}\in \reals^{S \times T}$ denote $(M_{ij})_{i\in S,j \in T}$ and for $m=n$ abbreviate $M_{S}=M_{SS}$.
For a vector $x$, let $\|x\|$ denote its Euclidean norm.
We use standard big $O$ notations,
e.g., for any sequences $\{a_n\}$ and $\{b_n\}$, $a_n=\Theta(b_n)$ or $a_n  \asymp b_n$
if there is an absolute constant $c>0$ such that $1/c\le a_n/ b_n \le c$.
All logarithms are natural and we use the convention $0 \log 0=0$. Let $\Phi$ and $Q$
denote the cumulative distribution function (CDF) and complementary CDF  of the standard normal distribution,
respectively.
For $\epsilon\in[0,1]$, define the binary entropy function $h(\epsilon)\triangleq \epsilon\log \frac{1}{\epsilon} + (1-\epsilon)\log \frac{1}{1-\epsilon}$.
We say a sequence of events $\calE_n$ holds with high probability, if $\prob{\calE_n} \to 1$ as $n \to \infty$.
%\nbr{maybe. probably worth define $A_C$ and $A_{C_1C_2}$ instead of $\left.A\right|_{...}$,
% ordered singular values $\sigma_i$'s, etc.}

 \subsection{Algorithms and main results}
To avoid a plethora of factors $\frac{1}{\sqrt{n}}$ in the notation, we describe the message-passing algorithm
using the scaled version $A=\frac{1}{\sqrt{n}}W.$   The entries of $A$ have variance $\frac{1}{n}$ and mean
$0$ or $\frac{\mu}{\sqrt{n}}.$   This section presents algorithms and theoretical guarantees
for the symmetric model \prettyref{eq:models}.   \prettyref{sec:bi_clusterMP} gives adaptations to the asymmetric
case for the biclustering problem \prettyref{eq:model}.

Let $f(\cdot, t)\colon \reals \to \reals$ be a scalar function for each iteration $t$. Let $\theta^{t+1}_{i \to j}$ denote the
message transmitted from index $i$ to index $j$ at iteration $t+1$, which is given by
\begin{align}   \label{eq:theta_update_ij}
\theta^{t+1}_{i \to j} = \sum_{\ell \in [n] \backslash \ \{i, j \} } A_{\ell i} f( \theta^t_{\ell \to i }, t), \quad \forall j \neq i \in [n].
\end{align}
with the initial conditions $\theta^0_{i\to j} \equiv 0.$
Moreover, let $\theta^{t+1}_{i}$ denote index $i$'s belief at iteration $t+1$, which is given by
\begin{align}   \label{eq:theta_update_i}
\theta^{t+1}_i = \sum_{\ell \in [n] \backslash \{i \} } A_{\ell i} f( \theta^t_{\ell \to i }, t).
\end{align}
The form of  \prettyref{eq:theta_update_ij} is inspired by belief propagation algorithms, which have
the natural non backtracking property:  the message sent from $i$ to $j$ at time $t+1$ does not depend
on the message sent from $j$ to $i$ at time $t,$  thereby reducing the effect of echoes of messages sent by
$j.$

Suppose as $n \to \infty$, the messages $\theta^{t}_i$ (for fixed $t$) are such that the
empirical distributions of $(\theta^{t}_i: i\in [n]\backslash C^*)$ and $(\theta^{t}_i  :  i\in C^*)$ converge
to Gaussian  distributions with a certain mean and variance. Specifically, $\theta^{t}_i$ is approximately $\calN( \mu_t , \tau_t)$ for $i \in C^\ast$
and $\calN(0, \tau_t)$ for $i \notin C^\ast$.
Then \prettyref{eq:theta_update_ij}, \prettyref{eq:theta_update_i}, and the fact
$\theta_{i \to j }^t \approx \theta_i^t$
for all $i, j$  suggest the following recursive equations for $t\geq 0$:
\begin{align}
\mu_{t+1} &=\sqrt{\lambda} \expect{ f( \mu_t + \sqrt{\tau_t} Z, t) },  \label{eq:state_evolution1}  \\
\tau_{t+1} & = \expect{ f( \sqrt{\tau_t} Z, t )^2},    \label{eq:state_evolution2}
\end{align}
where $Z$ represents a standard normal random variable, and the initial conditions are $\mu_0=\tau_0=0.$
Following \cite{Deshpande12}, we call \prettyref{eq:state_evolution1} and \prettyref{eq:state_evolution2}
the {\em state evolution equations}, which are justified in \prettyref{sec:state_evolution}.
Thus, it is reasonable to estimate $C^*$ by selecting those indices $i$ such that $\theta_i^t$ exceeds a
given threshold.

Suppose, for the time being, that message distributions are Gaussian with parameters accurately
tracked by the state evolution equations.       Then classifying an index $i$ based
on $\theta_i^t$ boils down to testing two Gaussian
hypotheses with signal-to-noise ratio  $\frac{\mu_{t+1}^2}{\tau_{t+1}}.$   This gives
guidance for selecting the functions $f(\cdot, t)$ based on $\mu_t$
and $\tau_t$
to maximize $\frac{\mu_{t+1}}{\sqrt{\tau_{t+1}}}$.   For $t=0$ any choice of $f$ is equivalent, so long as $f(0, 0) >0.$
Without loss of generality, for $t\geq 1,$ we  can assume that the variances are normalized, namely, $\tau_t=1$ (e.g. we take $f(0,0)=1$ to make $\tau_1=1$) and choose $f(\cdot, t)$ to be the maximizer of
\begin{equation}
	\max\{\Expect[g(\mu_t + Z)]\colon \Expect[g(Z)^2]=1  \}
	\label{eq:optf}
\end{equation}
where $Z \sim {\cal N}(0,1)$.
%Define the inner product $\iprod{g}{h} = \Expect[g(Z)h(Z)]$ and the norm $\|g\| = \sqrt{\iprod{g}{g}}$.
By change of measure, $\Expect[g(\mu_t + Z)] = \Expect[g(Z)\rho(Z)]$, where
\begin{equation}
	\rho(x)= \frac{\diff \calN(\mu_t,1)}{\diff \calN(0,1)}(x)= e^{\mu_t x - \mu_t^2/2}.
	\label{eq:rho}
\end{equation}
% is the relative density of $\calN(\mu_t,1)$ with respect to $\calN(0,1)$.
Clearly, the best $g$ aligns with $\rho$ and we obtain
\begin{equation}
f(x,t) = \frac{\rho(x)}{\sqrt{\Expect[\rho^2(Z)]}} =  e^{x\mu_t-\mu_t^2}.
	\label{eq:fexp}
\end{equation}
%so we want to select $f(\cdot, t)$ to be a function $g$ that maximizes
%$\frac{\Expect[g(\mu_t + Z)]}{\sqrt{ \Expect[g(Z)^2]}},$ where $Z$ is ${\cal N}(0,1).$
%A simple exercise in change of variable and calculus of variations shows that the optimal choice of
%$g(x)$ is $ce^{x\mu_t}$ for an arbitrary constant $c>0,$  suggesting we take
%$f(x , t)$ of the form
%\begin{equation}
%f(x,t) = ce^{x\mu_t}.
%	\label{eq:fexp}
%\end{equation}
%
%  Another way to determine $f$ is  to use the form of message passing algorithm
%   suggested by belief propagation;   (This is currently incomplete and in the supplementary file.)
%Letting $c=1/\sqrt{\Expect[\eexp^{2\mu_t Z}]} = e^{-\mu^2_t} $, it follows that
With this optimized $f$, we have $\tau_t \equiv 1$ and
the state evolution \prettyref{eq:state_evolution1}  reduces to
$$\mu_{t+1} = \sqrt{\lambda} \expect{ f( \mu_t + Z, t) } = \sqrt{\lambda} \eexp^{\frac{\mu_t^2}{2}},$$
%\begin{equation}
%\mu_{t+1}= \frac{ \sqrt{\lambda} \Expect[\eexp^{\mu_t(\mu_t + Z)}]}{\sqrt{\Expect[\eexp^{2\mu_t Z}]}} = \frac{\sqrt{\lambda} \eexp^{\frac{3\mu_t^2}{2}}}{\eexp^{\mu_t^2}}
%=\sqrt{\lambda} \eexp^{\frac{\mu_t^2}{2}},
%	\label{eq:mut}
%\end{equation}
or, equivalently,
\begin{equation}
\mu_{t+1}^2=\lambda \eexp^{\mu_t^2}.	
	\label{eq:mu-ideal}
\end{equation}
Therefore if $\lambda > 1/\eexp$, then \prettyref{eq:mu-ideal} has no fixed point and hence $\mu_t \rightarrow \infty$ as $t\rightarrow \infty$.

Directly carrying out the above heuristic program, however, seems challenging. To rigorously justify the state evolution equations in \prettyref{sec:state_evolution}
we rely on the the method  of moments,
%Notice that the state evolution equations  are justified in \prettyref{sec:state_evolution} by the method  of moments,
requiring $f$ to be a polynomial, which prompts us to look for the best polynomial of a given degree that maximizes the signal-to-noise ratio.
Denoting the corresponding state evolution by $(\hat \mu_t,\hat \tau_t)$, we aim to solve the following finite-degree version of \prettyref{eq:optf}:
\begin{equation}
	\max\{\Expect[g(\hat \mu_t + Z)]\colon \Expect[g(Z)^2]=1, \deg(g) \leq d \}.
	\label{eq:optf-poly}
\end{equation}
As shown in \prettyref{lmm:hermite}, this problem can be easily solved via Hermite polynomials, which form an orthogonal basis with respect to the Gaussian measure, and the optimal choice, say, $f_{d}(\cdot,t)$ can be obtained by normalizing the first $d+1$ terms in the orthogonal expansion of relative density \prettyref{eq:rho}, \ie, the best degree-$d$ $L_2$-approximation.
%an explicit linear combination of Hermite polynomials (orthogonal basis)
Compared to \cite[Lemma 2.3]{Deshpande12} which shows the existence of a good choice of polynomial that approximates the ideal state evolution \prettyref{eq:mu-ideal} based on Taylor expansions, our approach is to find the best message-passing rule of a given degree which results in the following state evolution that is optimal among all $f$ of degree $d$:
\begin{equation}
\hat\mu_{t+1}^2=\lambda \sum_{k=0}^d \frac{\hat\mu_t^{2k}}{k!}.	
	\label{eq:mu-poly}
\end{equation}
For any $\lambda > 1/e$, there is an explicit choice of the degree $d$ depending only on $\lambda$,\footnote{As $\lambda$ gets closer to the critical value $1/e$, we need a higher degree to ensure \prettyref{eq:mu-poly} diverges and in fact $d$ grows quite slowly as $\Theta(\log \frac{1}{\lambda e-1} / \log \log \frac{1}{\lambda e-1})$
See \prettyref{rmk:deg-lambda}.} so that $\hat\mu_t \to \infty$ as $t \diverge$ and the state evolution \prettyref{eq:mu-poly} for fixed $t$ correctly predicts the asymptotic behavior of the messages when $n\diverge$.
%instead of choosing $f$ to be the optimal exponential function in \prettyref{eq:fexp}, we let $f$ be a high-degreee polynomial which approximates the exponential.
% but then polynomials can be taken to approximate the exponential function.
As discussed above, $\tilde{C}$ produced by thresholding messages $\theta_i^t$,  is likely to contain a large portion of $C^\ast$,
but since $K=o(n)$, it may (and most likely will)  also contain a large number of indices not in $C^\ast$.
Following \cite[Lemma 2.4]{Deshpande12}, we show that the power iteration\footnote{Note that as far as statistical utility is concerned, we could replace $\hat{u}$  produced by the power iteration by the leading singular vector of $A_{\tilde{C}}$, but that would incur a higher time complexity because
singular value decomposition in general takes $O(n^3)$ time to compute.}  (a standard spectral method)
in Algorithm \ref{alg:MP} can remove a large portion of the outlier vertices in $\tilde{C}.$
%The key idea is that $|\tilde{C} \cap C^*| \approx K(1-\epsilon)$ and $|\tilde{C} \backslash C^*| \approx
%n(\epsilon)$ for some small $\epsilon>0$. Thus the signal-to-noise ratio $\lambda$ is boosted up to $\frac{K^2 (1-\epsilon)^2}{
%n \epsilon}$, which converges to $\infty$ by letting $\epsilon \to 0$, and consequently the spectral method

Combining message passing plus spectral cleanup yields  the following algorithm for estimating $C^\ast$ based on the messages $\theta_i^t$.

\begin{algorithm}[htb]
\caption{Message passing}\label{alg:MP}
\begin{algorithmic}[1]
\STATE Input: $n, K \in \naturals$, $\mu>0$, $A \in \reals^{n\times n}$, $d^\ast, t^\ast \in \naturals,$  and $s^* > 0.$
\STATE Initialize: $\theta_{i \to j}^0=0$ for all $i, j \in [n]$ with $i \neq j$ and $\theta_{i}^0=0$.
For $t\geq 0$, define
the sequence of degree-$d^*$ polynomials $f_{d^*}(\cdot, t)$ as per  \prettyref{lmm:hermite}
and $\hat{\mu}_t$ in \prettyref{eq:mu-poly}.
\STATE Run $t^\ast-1$ iterations of message passing as in \eqref{eq:theta_update_ij} with $f =f_{d^*}$ and
compute $\theta_{i}^{t^{\ast} }$ for all $i \in [n]$ as per \eqref{eq:theta_update_i}.
\STATE Find the set $\tilde{C}=\{ i \in [n]: \theta_{i}^{t^\ast} \ge \hat{\mu}_{t^\ast}/2 \}$.
\STATE (Cleanup via power method)
Recall that $A_{\tilde{C}}$ denotes the restriction of $A$ to the rows and columns with index in $\tilde{C}.$
Sample $u^0$ uniformly from the unit sphere in $\reals^{\tilde{C}}$ and compute
$u^{t+1}=A_{\tilde{C}}u^t / \|   A_{\tilde{C}}u^t        \| $  for $ 0\leq t \leq \lceil s^*\log n\rceil -1.$  Let $\hat{u}=u^{\lceil s^*\log n\rceil }.$
Return  $\hat{C},$  the set of  $K$ indices $i$ in $\tilde C$ with the largest values of $|\hat{u}_i|.$
%\STATE Let $A|_{\tilde{C}}$ denote the restriction of $A$ to the rows and columns with index in $\tilde{C}$, and estimate
%its principal eigenvector by running the power method for $s^\ast \log n $ iterations; denote the
%estimate by $u \in \reals^{|\tilde{C}|}.$
%\STATE Return  $\hat{C} \subset [n]$ as the set of the top $K$ entries of $u$ in absolute value.
\end{algorithmic}
\end{algorithm}

%%%%%% subsequence argument %%%%%%%
%\nbr{YW: I am not sure if we should modify the statement of \prettyref{thm:almost_exactBP_submat} or add a corollary. Anyway here it is:}
%
%First note that \prettyref{alg:MP} runs in $O((s* \log n +t^*) n^2$ time.
%For each $m\in\naturals$, let $\hat C^{(n)}_m$ be the output of \prettyref{alg:MP}
%with input $(n,K,d^*, t^*, s^*)$, where
%$d^*=d^*(8\log m), t^*=t^*(8\log m)$ are defined in \prettyref{lmm:polynomialapproximation}  and $s^*$ is defined in \prettyref{eq:sstar} with $\epsilon=1/m$.
%By \prettyref{thm:almost_exactBP_submat}, there exists a subsequence $\{n_m\}$ such that the following holds:
%\[
%\sup_{n \geq n_m} \prob{|\hat C_m^{(n)} \Delta C^*|  \geq K/m} \leq 1/m \quad \text{and} \quad n_m \geq t^*(\log m).
%\]
%Set $\hat C^{(n)} \triangleq C_m^{(n)}$ for all $n_m \leq n < n_{m+1}$. Then we have $|\hat C^{(n)} \Delta C^*|/K \to 0$ in probability as $n\diverge$. Furthermore, since $s^* = O(\frac{1}{\log m})$ and $t^* \leq n_m \leq n$ by definition, the total run time is $O(n^2 \log n)$.
%
%\nbr{please double check the definition of $s^*$ in \prettyref{eq:sstar}. i felt the proof can use some more editing. for example what is $C'$ there? can I pick $C'=100$?}

The following theorem provides a performance guarantee for Algorithm \ref{alg:MP} to approximately recover $C^\ast.$
\begin{theorem}  \label{thm:almost_exactBP_submat}
Fix $\lambda > 1/\eexp.$   Let $K$ and $\mu$ depend on $n$ in such a way that
$ \frac{\mu^2 K^2}{n} \to  \lambda$ and $\Omega(\sqrt{n}) \leq K \leq o(n)$ as $n\to \infty.$
Consider the model \prettyref{eq:models} with $|C^*|/K \to 1$ in probability as $n\to \infty.$
Define $d^\ast(\lambda)$ as in \prettyref{eq:dlambda}.
For every $\eta \in (0,1),$  there exist explicit positive constants  $t^\ast, s^\ast, c$ depending on  $\lambda$ and $\eta$
such that Algorithm \ref{alg:MP} returns
$ | \hat{C} \Delta C^\ast| \le \eta K $,
with probability converging to one as $n \to \infty$, and the
total time complexity is bounded by $c(\eta, \lambda) n^2 \log n$, where $c(\eta,\lambda) \to \infty$ as either $\eta \to 0$ or $\lambda  \to 1/\eexp.$
%\nb{below goes inside the proof}
%For every $\epsilon \in (0,1),$  there exist positive constants  $t^\ast, s^\ast, \eta, c$ depending on  $\lambda$ and $\epsilon$
%such that Algorithm \ref{alg:MP} returns
%$ | \hat{C} \cap C^\ast| \ge (1- \eta(\epsilon, \lambda) )K $,
%with probability converging to $1$ as $n \to \infty$, and the
%total time complexity is bounded by $c(\epsilon, \lambda) n^2 \log n$, where $\eta(\epsilon,\lambda)= 2 \epsilon + \frac{5000 \sqrt{\epsilon} }{\lambda(1-\epsilon)^2}$
%and $c(\epsilon,\lambda) \to \infty$ as $\epsilon \to 0$ or $\lambda  \to 1/\eexp.$
\end{theorem}

After the message passing algorithm and spectral cleanup are applied in   \prettyref{alg:MP},
a final linear-time voting procedure is deployed to obtain weak or exact recovery, leading to \prettyref{alg:exactrecovery} next.
As in \cite{Deshpande12}, we consider a threshold estimator for each vertex $i$ based on a sum over
$\hat C$ given by $r_i=\sum_{j \in \hat{C} } A_{ij}$. Intuitively, $r_i$ can be viewed as the aggregated ``votes'' received by the index $i$ in $\hat C$,
and the algorithm picks the set of $K$ indices with the most significant ``votes''. To show that
this voting procedure succeeds in weak recovery, a key step is to prove that $r_i$ is close to $\sum_{j \in C^\ast } A_{ij}.$
If $\mu =\Theta(1)$ as in \cite{Deshpande12},  given that $| \hat C  \triangle C^* |  = o(K),$  the error incurred by  summing
over $\hat C$ instead of over $C^*$  could be bounded by truncating $A_{ij}$ to a large magnitude.
However, for $\mu \to 0$ that approach fails (see \prettyref{rmk:differences} for more details).
Our approach is to introduce the clean-up procedure in  \prettyref{alg:exactrecovery}
based on the {\em successive withholding} method described in \cite{HajekWuXu_one_info_lim15} (see also  \cite{Condon01,MosselNeemanSlyCOLT14,Mossel14} for variants of this method).
In particular, we randomly partition
the set of vertices into $1/\delta$ subsets.   One at a time,
one subset, say $S$, is withheld to produce a reduced set of vertices $S^c$, on which we apply \prettyref{alg:MP}.
The estimate obtained from $S^c$ is then used by the voting
procedure to classify the vertices in $S$.
The analysis of the two stages is decoupled because conditioned on $C^*$, the outcome of \prettyref{alg:MP} depends only on $A_{S^c}$, which is independent of $A_{SS^c}$
used in the voting.

\begin{algorithm}[htb]
\caption{Message passing plus voting}\label{alg:exactrecovery}
\begin{algorithmic}[1]
\STATE Input: $n, K \in \naturals$, $\mu>0$, $A \in \reals^{n\times n}$, $\delta \in (0,1)$ with $1/\delta, n\delta \in \naturals$, $d^\ast, t^\ast  \in \naturals,$
and $s^\ast > 0.$
\STATE  Partition $[n]$ into $ 1/\delta$ subsets $S_k$ of size $n\delta$ randomly.
% \nbr{YW: in the intro we said partition randomly. this does not matter but we should be consistent.}
\STATE (Approximate recovery) For each $k=1, \ldots,  1/\delta $,
%let $A_k$ denote the restriction of $A$ to the rows and columns with index in $[n]\backslash S_k$,
run  Algorithm \ref{alg:MP} (message passing for approximate recovery)
with input $\left(n(1-\delta), \lceil K(1-\delta) \rceil, \mu, A_{S_k^c}, %\epsilon,   (epsilon is not an input to Alg. 1)
 d^\ast, t^\ast, s^\ast\right)$ which outputs $\hat{C}_k$.
\STATE (Clean up) For each $k=1, \ldots,  1/\delta $  compute $r_i=\sum_{j \in \hat{C}_k } A_{ij}$ for all $i \in S_k$ and return
$\check{C}$, the set of $K$ indices in $[n]$ with the largest values of $r_i.$
\end{algorithmic}
\end{algorithm}

The following theorem provides a sufficient condition for the message passing plus voting
cleanup procedure (\prettyref{alg:exactrecovery}) to achieve
weak  recovery, and, if the information-theoretic sufficient condition is also satisfied, exact recovery.

\begin{theorem}  \label{thm:weakexactBP_submat}
 Let $K$ and $\mu$ depend on $n$ in such a way that
$ \frac{\mu^2 K^2}{n} \to  \lambda$ for some fixed $\lambda > 1/\eexp$.
and $\Omega(\sqrt{n}) \leq K \leq o(n)$ as $n\to \infty.$
Consider the model \prettyref{eq:models} with $|C^*| \equiv K $.
Let $\delta > 0$  be such that
$\lambda \eexp (1-\delta) >1$.
Define $d^*=d^\ast(\lambda(1-\delta) )$ as per \prettyref{eq:dlambda}.
Then  there exist positive constants $t^\ast, s^\ast, c$ determined explicitly by
$\delta$ and $\lambda$, such that
\begin{enumerate}
	\item (Weak recovery) \prettyref{alg:exactrecovery} returns $\check{C}$ with
$ | C' \Delta C^\ast|/K  \to 0$ in probability   as $n \to \infty$.

\item (Exact recovery) Furthermore, assume that
\begin{equation}
\liminf_{n\to \infty}   \frac{ \sqrt{K}\mu }{ \sqrt{ 2 \log K} + \sqrt{ 2  \log n } } >1.   \label{eq:submat-mle-suff}
\end{equation}
Let $\delta > 0$  be chosen such that for all sufficiently large $n,$
$$\min\left\{ \lambda \eexp (1-\delta) , \frac{ K\mu(1-2\delta) }{ \sqrt{ 2K \log K} + \sqrt{ 2 K \log (n- K) } + \delta\sqrt{K}}  \right\} >1.$$
Then \prettyref{alg:exactrecovery} returns
$\check{C}$ with $\pprob{ \check{C}\neq C^\ast }\to 0 $  as $n \to \infty$.
\end{enumerate}
The total time complexity is bounded by $c(\delta,\lambda) n^2 \log n$, where $c(\delta,\lambda) \to \infty$ as $\delta \to 0$ or $\lambda  \to 1/\eexp$.
\end{theorem}

%
%%The following theorem provides a sufficient condition for the message passing plus the voting
%%cleanup procedure (\prettyref{alg:exactrecovery}) to achieve
%%exact recovery.
%Assuming the information-theoretic sufficient condition, we can further show \prettyref{alg:exactrecovery} achieves
%exact recovery.
%\begin{theorem}  \label{thm:weakexactBP_submat}
%Under the same assumptions of \prettyref{thm:weakexactBP_submat},
%% Let $K$ and $\mu$ depend on $n$ in such a way that
%%$ \frac{\mu^2 K^2}{n} \to  \lambda$ for some $\lambda > 0$,    and $\Omega(\sqrt{n}) \leq K \leq o(n).$
%%Consider the model \prettyref{eq:models} with $|C^*|\equiv K.$
%%Assume $\lambda > 1/\eexp$ and
%further assume that
%%the information-theoretic condition
%\begin{equation}
%\liminf_{n\to \infty}   \frac{ \sqrt{K}\mu }{ \sqrt{ 2 \log K} + \sqrt{ 2  \log n } } >1.   \label{eq:submat-mle-suff}
%\end{equation}
%Let $\delta > 0$  be chosen such that for all sufficiently large $n,$
%$$\min\left\{ \lambda \eexp (1-\delta) , \frac{ K\mu(1-2\delta) }{ \sqrt{ 2K \log K} + \sqrt{ 2 K \log (n- K) } }  \right\} >1.$$
%Then  there exist  explicit positive constants $d^\ast(\lambda(1-\delta))$, $t^\ast(\delta, \lambda)$, and  $s^\ast(\delta, \lambda)$,
%such that \prettyref{alg:exactrecovery} returns
%$\check{C}$ with $\pprob{ \check{C}\neq C^\ast }\to 0 $  as $n \to \infty$, and the
%total time complexity is bounded by
%$c(\delta,\lambda) n^2 \log n$, where
%$c(\delta,\lambda) \to \infty$ as $\delta \to 0$ or $\lambda  \to 1/\eexp$.
%\end{theorem}

\begin{remark} \label{rmk:two_step}
As shown in \cite[Theorem 7]{HajekWuXu_one_info_lim15}, if there is an algorithm that can approximately recover $|C^*|$ even if
$|C^*|$ is random and only approximately equal to $K,$  then that algorithm can be combined with a linear-time voting procedure to achieve exact recovery.
By \prettyref{thm:almost_exactBP_submat},    \prettyref{alg:MP}
indeed works for such random $|C^*|$ and  so the second part of \prettyref{thm:weakexactBP_submat} follows from  \prettyref{thm:almost_exactBP_submat} and
the general results of  \cite{HajekWuXu_one_info_lim15}.
\end{remark}

\begin{remark}\label{rmk:sufficient_comparison}
\prettyref{thm:weakexactBP_submat} ensures \prettyref{alg:exactrecovery} achieves exact recovery if
both  \prettyref{eq:submat-mle-suff} and $\lambda > 1/\eexp$ hold;  it is of interest
to compare these two conditions.  Note that
$$
\frac{ \sqrt{K} \mu  }{ \sqrt{ 2 \log K} + \sqrt{ 2  \log n } } =\sqrt{\lambda \eexp} \times \sqrt{\frac{n}{8\eexp K \log n}} \frac{2}{(1+ \sqrt{ \log K/\log n} )}.
$$
Hence, if $\liminf_{n\to \infty} K \log n/n  \ge \frac{1}{8 \eexp}$,  \prettyref{eq:submat-mle-suff} implies  $\lambda > 1/\eexp$
and thus \prettyref{eq:submat-mle-suff} alone is sufficient for \prettyref{alg:exactrecovery} to succeed;
if $ \limsup_{n\to \infty} K \log n/n  \le \frac{1}{8 \eexp}$, then $\lambda > 1/\eexp$ implies \prettyref{eq:submat-mle-suff}
and thus $\lambda > 1/\eexp$ alone is sufficient for \prettyref{alg:exactrecovery} to succeed. The asymptotic
regime considered in \cite{Deshpande12} entails $K=\Theta(\sqrt{n}),$ in which case the condition $\lambda > 1/\eexp$
is sufficient for exact recovery.
\end{remark}

\subsection{Comparison with information theoretic limits}  \label{sec:compare}

As noted in the introduction, in the regime $K= \Theta(n)$, a thresholding algorithm based on row sums
provides weak and, if a voting procedure is also used, exact recovery whenever it is informationally possible.
In this subsection,
we compare the performance of the message passing algorithms
to the information-theoretic limits on the recovery problem in the regime \prettyref{eq:focus}.
%We find that computational complexity seems to incur a severe penalty on the statistical optimality in this regime.
Notice that the comparison here takes into account the sharp constant factors.
Information-theoretic limits for the biclustering problem are discussed
in \prettyref{sec:bi_cluster_info_limits}.

\paragraph{Weak recovery}
The information-theoretic threshold for weak recovery has been determined in \cite[Theorem 2]{HajekWuXu_one_info_lim15}, which, in the regime of \prettyref{eq:focus}, boils down to the following:
%\footnote{More general results given in \cite{HajekWuXu_one_info_lim15} do not require $K=o(n)$ or $K\to \infty.$}
Weak recovery is possible if
 \begin{equation}
\liminf _{n\to\infty} \frac{K \mu^2}{4 \log \frac{n }{K}}> 1,
	\label{eq:weak_Gaussian_suff}
\end{equation}
and impossible if
%then weak recovery is possible. If weak recovery is possible, then
 \begin{equation}
%\liminf _{n\to\infty} \frac{ K\mu^2}{4 \log \frac{n }{K}} \ge 1.\label{eq:weak_Gaussian_nec}
\limsup _{n\to\infty} \frac{ K\mu^2}{4 \log \frac{n }{K}} < 1.\label{eq:weak_Gaussian_nec}
\end{equation}
This implies that the minimal signal-to-noise ratio for weak recovery is
\[
\lambda > (4+\epsilon) \frac{K}{n} \log \frac{n}{K}
\]
for any $\epsilon > 0$, which vanishes in the sublinear regime of $K=o(n)$.
%If $K=o(n)$, we have
%Note that
%\begin{align}     \label{eq:exact_vs_poly_Gaussian}
%\frac{K\mu^2}{4 \log(n/K)} =
%%\frac{K(p-q)^2}{\log(n/K) q} \asymp
% \left[   \frac{n/K }{\log(n/K)} \right] \frac{\lambda}{4}.
%\end{align}
%Thus, the larger $n/K$ grows, the further the sufficient condition for weak recovery, \prettyref{eq:weak_Gaussian_suff},  becomes
%from the sufficient condition for recovery by the message passing algorithms,
In contrast, in the regime \prettyref{eq:focus},
to achieve weak recovery message passing (\prettyref{alg:MP}) demands a non-vanishing signal-to-noise ratio, namely, $\lambda > 1/\eexp$.
%\footnote{Note that
%sufficiency of $\lambda > 1/\eexp$ is predicated upon $K=o(n).$}
No polynomial-time algorithm is known to succeed if
$\lambda \leq  1/ \eexp$,  suggesting that computational complexity might incur a severe penalty on the statistical optimality when $K=o(n)$.

\paragraph{Exact recovery}

%Next, we compare the sufficient conditions of \prettyref{alg:exactrecovery} to the information-theoretic conditions for exact recovery
%given below.
%We can also compare the conditions for exact recovery to the conditions for the message passing algorithm to provide
%recovery.   Both conditions are stronger than the conditions for weak recovery.

%\begin{theorem}  \label{thm:nec_exact_submat}
%(\cite{HajekWuXu_one_info_lim15}, Exact recovery)
In the regime of \prettyref{eq:focus}, the information limits of exact recovery (see \cite[Theorem 4 and Remark 7]{HajekWuXu_one_info_lim15})
are as follows: Exact recovery is possible if
 \prettyref{eq:submat-mle-suff} holds, and impossible if
 \begin{equation}
\limsup_{n \to \infty} \frac{ \sqrt{K} \mu  }{ \sqrt{ 2 \log K} + \sqrt{ 2  \log n } }  < 1.
	\label{eq:submat-mle-nece}
\end{equation}
%and
%\begin{equation}
%\liminf_{n \to \infty} \frac{ \sqrt{K} \mu  }{ \sqrt{ 2 \log K} + \sqrt{ 2  \log n } }  > 1,
%	\label{eq:submat-mle-suff}
%\end{equation}
%then exact recovery is possible. If exact recovery is possible, then \prettyref{eq:weak_Gaussian_nec} holds,
%and
%\begin{equation}
%\liminf_{n \to \infty} \frac{ \sqrt{K} \mu  }{ \sqrt{ 2 \log K} + \sqrt{ 2  \log n } }  \ge 1.
%	\label{eq:submat-mle-nece}
%\end{equation}
%\end{theorem}
%It is easy to show that if $K \leq n^{1/9}$, \prettyref{eq:weak_Gaussian_suff}
%implies \prettyref{eq:submat-mle-suff}, and thus  \prettyref{eq:weak_Gaussian_suff} alone is sufficient for exact recovery;
%if $K\geq  n^{1/9}$, then \prettyref{eq:submat-mle-suff}
%implies \prettyref{eq:weak_Gaussian_suff}, and  \prettyref{eq:submat-mle-suff} alone
%is sufficient for exact recovery. Therefore, in the regime $\Omega(\sqrt{n})\le K \le o(n)$, \prettyref{eq:submat-mle-suff}
%alone is sufficient for exact recovery informationally.
In view of \prettyref{rmk:sufficient_comparison}, we conclude that \prettyref{alg:exactrecovery} achieves the sharp threshold of exact recovery if
\begin{equation}
K \geq \pth{\frac{1}{8 \eexp} + o(1)} \frac{n}{\log n}.	
	\label{eq:Kexact}
\end{equation}
We note that a counterpart of this conclusion for the biclustering problem is obtained in \prettyref{rmk:K12exact} in terms of the submatrix sizes.
%on one hand,
%if $\liminf_{n\to \infty} K \log n/n  \ge \frac{1}{8 \eexp}$, condition \prettyref{eq:submat-mle-suff} alone is sufficient for \prettyref{alg:exactrecovery} to succeed, thus matching
%the information-theoretic sufficient condition.
%On the other hand, if $ \limsup_{n\to \infty} K \log n/n  \le \frac{1}{8 \eexp}$, then  for  \prettyref{alg:exactrecovery} to succeed,
%$\lambda > 1/\eexp$ is sufficient, which is at least as strong as the information-theoretic sufficient condition  \prettyref{eq:submat-mle-suff}.

%\medskip
To further the discussion on weak and exact recovery,
consider the regime
$$
K=  \frac{\rho n }{\log^{s-1} n}, \quad  \mu^2= \frac{\mu_0^2 \log^s n}{n},$$
 where
$s\ge 1$, $\rho \in (0,1)$, and $\mu_0 > 0$ are fixed constants.
Throughout this regime, weak recovery is information theoretically possible because
the left-hand side of \prettyref{eq:weak_Gaussian_suff} is $\Omega(\frac{\log n}{\log\log n}) \to \infty$.
%$ \frac{K \mu^2}{4 \log \frac{n }{K}} = \Omega(\frac{\log n}{\log\log n})
%\frac{\rho \mu_0^2 \log n}{4( (s-1)\log\log n - \log \rho) }
% \to \infty$ as $n\to \infty,$
On one hand, in view of \prettyref{eq:submat-mle-suff} and \prettyref{eq:submat-mle-nece},
  exact recovery is possible if $\frac{\rho \mu_0^2}{8} > 1$ and impossible
 if   $\frac{\rho \mu_0^2}{8} <1$. On the other hand, $\lambda = \rho^2\mu_0^2(\log n)^{2-s},$ yielding:
\begin{itemize}
\item   When $1 \le s<2$,  then $\lambda =\Omega(\log^{2-s} n )\to \infty.$
Thus weak recovery is  achievable in polynomial-time
by the message passing algorithm,  spectral methods, or even row-wise thresholding.
If  $\frac{\rho \mu_0^2}{8} >1$, exact recovery is  attainable in polynomial-time by combining the weak recovery algorithm with a linear time
voting procedure as shown in \cite{HajekWuXu_one_info_lim15}.
\item
When $s=2$, then $\lambda = \rho^2\mu_0^2,$  and  weak recovery by the message passing algorithm is possible if  $\rho^2 \mu_0^2 \eexp  > 1.$
  \prettyref{fig:MP_phase_plot} shows the curve $\{(\mu_0, \rho):  \rho^2 \mu_0^2 \eexp =1 \}$ corresponding to the weak recovery condition by the message
passing algorithm,  and the curve $\{ (\mu_0, \rho): \rho \mu_0^2/8 =1\}$ corresponding to the information-theoretic exact recovery condition.
When $\rho \geq \frac{1}{8e}$, the latter curve dominates the former and \prettyref{alg:exactrecovery} achieves optimal exact recovery.
\item
When $s>2$,  $\lambda \to 0$, and no polynomial-time procedure is known to provide weak, let alone exact, recovery.
\end{itemize}

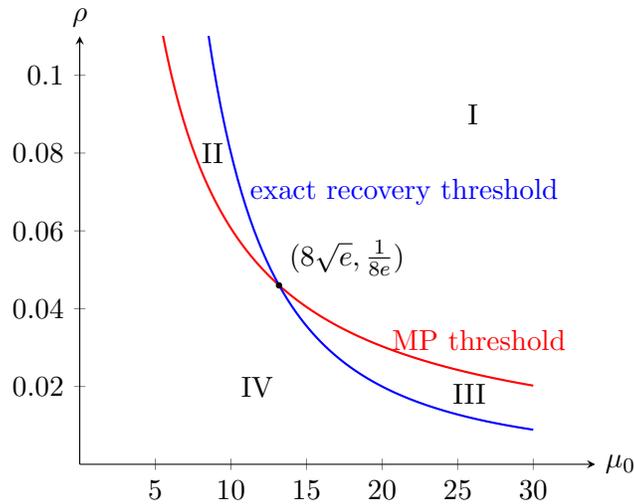
\begin{figure}[hbt]
\centering
\begin{tikzpicture}[transform shape,scale=1]
\begin{axis}[
    standard,
    enlargelimits=upper,
    xmin=0,   xmax=31,
	ymin=0,   ymax=0.1,
%extra x ticks={10},
%	extra y ticks={0.01},
%        enlarge x limits=0.02,
%        enlarge y limits=0.02,
xlabel=$\mu_0$,ylabel=$\rho$,
scaled ticks=false, tick label style={/pgf/number format/fixed},
every axis plot post/.append style={
  mark=none,samples=50,smooth,thick} % All plots: 50 samples, smooth, no marks
]
%ratio c=a/b > 1
%\pgfmathsetmacro{\C}{2}

\addplot[domain=3:30,red]{exp(-1/2)/x};
\addplot[domain=3:30,blue]{8/(x^2)};

\node[anchor=west] at (axis cs:25,0.09){I};
\node[anchor=west] at (axis cs:24,0.018){III};
\node[anchor=west] at (axis cs:7.4,0.08){II};
\node[anchor=west] at (axis cs:10,0.02){IV};

\node[anchor=west,red] at (axis cs:20,0.032){MP threshold};
\node[anchor=west,blue] at (axis cs:10.6,0.07){exact recovery threshold};
%\node[anchor=west,blue,rotate=-60] at (axis cs:160,0.08){\scriptsize optimal threshold};
\draw[fill=black] (axis cs:13.18,0.046) circle[radius=1pt] node[anchor=south west] {$(8\sqrt{e},\frac{1}{8e})$};
%\draw[blue,dashed] (axis cs:0,1) -- (axis cs:1,1);
%\draw[blue,dashed] (axis cs:1,0) -- (axis cs:1,1);
%\draw[blue,dashed] (axis cs:0,.75) -- (axis cs:.25,.75);
%\draw[blue,dashed] (axis cs:.25,0) -- (axis cs:.25,.75);
\end{axis}
\end{tikzpicture}
\caption{Phase diagram for the Gaussian model with $K=\rho n /\log n$ and
$\mu^2=\mu_0^2\log^2 n/n$  for  $\mu_0, \rho$ fixed as $n\to \infty.$
In region I, exact recovery is provided by the message passing (MP) algorithm plus voting cleanup.
In region II, weak recovery is provided by MP, but exact recovery is not information theoretically possible.
In region III exact recovery is possible, but no polynomial time algorithm is known for even weak recovery.
In region IV, with $\mu_0  >0$ and $\rho>0$, weak recovery, but not exact recovery, is possible
and no polynomial time algorithm is known for weak recovery.}
\label{fig:MP_phase_plot}
\end{figure}

\subsection{Comparison with the spectral limit}
\label{sec:spectral}

It is reasonable to conjecture that $\lambda > 1$ is the spectral limit for recoverability by spectral estimation
methods.   This conjecture is rather vague, because it is difficult to define what constitutes spectral
methods.   Nevertheless, some evidence for this conjecture is provided by \cite[Proposition 1.1]{Deshpande12},
which, in turn, is  based on results on the spectrum of a random matrix perturbed by adding a
rank-one deterministic matrix \cite[Theorem 2.7]{KnowlesYin13}.

The message passing framework used in this paper itself provides some evidence for the conjecture.  Indeed,
if $f(x,0)\equiv 1$ and $f(x,t)=x$ for all $t\geq 1$, the iterates $\theta^t$ are close to what is obtained by iterated multiplication by the matrix $A,$  beginning with the all one vector, which is the power method for computation of the eigenvector corresponding to the principal eigenvalue of $A$.\footnote{Note that if we included $i,j$ in the summation in \prettyref{eq:theta_update_ij} and \prettyref{eq:theta_update_i}, then we would have $\theta^t = A^t \ones$ exactly. Since the entries of $A$ are $O_P(1/\sqrt{n})$, we expect this only incurs a small difference to the sum for finite number of iterations.}
To be more precise, with this linear $f$ the message passing equation \prettyref{eq:theta_update_ij} can be expressed in terms of powers of the {\em non-backtracking matrix} $\mathbf{B} \in \reals^{\binom{n}{2} \times \binom{n}{2} }$ associated with the
matrix $A$, defined by $B_{e f} =  A_{e_1,e_2} \indc{ e_2=f_1} \indc{e_1 \neq f_2}$, where $e=(e_1,e_2)$ and $f=(f_1,f_2)$
are directed pairs of indices.   Let ${\Theta}^t \in \reals^{n(n-1)}$ denote the
messages on directed edges with ${\Theta}_e^t =\theta^t_{e_1 \to e_2}$.
Then,  \prettyref{eq:theta_update_ij} simply becomes   $ \Theta^t =  \mathbf{B}^t \mathbf{1}.$
To evaluate the performance of this method, we turn to the state evolution equations \prettyref{eq:state_evolution1} and \prettyref{eq:state_evolution2}, which yield
$\mu_t = \lambda^{t/2}$  and $\tau_t = 1$ for all $t \geq 1.$
Therefore,  by a simple variation of \prettyref{alg:MP} and   \prettyref{thm:almost_exactBP_submat},
if $\lambda > 1,$  the linear message passing algorithm can provide weak recovery.

%\nbr{JX: since we do not have converse results for linear message passing, it is inappropriate to
%conclude that optimized message passing exceeds the performance limit of spectral algorithms by a factor of $1/\eexp$.}

%\nb{JX: deleting the following. As shown in \cite{BordenaveLelargeMassoulie:2015dq} in the case the matrix $A$ is the incidence matrix of an Erdos-Renyi graph,  the spectral properties of the non-backtracking matrix closely match those of the original  matrix $A.$  It is therefore reasonable to take the linear message passing algorithm as a form of spectral method for the submatrix localization problem. To see the performance of this method, we turn to the state evolution equations, which yield $\mu_t = \lambda^{t/2}$  and $\tau_t = 1$ for all $t \geq 1.$ Therefore,  by a simple variation of \prettyref{alg:MP} and   \prettyref{thm:almost_exactBP_submat}, if $\lambda > 1,$  the linear message passing algorithm can provide weak recovery, and if $\lambda \leq 1,$ it cannot\nb{JX: I think we do not have converse results}. Hence,  the optimized message passing algorithm, which achieves weak recovery if $\lambda > 1/\eexp,$ exceeds the performance limit achievable by spectral algorithms. Two nonrigorous aspects of this argument are that the analysis in this paper applies only for any fixed, finite number of iterations, and the matrices $A$ and $\mathbf{B}$ are indeed different.}

For the submatrix {\em detection} problem, namely, testing $\mu=0$ (pure noise) versus $\mu>0$,
as opposed to support recovery, if $\lambda$ is fixed with $\lambda >1,$
a simple thresholding test based on the largest eigenvalue of the matrix $A$
provides detection error probability converging to zero \cite{FeralPeche2007},
while if $\lambda < 1$  no test based solely on the eigenvalues of $A$ can achieve vanishing probability of error \cite{MontanariReichmanZeitouni14}.   It remains, however, to establish a solid
connection between the detection and estimation problem for submatrix
localization for spectral methods.

\subsection{Computational barriers}
A recent line of work \cite{kolar2011submatrix,ma2013submatrix,ChenXu14,CLR15} has uncovered
a fascinating interplay between statistical optimality and computational efficiency for the \emph{recovery} problem
and the related \emph{detection} and \emph{estimation} problem.\footnote{The papers \cite{kolar2011submatrix,ma2013submatrix,ChenXu14,CLR15} considered the  biclustering version of the submatrix localization problem \prettyref{eq:model}.}
Assuming the hardness of the planted clique problem, rigorous computational lower bounds have been obtained  in \cite{ma2013submatrix,CLR15} through reduction arguments.
%the detection and estimation problem in \cite{ma2013submatrix}  and exact recovery problem in \cite{CLR15} via reduction arguments.
In particular, it is shown in \cite{ma2013submatrix} that when $K=n^{\alpha}$ for $0<\alpha<2/3$, merely achieving the information-theoretic limits of detection within any constant factor (let alone sharp constants) is as hard as
%requires computational resources that are powerful enough to solve
detecting the planted clique; the same hardness also carries over to exact recovery in the same regime.
Furthermore, it is shown that the hardness of estimating this type of matrix, which is both low-rank and sparse, highly depends on the loss function \cite[Section 5.2]{ma2013submatrix}. For example, for $K=\Theta(\sqrt{n})$, entry-wise thresholding attains an $O(\log n)$ factor of the minimax mean-square error; however, if the error is gauged in squared operator norm instead of Frobenius norm, attaining an $O(\sqrt{n}/\log n)$ factor of the minimax risk is as hard as solving planted clique.
Similar reductions have been shown in \cite{CLR15} for exact recovering of the submatrix of size $K=n^{\alpha}$
and the planted clique recovery problem for any $0<\alpha<1$.

%In sharp contrast to the computational barriers discussed in the previous paragraph,
The results in \cite{ma2013submatrix,CLR15} revealed that the difficulty of submatrix localization crucially depends on the size and planted clique hardness kicks in if $K = n^{1-\Theta(1)}$. In search of the exact phase transition point where statistical and computational limits depart,
  we further zoom into the regime of $K=n^{1-o(1)}$. We showed in~\cite{HajekWuXu_one_sdp15}
%For the exact recovery of the submatrix, it is shown in  \cite{CLR15}
%that as long as $K=n^{\alpha}$ for $0<\alpha<1$, achieving the information recovery limits is no easier than solving the planted clique recovery problem.
no computational gap exists in the regime $K=\omega(n/\log n)$, since a semi-definite programming relaxation of the maximum likelihood estimator can achieve the information limit for exact recovery with sharp constants.
%in the regime $K=\Theta(n)$
% a semi-definite programming relaxation (SDP)  of maximum likelihood estimation is shown in~\cite{HajekWuXuSDP14}
%to achieve the information limits for exact recovery with sharp constants. The subsequent work~\cite{HajekWuXu_one_sdp15} extends the optimality
%of SDP to the regime $K=\omega(n/\log n)$.
The current paper further pushes the boundary to $ K \ge \frac{n}{\log n} (\frac{1}{8 \eexp}+o(1))$, in which case the sharp information limits can be attained in nearly linear-time via message passing plus clean-up. However, as soon as $\limsup_{n\to \infty} K \log n/n  < \frac{1}{8 \eexp}$, there is a gap between
the information limits and the sufficient condition of message passing plus clean-up, given by $\lambda > 1/\eexp$.
For weak recovery, a similar departure emerges whenever $K=o(n)$.

%In fact, as discussed above,
%the information limits are also achievable in linear time by the thresholding algorithm on row sums followed by a linear time voting procedure.
%The main results in this paper suggest that a gap between
%information limits and performance limits of computationally feasible algorithms starts to emerge as soon as $\liminf_{n\to \infty} K \log n/n  < \frac{1}{8 %\eexp}$ for exact recovery and $K=o(n)$ for weak recovery.

\section{Justification of state evolution equations}   \label{sec:state_evolution}

In this section, we justify the state evolution equations by establishing the following key lemma.
The method of moments is used, closely following \cite{Deshpande12}.  \prettyref{rmk:differences}
describes the main differences between the analysis here and in \cite{Deshpande12}.

\begin{lemma} \label{lmm:approximatecentrallimit}
Let $f(\cdot, t)$ be a finite-degree polynomial for each $t \geq 0$. For each $n$,
let $W \in \reals^{n\times n} $ be defined in \prettyref{eq:models} with
$K$ and $\mu$ such that $\frac{K^2\mu^2}{n} \equiv \lambda$ for some $\lambda > 0$ and $\Omega(\sqrt{n}) \le K \le o(n).$
Let $A=W/\sqrt{n}$
and set $\theta^{0}_{i \to j} =0$.  Consider the message passing algorithm defined by
\eqref{eq:theta_update_ij} and \eqref{eq:theta_update_i}.
%Then for any bounded Lipschitz function $\psi: \reals \to \reals$,
%the following limits hold in probability:
%Denote by $\dKS$ the Kolmogorov distance between two distributions, \ie, the $L_\infty$-distance between the CDFs.
Denote the Kolmogorov-Smirnov distance between distributions $\mu$ and $\nu$ by $\dKS(\mu,\nu) \triangleq \sup_{x\in \reals} |\mu((-\infty,x]) - \nu((-\infty,x])|$.
Then as $n\diverge$,
\begin{align*}
\dKS\pth{ \frac{1}{|C^*|} \sum_{i \in C^\ast} \delta_{\theta_i^t}, \calN(\mu_t, \tau_t^2)} & \overset{p}{\to} 0, \\
\dKS\pth{ \frac{1}{n-|C^*|} \sum_{i \notin C^\ast} \delta_{\theta_i^t}, \calN(0,\tau_t^2)} & \overset{p}{\to} 0,
\end{align*}
where $\mu_t$ and $\tau_t$ are defined in \prettyref{eq:state_evolution1} and \prettyref{eq:state_evolution2}, respectively.
%\begin{align*}
%\lim_{n \to \infty} \frac{1}{K} \sum_{i \in C^\ast} \psi (\theta_i^t) & \overset{p}{=} \expect{ \psi ( \mu_t + \tau_t Z ) }, \\
%\lim_{n \to \infty} \frac{1}{n} \sum_{i \notin C^\ast} \psi (\theta_i^t) & \overset{p}{=}   \expect{ \psi ( \tau_t Z ) }.
%\end{align*}
\end{lemma}
We note that a version of the above lemma is proved in \cite{Deshpande12} by assuming $\mu=\Theta(1)$ and $K= \Theta( \sqrt{n} )$.
%\begin{proof}
Let $f(x, t) = \sum_{i=0}^d q_i^t x^i$ with $|q_i^t|  \le C$ for a constant $C$.
Let $\{A^t, t \ge 1\}$ be i.i.d.\  matrices distributed as $A$ conditional on $C^\ast$
and let $A^0=A$. We now define a sequence of vectors $\{ \xi^t, t \ge 1\}$ with $\xi^t \in \reals^n$ given by
\begin{align}
\xi^{t+1}_{i \to j} & = \sum_{\ell \in [n] \backslash \ \{i, j \} } A^t_{\ell i} f( \xi^t_{\ell \to i }, t), \quad \forall j \neq i \in [n]  \label{eq:defxi}  \\
\xi^{t+1}_i & = \sum_{\ell \in [n] \backslash \{i \} } A^t_{\ell i} f( \xi^t_{\ell \to i }, t) \nonumber\\
\xi^0_{i \to j} &  = 0.
%\xi^{t}_{i \to i} & = 0, \quad \forall i \in [n].
\end{align}
Note that in the definition of  $\xi^t$, fresh samples, $A^t,$ of $A$ are used at each iteration, and thus the moments of $\xi^t$ in
the asymptotic limit are easier to compute than those of $\theta^t$.     Use of the fresh samples $A^t$ does not make
the messages $(\xi^t_{i \to \ell} : i \in [n]\backslash\ell )$ independent for fixed $\ell \in [n]$ and fixed $t\geq 2$, because at $t=1$ the messages
sent by any one vertex to all other vertices are statistically dependent, so at $t=2$ the messages sent by all vertices are
statistically dependent.   However, we can take advantage of the fact that the contribution of each individual message is small
in the limit as $n\to \infty$.      Hence, we first prove that $\xi^t$ and $\theta^t$ have the
same moments of all orders as $n \to \infty,$ and then prove the lemma using the method of moments.

The first step is to represent $(\theta^t_{i \to j}, \theta^{t}_i)$ and $(\xi^t_{i \to j}, \xi^{t}_i)$ as sums over a family of finite rooted labeled trees
as shown by \cite[Lemma 3.3]{Deshpande12}. We next introduce this family in detail.
We shall consider rooted trees $T$ of the following form.   All edges are directed towards the root.
The set of vertices and the set of (directed) edges in  a tree $T$  are denoted by $V(T)$ and $E(T)$, respectively.
Each vertex has at most $d$ children.  The set of leaf vertices of $T$, denoted by $L(T),$ is the set of vertices with no children.  Every vertex in the tree has a {\em label} which includes the {\em type} of the vertex, where the
types are selected from $[n].$
The label of the root vertex consists of the type of the root vertex,  and for every non-root vertex the
label has two arguments,
where the first argument in the label is the type of the vertex (in $[n]$), and the second one is the {\em mark} (in $\{0, \ldots , d\}$).
For a vertex $v$ in $T$, let $\ell(v)$ denote is type, $r(v)$ its mark (if $v$ is not the root), and $|v|$ its distance from the
root in $T$. For clarity, we
restate the definition of family of rooted labeled trees introduced in \cite[Definition 3.2]{Deshpande12}.
\begin{definition}
Let $\calT^t$ denote the family of labeled trees $T$ with exactly $t$ generations satisfying the conditions:
\begin{enumerate}
\item The root of $T$ has degree 1.
\item Any path $(v_1, v_2, \ldots, v_k)$ in the tree is non-backtracking, \ie, the types $\ell(v_i), \ell(v_{i+1}), \ell(v_{i+2})$ are
distinct for all $i, k$.
\item For a vertex $u$ that is not the root or a leaf, the mark $r(u)$ is set to the number of children of $v$.
\item Note that $t= \max_{v \in L(T)} |v|$. All leaves $u$ with $|u| \le t-1$ have mark $0$.
\end{enumerate}
Let $\calT^t_{i \to j} \subset \calT^t$ be the subfamily satisfying the following additional conditions:
\begin{enumerate}
\item The type of the root is $i$.
\item The root has a single child with type distinct from $i$ and $j$.
\end{enumerate}
Similarly, let $\calT^t_i \subset \calT^t$ be the subfamily satisfying the following:
\begin{enumerate}
\item The type of the root is $i$.
\item The root has a single child with type distinct from $i$.
\end{enumerate}
\end{definition}
We point out that under the above definition,  a vertex of a tree in $\calT^t$
can have siblings of the same type and mark. Also two trees in $\calT^t$ are considered
to be the same if and only if the labels of all nodes are the same, with the understanding that
the order of the children of any given node matters.   In addition, the mark of a leaf $u$ with $|u|=t$
is not specified and can possibly take any value in $\{0, \ldots , d\}$.
The following lemma is proved by induction on $t$ and the proof can be found in  \cite[Lemma 3.3]{Deshpande12}.
\begin{lemma}\label{lmm:treerepresentation}
\begin{align*}
\theta^t_{i \to j} &= \sum_{T \in \calT^t_{i \to j} } A(T) \Gamma(T, \mathbf{q}, t) \theta(T), \\
\theta^t_i & = \sum_{T \in \calT^t_{i} } A(T) \Gamma(T, \mathbf{q}, t) \theta(T),
\end{align*}
where\footnote{Often the initial messages for message passing are taken, with some abuse of notation,
to have the form $\theta_{i\to j}^0=\theta_i^0$ for all $j$, and then only the $n$ variables $\theta^0_i$ need
to be specified.  In that case, the expression for $\theta(T)$ simplifies
to  $\theta(T) \triangleq \prod_{u \in L(T) } (\theta^0_{\ell(u)} )^{r(u)}. $
%As usual, the convention $0^0=1$ is in force.   In this paper we take
%the initial messages to all be zero, so $\theta(T)$ is nonzero only for those trees such that the
%marks of all leaf nodes are zero.
}
\begin{align*}
A(T) & \triangleq \prod_{u \to v \in E(T) } A_{\ell(u), \ell(v)}, \\
\Gamma(T, \mathbf{q}, t) & \triangleq \prod_{u \to v \in E(T) } q^{t-|u|}_{r(u)}, \\
%\theta(T) &\triangleq \prod_{u \in L(T) } (\theta^0_{\ell(u)} )^{r(u)}.  \\
\theta(T) &\triangleq \prod_{u \to v \in E(T) : u \in L(T) } (\theta^0_{\ell(u)\to \ell(v) } )^{r(u)}.
\end{align*}
Similarly,
\begin{align*}
\xi^t_{i \to j} &= \sum_{T \in \calT^t_{i \to j} } \bar{A}(T) \Gamma(T, \mathbf{q}, t) \theta(T), \\
\xi^t_i & = \sum_{T \in \calT^t_{i} } \bar{A}(T) \Gamma(T, \mathbf{q}, t) \theta(T),
\end{align*}
where
\begin{align*}
\bar{A}(T) & \triangleq \prod_{u \to v \in E(T) } A^{t-|u| }_{\ell(u), \ell(v)}.
\end{align*}
\end{lemma}

Since the initial messages are zero, $f(\theta_{i\to j}^0 , 0 ) = q_0^0.$  {\em Thus,
for notational convenience in what follows, we can assume without loss of generality that
$f(x,0) \equiv q_0^0$, i.e., $f(x,0)$ is a degree zero polynomial.}     With this assumption, it follows
that for a labeled tree $T \in {\cal T}^t,$   $\Gamma(T, \mathbf{q}, t)=0$ unless the mark
of every leaf of  $T$ is zero.     If the mark of every leaf is zero, then $\theta(T)=1,$ because
in this case $\theta(T)$ is a product of terms of the form $0^0,$  which are all one, by convention.
Therefore,  $\Gamma(T, \mathbf{q}, t)\theta(T)=\Gamma(T, \mathbf{q}, t)$ for all $T \in {\cal T}_t.$
Consequently, the factor $\theta(T)$  can be dropped from the representations of
$\theta_{i\to j}^t,$  $\theta_i^t,$  $\xi_{i\to j}^t,$  and $\xi_i^t$ given in \prettyref {lmm:treerepresentation}.
Applying \prettyref{lmm:treerepresentation}, we can prove that all finite moments of $\theta_i^t$
and $\xi_i^t$ are asymptotically the same.

\begin{lemma}\label{lmm:momentmatching}
For any $ t \ge 1$, there exists a constant $c$ independent of $n$ and dependent on $m, t, d, C$ such that
for any $i \in [n]$:
\begin{align*}
\big| \expect{(\theta_i^t)^m } - \expect{(\xi_i^t)^m } \big| \le c n^{-1/2}.
\end{align*}
\end{lemma}
\begin{proof}
As explained just before the lemma, the assumption that $f(x,0) \equiv q_0^0$ implies that
the factor $\theta(T)$ can be dropped in the representations given in \prettyref{lmm:treerepresentation}.
Therefore, it follows from \prettyref{lmm:treerepresentation} that for $t\geq 1,$
\begin{align*}
\expect{(\theta_i^t)^m } = \sum_{T_1, \ldots, T_m \in  \calT^t_{i}}  \prod_{\ell=1}^m \Gamma(T_\ell, \mathbf{q}, t) \expect{\prod_{\ell=1}^m A(T_\ell ) }, \\
\expect{(\xi_i^t)^m } = \sum_{T_1, \ldots, T_m \in  \calT^t_{i}}  \prod_{\ell=1}^m \Gamma(T_\ell, \mathbf{q}, t) \expect{\prod_{\ell=1}^m \bar{A}(T_\ell ) }
\end{align*}
Because the coefficients in the polynomial are bounded by $C$ and there are $m$ trees with each tree containing at most
$1+d+ \cdots + d^{t-1} \leq  (d+1)^t$ edges,
$ | \prod_{\ell=1}^m  \Gamma(T_\ell, \mathbf{q}, t) | \le C^{m (d+1)^t }$. Therefore, it suffices to show
\begin{align*}
 \sum_{T_1, \ldots, T_m \in  \calT^t_{i}}  \Bigg|  \expect{\prod_{\ell=1}^m A(T_\ell ) } -  \expect{\prod_{\ell=1}^m \bar{A}(T_\ell ) } \Bigg| \le cn^{-1/2}.
\end{align*}
In the following, let $c$ denote a constant only depending on $m, t, d$ and its value may change line by line.
Let $\phi(T)_{rs}$ denote the number of occurrences of edges $(u \to v)$ in the tree $T$ with types $\ell(u),\ell(v)=\{r, s\}$.
Let $G$ denote the undirected graph obtained by identifying the vertices of the same type in the tuple of trees $T_1, \ldots, T_m$ and
removing the edge directions. Let $E(G)$ denote the edge set of $G$. Then an edge $(r,s)$ is  in $E(G)$ if and only if $\sum_{\ell=1}^m \phi(T_\ell)_{rs} \ge 1$,
 \ie, the number of times covered is at least one. Let $G_1$ denote the restriction of $G$ to the vertices in $C^\ast$
and $G_2$ the restriction of $G$ to the vertices in $[n] \backslash C^\ast$. Let $E(G_1)$ and $E(G_2)$ denote the edge set of $G_1$ and $G_2$, respectively.
Let $E_J$ denote the set of edges in $G$ with one endpoint in $G_1$ and the other end point in $G_2$.
We partition set $\{ (T_1, \ldots, T_m) : T_\ell \in  \calT^t_{i} \}$ as a union of four disjoint sets $  Q \cup R_1 \cup R_2 \cup R_3$, where
\begin{enumerate}
\item $Q$ consists of $m$-tuples of trees $(T_1, \ldots, T_m)$ such that there exists an edge  $(r, s)$ in $E(G_2) \cup E_J$ which is covered exactly once.
\item $R_1$ consists of $m$-tuples of trees $(T_1, \ldots, T_m )$ such that all edges in $E(G_2) \cup E_J$ are covered at least twice and
at least one of them  is covered at least $3$ times.
\item $R_2$ consists of $m$-tuples of trees $(T_1, \ldots, T_m )$ such that each edge in $E(G_2) \cup E_J$ is covered exactly twice and the graph $G$ contains a cycle.
\item $R_3$ consists of $m$-tuples of trees $(T_1, \ldots, T_m)$ such that  each edge in $E(G_2) \cup E_J$ is covered exactly twice and the graph $G$ is a tree.
\end{enumerate}
Fix any $(T_1, \ldots, T_m) \in Q$ and let $(r,s)$ be an edge in $E(G_2) \cup E(J)$ which is covered exactly once. Since $\expect{A_{rs}} =0$ and $A_{rs}$ appears
in the product $\prod_{\ell=1}^m A(T_\ell )$ once,  it follows that
$\expect{\prod_{\ell=1}^m A(T_\ell ) } =0$. Similarly, $\expect{\prod_{\ell=1}^m \bar{A}(T_\ell ) } =0$. Therefore, it is sufficient to show that for $j=1,2,3$,
\begin{align*}
 \sum_{(T_1, \ldots, T_m) \in  R_j }  \Bigg|  \expect{\prod_{\ell=1}^m A(T_\ell ) } -  \expect{\prod_{\ell=1}^m \bar{A}(T_\ell ) } \Bigg| \le cn^{-1/2}.
\end{align*}
First consider $R_1$.   Further,  divide $R_1$ according to the total number of edges in $T_1, \ldots, T_m$
and the number of edges in $E(G_1)$ which are covered exactly once. In particular,
for $\alpha=1, \ldots, m(d+1)^t$ and $k=0, 1, \ldots, \alpha$,  let $R_{1,\alpha,k}$ denote the subset of $R_1$ consisting of $m$-tuples of trees $T_1, \ldots, T_m $ such that
there are $\alpha$ edges in $T_1, \ldots, T_m$  and  there are $k$ edges in $E(G_1)$ which are covered exactly once. It suffices to show that
\begin{align}
 \sum_{(T_1, \ldots, T_m) \in  R_{1,\alpha,k} }  \Bigg|  \expect{\prod_{\ell=1}^m A(T_\ell ) } -  \expect{\prod_{\ell=1}^m \bar{A}(T_\ell ) } \Bigg|  \le cn^{-1/2}.
 \label{eq:boundR1}
\end{align}

Fix $\alpha, k$ and an $m$-tuple of trees $(T_1, \ldots, T_m) \in R_{1, \alpha, k}$. Then
\begin{align}
\Bigg| \expect{ \prod_{\ell=1}^m  A(T_\ell) } \Bigg| &= \Bigg| \expect{ \prod_{j < j' } (A_{j j'} ) ^{\sum_{\ell=1}^m \phi (T_\ell)_{jj'} } } \Bigg| = \prod_{j<j'}
\Bigg| \expect{ (A_{j j'}) ^{\sum_{\ell=1}^m \phi (T_\ell)_{jj'} } } \Bigg|  \nonumber   \\
& = \left( \frac{\mu}{\sqrt{n} } \right)^{k} \prod_{j<j':  \sum_{\ell=1}^m \phi (T_\ell)_{jj'}  \ge 2 }  \Bigg|\expect{ (A_{j j'}  )^{\sum_{\ell=1}^m \phi (T_\ell)_{jj'} } } \Bigg| \nonumber  \\
& \le  \left( \frac{\mu}{\sqrt{n} } \right)^{k}  \prod_{j<j':  \sum_{\ell=1}^m \phi (T_\ell)_{jj'}  \ge 2 } \expect{ | A_{j j'} | ^{\sum_{\ell=1}^m \phi (T_\ell)_{jj'} } }  \nonumber  \\
& \le c \left( \frac{\mu}{\sqrt{n} } \right)^{k} \prod_{j<j':  \sum_{\ell=1}^m \phi (T_\ell)_{jj'}  \ge 2 }   \left( \frac{1}{\sqrt{n} } \right)^{\sum_{\ell=1}^m \phi (T_\ell)_{jj'} }  \nonumber  \\
&= c \mu^k n^{-\alpha /2},  \label{eq.sincle_tree_bnd}
\end{align}
where the last inequality follows because for $1\leq p \leq m(d+1)^t$, if  $Z$ is a standard Gaussian random variable then
$\expect{  \bigg| \frac{Z}{\sqrt{n}}\bigg|^p } \leq cn^{-p/2}$  and $\expect{  \bigg| \frac{Z+ \mu}{\sqrt{n}}\bigg|^p } \leq cn^{-p/2}$
where $c=\expect{|Z+\mu_{\max} |^{ m(d+1)^t}},$  and $\mu_{\max}$ is an upper bound on $\mu$ for all $n$, which is finite by
the assumptions.\footnote{This is where the assumption $K=\Omega(\sqrt{n})$ is used because $\frac{K^2\mu^2}{n}$ is assumed to be a constant $\lambda$.}

We consider breaking $R_{1,\alpha, k}$ down into a large number of smaller sets.   While large,
the number of these smaller sets depends on $m,t,d$, but not on $n.$    One way to describe
these sets is that they are equivalence classes for the following equivalence relation over  $R_{1,\alpha, k}:$
Two $m$-tuples in $R_{1,\alpha, k}$ are equivalent if there is a
permutation of  the set of types $[n]$ such that $i$ maps to $i$,  $C^*$ maps $C^*$,
and the second $m$-tuple is obtained by applying the
permutation to the types of the vertices of the first $m$-tuple.
In particular the marks of the two $m$-tuples must be the same.

 Another way to think about these equivalence classes is the following.   Given an $m$-tuple
$(T_1, \ldots , T_m)$ in $R_{1,\alpha, k}$,   form the graph $G$ as described above.   Let the type
of each vertex in $G$ be the common type of the vertices it represents in the $m$-tuple.
For convenience, refer to the vertex of $G$ with type $i$ as vertex $i.$
Let $V_1$ be the set of vertices in $G$ with types in $C^*\backslash \{i\}$ and $V_2$  be
the set of vertices in $G$ with types in $([n]-C^*)\backslash \{i\}.$
Record $V_1$ and $V_2,$
and then erase the types of the vertices in $G\backslash \{i\}.$
Then the class of $m$-tuples equivalent to  $(T_1, \ldots , T_m)$ is the set
of $m$-tuples in $R_{1,\alpha, k}$ that can be obtained by assigning distinct types
to the vertices of $G$ (which are inherited by the corresponding vertices
in the $m$-tuple of trees) consistent with the specified vertex of type
$i$ and sets $V_1$ and $V_2.$    Note that the marks (as opposed to the types)
of all $m$-tuples in the equivalence class are the same as the
marks on the representative $m$-tuple.

The number of equivalence classes
is bounded by a function of $m,t,d$ alone, because
the total number of vertices of an $m$-tuple
$(T_1, \ldots , T_m)$ is bounded independently of $n$, therefore so are
the number of ways to partition these vertices to be identified with each other to
form vertices in a graph $G,$   along with binary designations
on the subsets of the partitions of whether the types of the vertices in
the subset are in $C^*$ or not (i.e. determining $V_1$ and $V_2$)
and the number of ways to assign
marks to the vertices of the trees.
Not all partitions with binary designations on the partition
subsets correspond to valid equivalence classes because valid
partitions must respect the non-backtracking rule and they
should have all the root vertices in the same partition set.
Also, whether the type of the subset of the partition containing
the root vertices corresponds to a type in $C^*$ or not is already
determined by $i.$    The purpose here is only to verify that the
number of such equivalence classes is bounded above by
a function of  $m,t,d,$  independently of $n.$

Hence, fix such an  equivalence class $S \subset R_{1,\alpha, k}.$
It follows from \eqref{eq.sincle_tree_bnd}
\begin{align}
\sum_{(T_1, \ldots, T_m) \in  S }  \Bigg| \expect{\prod_{\ell=1}^m A(T_\ell ) } \Bigg| \le   c \mu^k  n^{-\alpha /2} | S  |. \label{eq:countingS}
\end{align}
Note that $|S| \leq K^{n_1}n^{n_2},$   where $n_i=|V_i|$ for $i=1,2,$ because there are at most $K$ choices of
type for each vertex in $V_1$ and fewer than $n$ choices of type for each vertex in $V_2.$
The graph $G$ is connected (because all the trees have a root of type $i$),
so $n_1+n_2$ (the number of vertices of $G$ minus one) is less than or equal to the number of edges in $G.$
The number of edges in $G$ is at most $k + \frac{\alpha - k - 1}{2}$ because
there are $k$ edges in $G$ covered once, and the rest are covered at least twice, with one edge covered at least three times.
So $n_1+n_2 \leq  k + \frac{\alpha - k - 1}{2}.$
Also, since $k$ of the edges in $G$ have both endpoints in $C^*$,  and the vertices of $V_2$ have types in  $[n]-C^*,$  there are at most $\frac{\alpha - k - 1}{2}$ edges
in $G$ with at least one endpoint in $V_2.$  Therefore, since $G$ is connected, $n_2 \leq \frac{\alpha - k - 1}{2}$;
otherwise, there must exist a node in $V_2$ which has no neighbors in $G$, contradicting the connectedness of $G$.
The bound $K^{n_1}n^{n_2}$ is maximized subject to $n_1+n_2 \leq  k + \frac{\alpha - k - 1}{2}$ and
$n_2 \leq \frac{\alpha - k - 1}{2}$ by letting equality hold in both constraints,  yielding  $|S| \leq  (K)^kn^{ \frac{\alpha - k - 1}{2} }.$
Combining with \eqref{eq:countingS} shows that
\begin{align}
\sum_{(T_1, \ldots, T_m) \in  S }  \Bigg| \expect{\prod_{\ell=1}^m A(T_\ell ) } \Bigg| \le   c \mu^k  n^{-\alpha /2}  K^kn^{ \frac{\alpha - k - 1}{2} }
=   \left( \frac{\mu K}{\sqrt{n}}\right)^k n^{-1/2}   \leq  c n^{-1/2},
\end{align}
where we've used the fact that $\frac{\mu K}{\sqrt{n}}$ is bounded independently of $n.$
In a similar way, it can be shown that
 \begin{align*}
 \sum_{(T_1, \ldots, T_m) \in  S}  \Bigg| \expect{\prod_{\ell=1}^m \bar{A}(T_\ell ) } \Bigg|\le c n^{-1/2}
 \end{align*}
 and thus
\begin{align}
\sum_{(T_1, \ldots, T_m) \in  S }  \Bigg|  \expect{\prod_{\ell=1}^m A(T_\ell ) } -  \expect{\prod_{\ell=1}^m \bar{A}(T_\ell ) } \Bigg|  \le cn^{-1/2}.\end{align}
Since the number of equivalence classes $S$ does not depend on $n,$  \eqref{eq:boundR1} follows.

 Next consider $R_2$. The previous argument carries over with a minor adjustment.
 In particular, define $R_{2,\alpha,k}$ accordingly as $R_{1,\alpha,k}$ and then consider
 an equivalence class $S \subset R_{2,\alpha,k}$  corresponding to some representative
 $m$-tuple in $R_{2,\alpha,k}.$  Let $G$ and the partition of its vertices into $\{i\}$, $V_1,$ and $V_2$
 be determined by the $m$-tuple as before.
The number of edges in $G$ is at most $k + \frac{\alpha - k}{2}$ because
there are $k$ edges in $G$ covered once, and the rest are covered at least twice.
Since $G$ has $n_1+n_2+1$ vertices, is connected, and has a cycle, $n_1+n_2 $ is less than or equal to the number
of edges of $G$ minus one, so $n_1+n_2 \leq k + \frac{\alpha - k-2}{2}.$
Also, since $k$ of the edges in $G$ have both endpoints with types in $C^*$,  and $V_2$ has types in $[n]-C^*,$  there are at most $\frac{\alpha - k }{2}$ edges
in $G$ with at least one endpoint in $V_2.$  Therefore, since $G$ is connected, $n_2 \leq \frac{\alpha - k}{2}$.
The bound $K^{n_1}n^{n_2}$ is maximized subject to these constraints
 by letting equality hold in both constraints,  yielding
 $|S| \leq  K^{k-1}n^{ \frac{\alpha - k}{2} }.$   So
$|S| \mu^k  n^{ -\alpha/2 } \leq
 \left( \frac{\mu K}{\sqrt{n}}\right)^k/K \leq  c /K \leq  cn^{-1/2},$
 and the reminder of the proof for bounding the contribution of $R_2$ is the
 same as for $R_1$ above.

Finally, consider $R_3$.   It suffices to establish the following claim.
The claim is that for any $m$-tuple such that $G$ has no cycles,
if two directed edges $(a \to b)$ and $(c\to d)$ map to the same edge in $G$, then they
are at the same level in their respective trees (their trees might be the same).   Indeed,
if the claim is true, then  for any $m$-tuple $(T_1, \ldots , T_m)$ in $R_3$ and any
 pair $\{r,s\}\subset  [n] $, $A_{rs}^t$ appears in $\prod_{\ell=1}^m \bar{A}(T_{\ell})$
 for at most one value of $t,$ so that
$ \expect{\prod_{\ell=1}^m A(T_\ell ) } =  \expect{\prod_{\ell=1}^m \bar{A}(T_\ell ) }.$

We now prove the claim.   Let $\{r,s\}$ denote the edge in $G$ covered by both $(a \to b)$
and $(c\to d)$, i.e.  $\{\ell(a), \ell(b) \} = \{\ell(c), \ell(d) \} = \{r, s\}.$     First consider
the case that $\ell(b)=\ell(d).$
Let $u_1, \ldots , u_k$ denote the directed path in the tree containing $b$ that goes from $b$ to the
root of that tree, so $b=u_1$ and $u_k$ is the root of the tree.   Since there are no cycles in $G$,
and hence no cycles in the set of edges $\{ \{\ell(u_1),\ell(u_2)\}, \ldots  , \{\ell(u_{k-1}),\ell(u_k)\} \},$
(viewed as a simple set, i.e. with duplications removed)  it follows from the non-backtracking property
that $\ell(u_1), \ldots , \ell(u_k)$ are distinct vertices in $G.$     That is,
$(\ell(u_1), \ldots , \ell(u_k))$ is a simple path in $G$.   Similarly, let $v_1, \ldots , v_{k'}$ denote
the path in the tree containing $d$ that goes from $d$ to the root of that tree, so $d=v_1$
and $v_{k'}$ is the root of that tree.
As for the first path,  $(\ell(v_1), \ldots , \ell(v_{k'}))$ is also a simple path in $G.$
Since the roots of all $m$ trees have the same type, $\ell(u_k)$ and $\ell(v_{k'})$  are the
same vertex in $G.$    Therefore,
$( \ell(u_1),  \ldots  , \ell(u_{k}), \ell(v_{k'-1}), \ldots  , \ell(v_1)  )$ is a closed walk
in $G$ that is the concatenation of two simple paths.  Since $G$ has no cycles those
two paths must be reverses of each other.   That is,  $k=k'$ and $\ell(u_j)=\ell(v_j)$ for
all $j$, and hence $(a \to b)$ and $(c\to d)$ are at the same level in their trees.

Consider the remaining case, namely, that $\ell(b)=\ell(c).$    Let $u_1, \ldots , u_k$
be defined as before, and let $v_1, \ldots , v_{k'}$ denote the path in the tree
containing $c$ that goes from $c$ to the root of that tree, so $c=v_1$, $d=v_2$,  and $v_{k'}$
is the root of that tree.   Arguing as before yields that $k=k'$ and
 $\ell(u_j)=\ell(v_j)$ for $1\leq j \leq k.$
 %\nb{JX: Note that $v_2$ cannot be the root; otherwise since $\ell(v_2)=\ell(d)=\ell(a)$,
% $a$ will be the root, contradicting the fact that $(a \to b)$ is an edge  directed towards the root.}
Note that $k'\geq 2$ and so $k\geq 2$
 and $\ell(u_2)=\ell(v_2)=\ell(d)=\ell(a).$   Thus, the types along the directed path $a \to u_1 \to u_2$
 within one of the trees violates the non-backtracking property, so the case
 $\ell(b)=\ell(c)$  cannot occur.    The claim is proved.   This completes the
 proof of Lemma 3.  \end{proof}

The second step is to compute the moments of $\xi^t$ in the asymptotic limit $n\to \infty$. We need the following lemma to ensure
that all moments of $\xi^t$ are bounded by a constant independent of $n$.
\begin{lemma}\label{lmm:momentbound}
For any $ t \ge 1$, there exists a constant $c$ independent of $n$ and dependent on $m, t, d, C$ such that
for any $i, j \in [n]$
\begin{align*}
| \expect{(\xi_{i \to j}^t)^m } | \le c, \quad | \expect{(\xi_{i}^t)^m } | \le c.
\end{align*}
\end{lemma}
\begin{proof}
We prove the claim for $\xi_i^t$; the claim for $\xi_{i\to j}^t$ follows by the similar argument. Since $\xi^0_i=\theta^0_i=0$ for all $i \in [n]$, it follows from \prettyref{lmm:treerepresentation} that
\begin{align*}
\expect{(\xi_i^t)^m } = \sum_{T_1, \ldots, T_m \in  \calT^t_{i}}  \prod_{\ell=1}^m \Gamma(T_\ell, \mathbf{q}, t) \expect{\prod_{\ell=1}^m \bar{A}(T_\ell ) }
\end{align*}
Following the same argument as used for proving \prettyref{lmm:momentmatching}, we can partition set $\{ (T_1, \ldots, T_m) : T_\ell \in  \calT^t_{i} \}$ as a union of four disjoint sets $  Q \cup R_1 \cup R_2 \cup R_3$, and show that
\begin{align*}
 \sum_{T_1, \ldots, T_m \in Q }  \prod_{\ell=1}^m \Gamma(T_\ell, \mathbf{q}, t) \expect{\prod_{\ell=1}^m \bar{A}(T_\ell ) } =0,
\end{align*}
and
\begin{align*}
  \sum_{T_1, \ldots, T_m \in R_1\cup R_2 }  \Bigg| \prod_{\ell=1}^m \Gamma(T_\ell, \mathbf{q}, t) \Bigg| \Bigg| \expect{\prod_{\ell=1}^m \bar{A}(T_\ell ) } \Bigg| \le c n^{-1/2}.
\end{align*}
Hence, we only need to check $R_3$. Again divide $R_3$ according to the total number of edges in $T_1, \ldots, T_m$
and the number of edges in $E(G_1)$ which are covered exactly once. In particular, $R_3=\cup_{1\le \alpha\le m(d+1)^t, 0 \le k \le \alpha} R_{3,\alpha,k}$,
where $R_{3,\alpha,k}$ is defined in the similar way as $R_{1, \alpha, k}$.   Furthermore,  consider dividing $R_{3,\alpha,k}$ into
a number of equivalence classes, the number of which depends only on $m,t,d,$ as in the proof of \prettyref{lmm:momentmatching}.
To prove the lemma, it suffices to show that for any such equivalence class $S,$
\begin{align*}
\sum_{(T_1, \ldots, T_m) \in S } \Bigg| \expect{\prod_{\ell=1}^m \bar{A}(T_\ell ) } \Bigg| \le c.
\end{align*}
In the proof of \prettyref{lmm:momentmatching},  we have shown that
\begin{align*}
\Bigg| \expect{ \prod_{\ell=1}^m  \bar{A}(T_\ell) } \Bigg| \le c \mu^k n^{-\alpha /2},
\end{align*}
so
\begin{align}
\sum_{(T_1, \ldots, T_m) \in  S  }  \Bigg| \expect{\prod_{\ell=1}^m \bar{A}(T_\ell ) }
\Bigg| \le   c \mu^k  n^{-\alpha /2} | S |. \label{eq:countinglabeledtree}
\end{align}
We can bound $| S |$ in the similar way as we did for $|R_{1,\alpha,k}|$, with the only adjustment being
we cannot use the assumption that there exists at least one edge which is covered at least three times.
Fix a representative $m$-tuple $(T_1, \ldots,  T_m)$ for $S$  and let $G$ and the partition of the vertices of $G$:
 $\{i\},$  $V_1$, $V_2$ , be as in the proof of \prettyref{lmm:momentmatching}.
Let $n_i=|V_i|$ as before.   There are  $n_1+n_2+1$ vertices in the connected graph $G$
and, since the $m$-tuple is in  $R_{3,\alpha,k}$,  there are at most $k + \frac{\alpha - k}2$ edges in $G$, so $n_1+n_2 \leq  k + \frac{\alpha - k}2.$
Also, at most $\frac{\alpha - k}2$ edges of $G$ have at least one endpoint in $V_2$ so $n_2 \leq  \frac{\alpha - k}2.$
Therefore, $|S| \leq K^{n_1}n^{n_2} \leq K^kn^{\frac{\alpha - k}2}.$   It follows that
 \begin{align*}
\sum_{(T_1, \ldots, T_m) \in S } \Bigg|  \expect{\prod_{\ell=1}^m \bar{A}(T_\ell ) } \Bigg| \le c \mu^k  n^{-\alpha /2} K^{k} n^{\frac{\alpha-k}{2}} = c \left( \frac{K \mu}{\sqrt{n}} \right)^{k}  \le c,
\end{align*}
and the proof is complete.
\end{proof}

We also need the following lemma to show the convergence of $\frac{1}{|C^\ast| } \sum_{i \in C^\ast} (\xi_i^t )^m $  in probability using the Chebyshev inequality.
\begin{lemma}\label{lmm:momentvariance}
For any $ t \ge 1$, $ m \ge 1$ and $i, j \in [n]$,
\begin{align*}
\lim_{n \to \infty} \var \left( \frac{1}{K} \sum_{i \in C^\ast} (\xi_i^t )^m  \right) & =0 \\
\lim_{n \to \infty} \var \left( \frac{1}{K } \sum_{\ell \in C^\ast  } (\xi_{\ell \to i}^t )^m  \right) & =0 \\
\lim_{n \to \infty} \var \left( \frac{1}{n} \sum_{i \in [n]\backslash C^\ast} (\xi_i^t )^m  \right) & =0 \\
\lim_{n \to \infty} \var \left( \frac{1}{n} \sum_{\ell \in [n]\backslash C^\ast  } (\xi_{\ell \to i}^t )^m  \right) & =0,
\end{align*}
where the same also holds when replacing $\xi^t$ by $\theta^t$.
\end{lemma}
\begin{proof}
We prove the first claim; the other claim follows by a similar argument. Notice that
\begin{align*}
\var\left( \frac{1}{K} \sum_{i \in C^\ast} (\xi_i^t )^m  \right) = \frac{1}{K^2} \sum_{i,j \in C^\ast}
\left( \expect{(\xi_i^t )^m (\xi_j^t )^m } - \expect{(\xi_i^t )^m } \expect{(\xi_j^t )^m } \right).
\end{align*}
There are $K$ diagonal terms with $i=j$ in the last displayed equation
and each diagonal term is bounded by a constant independent of $n$ in view of \prettyref{lmm:momentbound}.
Hence, to prove the claim, it suffices to consider the cross terms. Since there are $\binom{K}{2}$ cross terms, we only need to show that
for each cross term with $i \neq j$,  $ \expect{(\xi_i^t )^m (\xi_j^t )^m } - \expect{(\xi_i^t )^m } \expect{(\xi_j^t )^m } $ converges to $0$ as $n \to \infty$.
Using the tree representation as shown by \prettyref{lmm:treerepresentation} yields
\begin{align*}
& \big| \expect{(\xi_i^t )^m (\xi_j^t )^m } - \expect{(\xi_i^t )^m } \expect{(\xi_j^t )^m }  \big|   \\
& \le c \sum_{T_1, \ldots, T_m \in \calT^t_{i}, T'_1, \ldots, T'_m \in  \calT^t_{j}}
 \left(   \expect{  \prod_{\ell =1}^m \bar{A}(T_\ell )  \bar{A}(T'_\ell ) } -
\expect{\prod_{\ell=1}^m \bar{A}(T_\ell )} \expect{\prod_{\ell=1}^m \bar{A}(T'_\ell )} \right),
\end{align*}
where $c$ is a constant independent of $n$ and dependent of $m,t,d$.
As in the proof of \prettyref{lmm:momentmatching}, let $G$ denote the undirected simple graph
obtained by identifying vertices of the same type in the trees $T_1, \ldots, T_m,T'_1, \ldots, T'_m$
and  removing the edge directions. Let $E(G)$ denote the edge set of $G$.
Let $G_1$ denote the restriction of $G$ to the vertices in $C^\ast$
and $G_2$ the restriction of $G$ to the vertices in $[n] \backslash C^\ast$. Let $E(G_1)$ and $E(G_2)$ denote the edge set of $G_1$ and $G_2$, respectively.
Let $E_J$ denote the set of edges in $G$ with one endpoint in $G_1$ and the other end point in $G_2$.
Let $n(G_1)$ and $ n(G_2)$ denote the number of vertices in $G_1$ and $G_2$, respectively, not counting the vertices $i$ and $j$.
Notice that roots of $T_1, \ldots, T_m$ have type $i$ and  roots of $T'_1, \ldots, T'_m$ have type $j$, so
either $G$ is disconnected with one component containing $i$ and the other component containing $j$, or $G$ is connected.
In the former case, there is no edge $(r,s) \in E(G)$ which is covered by $T_1, \ldots, T_m$ and $T'_1, \ldots, T'_m$ simultaneously
and thus $ \expect{  \prod_{\ell =1}^m \bar{A}(T_\ell )  \bar{A}(T'_\ell ) } = \expect{\prod_{\ell=1}^m \bar{A}(T_\ell )} \expect{\prod_{\ell=1}^m \bar{A}(T'_\ell )}$.
In the latter case, \ie, $G$ is connected. We partition set $\{ (T_1, \ldots, T_m, T'_1, \ldots, T'_m) : T_\ell \in  \calT^t_{i} , T'_\ell \in \calT^t_j \}$ as a union of two
disjoint sets $  Q \cup R$, where
\begin{enumerate}
\item $Q$ consists of $2m$-tuples of trees such that $G$ is connected and there exists an edge  $(r, s)$ in $E(G_2) \cup E_J$ which is covered exactly once.
\item $R$ consists of $2m$-tuples of trees such that $G$ is connected and all edges in $E(G_2) \cup E_J$ are covered at least twice.
\end{enumerate}
If $(T_1, \ldots, T_m, T'_1, \ldots, T'_m) \in Q$, then $ \expect{  \prod_{\ell =1}^m \bar{A}(T_\ell )  \bar{A}(T'_\ell ) }=0$ and $\expect{\prod_{\ell=1}^m \bar{A}(T_\ell )} \expect{\prod_{\ell=1}^m \bar{A}(T'_\ell )}=0$. We are left to check $R$.
Following the argument used in  \prettyref{lmm:momentmatching},
further divide $R$ according to the total number of edges in trees
and the number of edges in $E(G_1)$ which is covered exactly once. In particular, define $R_{\alpha,k}$ in the similar manner as $R_{1,\alpha,k}.$
Furthermore,  consider dividing $R_{\alpha,k}$ into
a number of equivalence classes, the number of which depends only on $m,t,d,$ as in the proof of \prettyref{lmm:momentmatching}.
By the method of proof of  \prettyref{lmm:momentmatching} it can be shown that for any $2m$-tuple in $R_{\alpha,k}$
\begin{align*}
\Bigg| \expect{  \prod_{\ell =1}^m \bar{A}(T_\ell )  \bar{A}(T'_\ell ) } \le c \mu^k n^{-\alpha /2} \Bigg| \quad,
\Bigg | \expect{\prod_{\ell=1}^m \bar{A}(T_\ell )} \expect{\prod_{\ell=1}^m \bar{A}(T'_\ell )} \Bigg| \le c \mu^k n^{-\alpha /2},
\end{align*}
so that for any of the equivalence classes $S \subset R_{\alpha,k}:$
\begin{align*}
 \sum_{T_1, \ldots, T_m,  T'_1, \ldots, T'_m \in S}
 \Bigg| \expect{  \prod_{\ell =1}^m \bar{A}(T_\ell )  \bar{A}(T'_\ell )} \Bigg| + \Bigg|
\expect{\prod_{\ell=1}^m \bar{A}(T_\ell )} \expect{\prod_{\ell=1}^m \bar{A}(T'_\ell )} \Bigg| \le c \mu^k n^{-\alpha /2} |S|.
\end{align*}
Given a representative $2m$-tuple $(T_1, \ldots, T_m,  T'_1, \ldots, T'_m)  \in   R_{\alpha,k},$  the corresponding equivalence
class is defined as in  \prettyref{lmm:momentmatching}.       However, in this case there are two distinguished vertices, $i$ and $j$,
in the graph $G$, corresponding to the type of the root vertices of the first $m$ trees and the second $m$ trees, respectively.   We
then let $V_1$ be the set of vertices in $G\backslash \{i,j\}$ with types in $C^*$ and $V_2$ be the set of vertices in
$G\backslash \{i,j\}$ with types in $[n]-C^*.$  As before, let $n_1=|V_1|$  and $n_2=|V_2|.$
There are  $n_1+n_2+2$ vertices in the connected graph $G$ and at most $k + \frac{\alpha-k}2$ edges,
so $n_1+n_2 \leq  k-1 + \frac{\alpha - k}2.$   At most
$\frac{\alpha - k}2$ edges have at least one endpoint in $V_2$ and $G$ is connected, so $n_2 \leq  \frac{\alpha - k}2.$
Thus, $|S| \leq K^{n_1}n^{n_2} \leq K^{k-1}n^{\frac{\alpha - k}2}.$
Hence,
\begin{align*}
 \sum_{(T_1, \ldots, T_m,  T'_1, \ldots, T'_m) \in S }
 \Bigg| \expect{  \prod_{\ell =1}^m \bar{A}(T_\ell )  \bar{A}(T'_\ell )} \Bigg| + \Bigg|
\expect{\prod_{\ell=1}^m \bar{A}(T_\ell )} \expect{\prod_{\ell=1}^m \bar{A}(T'_\ell )} \Bigg| \le
 c \mu^k n^{-\alpha /2}   K^{k-1}n^{\frac{\alpha - k}2} \\
=  c \left( \frac{K \mu}{\sqrt{n}} \right)^{k} /K  \le c/K.
\end{align*}
In conclusion,  $\var \left( \frac{1}{K} \sum_{i \in C^\ast} (\xi_i^t )^m  \right) \le c/K$ and
the first claim follows.
\end{proof}

With \prettyref{lmm:momentbound} and \prettyref{lmm:momentvariance} in hand, we are ready to compute the
moments of $\xi^t$ in the asymptotic limit $n \to \infty$.
\begin{lemma}\label{lmm:momentcompute}
For any $ t \ge 0$, $m \ge 1$:
\begin{align*}
\lim_{n \to \infty} \expect{ (\xi_{i \to j} ^t)^m } &= \expect{ (\mu_t + \tau_t Z_t )^m}, \quad \forall i \in C^\ast, j \in [n], j\neq i \\
\lim_{n \to \infty} \expect{ (\xi_{i \to j} ^t)^m } &= \expect{ ( \tau_t Z_t )^m}, \quad \forall i \notin C^\ast, j \in [n], j\neq i . \\
\lim_{n \to \infty} \expect{ (\xi_{i} ^t)^m } &= \expect{ (\mu_t + \tau_t Z_t )^m}, \quad \forall i \in C^\ast \\
\lim_{n \to \infty} \expect{ (\xi_{i} ^t)^m } &= \expect{ ( \tau_t Z_t )^m}, \quad \forall i \notin C^\ast.
\end{align*}
\end{lemma}
\begin{proof}
We prove the first two claims; the last two follows by the similar argument.
We prove by induction over $t$. Suppose the following identities hold for $t$ and all $m \ge 1$:
\begin{align*}
\lim_{n \to \infty} \expect{ (\xi_{i\to j}^t)^m } &= \expect{ (\mu_t + \tau_tZ_t )^m}, \quad \forall i \in C^\ast, j \in [n], j\neq i \\
\lim_{n \to \infty} \expect{ (\xi_{i\to j}^t)^m } &= \expect{ (\tau_t Z_t )^m}, \quad \forall i \notin C^\ast, j \in [n], j\neq i\\
\lim_{n \to \infty} \frac{1}{K} \sum_{ \ell \in C^\ast } ( \xi_{\ell \to i}^t)^m  & \overset{p}{=} \expect{ (\mu_t + \tau_t Z_t )^m}, \quad \forall i \in [n], \\
\lim_{n \to \infty} \frac{1}{n} \sum_{ \ell \in [n]\backslash C^\ast } ( \xi_{\ell \to i}^t)^m  &\overset{p}{=} \expect{ (\tau_t Z_t )^m}, \quad \forall i \in [n],
\end{align*}
where $Z_t \sim \calN(0, 1)$. We aim to show they also hold for $t+1$. Notice that the above identities hold for $t=0$, because $\xi^0_{i \to j} =0$ for all $i \neq j$ and $\mu_0=\tau_0=0$.
Let $\calF_t$ denote the $\sigma$-algebra generated by $A^0, \ldots, A^{t-1}$.

Fix an $i \in C^\ast$. Then
\begin{align*}
\lim_{n \to \infty} \expect{ \xi_{i\to j}^{t+1}  | \calF_t} &= \lim_{ n\to \infty} \expect{ \sum_{\ell \in C^\ast \backslash\{j\}} A^t_{\ell i} f (\xi^t_{\ell \to i} )+ \sum_{\ell \in [n] \backslash C^\ast \backslash \{j\}  } A^t_{\ell i} f (\xi^t_{\ell \to i} ) |\calF_t} \\
& = \sqrt{\lambda} \lim_{ n\to \infty}  \frac{1}{K}  \sum_{\ell \in C^\ast \backslash\{j\}} f (\xi^t_{\ell \to i} ) \\
& =  \sqrt{\lambda} \lim_{ n\to \infty}  \frac{1}{K}  \sum_{\ell \in C^\ast } f (\xi^t_{\ell \to i} )\\
& \overset{p}{=}  \sqrt{\lambda}  \expect{ f(\mu_t + \tau_t Z_t )} = \mu_{t+1},
\end{align*}
where the first equality  follows from the definition of $\xi^{t+1}$ given by \prettyref{eq:defxi}; the second equality holds because $\expect{A^t_{\ell i}} =\mu$ if $\ell \in C^\ast $ and $\expect{A^t_{\ell i}} =0$ otherwise; the third equality holds in view of  Lemma
\ref{lmm:momentbound},  the fourth equality holds due to \prettyref{lmm:momentvariance} (showing the random sum
concentrates on its mean), the induction hypothesis and the
fact that $f$ is a finite-degree polynomial; the last equality holds due to the definition of $\mu_{t+1}$.

 Similarly,
\begin{align}
\lim_{n \to \infty} \var \left(  \xi_{i\to j}^{t+1}  | \calF_t \right) & =\lim_{n \to \infty}  \sum_{\ell \in [n] \backslash \{j\} } \var \left( A^t_{\ell i} f (\xi^t_{\ell \to i} ) | \calF_t \right)  \nonumber \\
& = \lim_{n \to \infty} \frac{1}{n} \sum_{\ell \in [n] \backslash \{j\} } f (\xi^t_{\ell \to i} )^2   \label{eq:start}   \\
& = \lim_{n \to \infty}  \frac{1}{n} \left\{   \sum_{\ell \in [n] \backslash C^*\cup \{j\} } f (\xi^t_{\ell \to i} )^2   +
  \sum_{\ell \in  C^*\backslash \{j\} } f (\xi^t_{\ell \to i} )^2 \right\}        \\
& =  \lim_{n \to \infty} \frac{1}{n} \sum_{\ell \in [n] \backslash C^* } f (\xi^t_{\ell \to i} )^2      \\
& \overset{p}{=} \expect{ f (\tau_t Z_t )^2 } = \tau^2_{t+1},     \label{eq:end}
\end{align}
where the first equality follows from the conditional independence of $A^t_{\ell i} f (\xi^t_{\ell \to i})$ for $\ell \in [n]$;
the second equality holds because $\var(A_{\ell i} ) = 1/n$ for all $\ell$; the third equality is the result of  breaking a sum into two parts,
the fourth equality holds in view of \prettyref{lmm:momentbound}
and the assumption that $K=o(n)$; the fifth equality holds  in view of \prettyref{lmm:momentvariance}, the induction
hypothesis and the fact that $f$ is a finite-degree polynomial; the last equality holds due to the definition of $\tau_{t+1}.$

Next, we argue that conditional on $\calF_t$,  $\xi_{i\to j}^{t+1}$ converges to Gaussian random variables in distribution.
In particular, conditional on $\calF_t$, $\xi_{i\to j}^{t+1} - \expect{\xi_{i \to j}^{t+1}}$ is a sum of independent random variables.
We show that the Lyapunov condition for the central limit theorem holds in probability, \ie,
\begin{align}
 \lim_{n \to \infty} \frac{1}{ \left( \var (  \xi_{i\to j}^{t+1}  | \calF_t )\right)^2 } \sum_{\ell \in [n] \backslash \{j \} } f (\xi^t_{\ell \to i} )^4  \expect{ (A^t_{\ell i} - \expect{A^t_{\ell i} } )^4 } \overset{p}{=}0. \label{eq:LyapunovCondition}
\end{align}
Notice that $\expect{ (A^t_{\ell i} - \expect{A^t_{\ell i} } )^4 } = 3 n^{-2} $  and thus
\begin{align*}
 \frac{1}{ \left( \var (  \xi_{i\to j}^{t+1}  | \calF_t )\right)^2 }
 \sum_{\ell \in [n] \backslash \{j \} } f (\xi^t_{\ell \to i} )^4  \expect{ (A^t_{\ell i} - \expect{A^t_{\ell i} } )^4 }
 = \frac{3}{n^2 \left( \var (  \xi_{i\to j}^{t+1}  | \calF_t )\right)^2 } \sum_{\ell \in [n] \backslash \{j \} } f (\xi^t_{\ell \to i} )^4.
\end{align*}
Taking the limit $n \to \infty$ on both sides of the last displayed equation and noticing that $ \var (  \xi_{i\to j}^{t+1}  | \calF_t ) \overset{p}{\to} \tau^2_{t+1}$
and $\frac{1}{n} \sum_{\ell \in [n] \backslash \{j \} } f (\xi^t_{\ell \to i} )^4 \overset{p}{\to} \expect{f(\tau_t Z_t )^4 }$ (using the same
steps as in \prettyref{eq:start}-\prettyref{eq:end}), we arrive at \prettyref{eq:LyapunovCondition}.
%It follows from the central limit theorem that the distribution of $\xi_{i\to j}^{t+1}$ conditional on $\calF_t$ converges to the distribution of $\mu_{t+1} + \tau_{t+1} Z_{t+1}$.
%One can further check that $\expect{ |\xi_{i\to j}^{t+1} |^m | \calF_t }$ is bounded by a constant independent of $n$, and thus by dominated convergence theorem,
It follows from the central limit theorem that for any $c$,
\begin{align*}
\lim_{ n \to \infty}  \prob{ \xi_{i\to j}^{t+1}\leq c  | \calF_t } \overset{p}{=} \prob{   \mu_{t+1} + \tau_{t+1} Z_{t+1} \leq c }.
\end{align*}
Since $\expect{\expect{ (\xi_{i\to j}^{t+1} )^m | \calF_t }} = \expect{(\xi_{i\to j}^{t+1} )^m} \le c$ for some $c$ independent of $n$,
by the dominated convergence theorem,
\begin{align*}
\lim_{n \to \infty}  \expect{(\xi_{i\to j}^{t+1} )^m} =  \expect{\lim_{ n \to \infty} \expect{ (\xi_{i\to j}^{t+1} )^m | \calF_t } }  =  \expect{ (\mu_{t+1} + \tau_{t+1}  Z_{t+1})^m}.
\end{align*}
In view of \prettyref{lmm:momentvariance} and Chebyshev's inequality,
\begin{align*}
\lim_{n \to \infty} \frac{1}{K} \sum_{ \ell \in C^\ast } ( \xi_{\ell \to i})^m \overset{p}{=} \expect{ ( \mu_{t+1} + \tau_{t+1}  Z_{t+1} )^m}.
\end{align*}

We now fix $ i \notin C^\ast$. Following the previous argument, one can easily check that
\begin{align*}
\expect{ \xi_{i\to j}^{t+1}  | \calF_t} & = 0 \\
\lim_{n \to \infty} \var \left(  \xi_{i\to j}^{t+1}  | \calF_t \right) & \overset{p}{=} \tau^2_{t+1}.
\end{align*}
Using the central limit theorem and Chebyshev's inequality, one can further show that
\begin{align*}
\lim_{n \to \infty}  \expect{(\xi_{i\to j}^{t+1} )^m}  & =   \expect{ ( \tau_{t+1}  Z_{t+1} )^m} \\
\lim_{n \to \infty} \frac{1}{n} \sum_{ \ell \in [n]\backslash C^\ast} ( \xi_{\ell \to i})^m  & \overset{p}{=} \expect{ (\tau_{t+1} Z_{t+1} )^m}.
\end{align*}
%For $R_1$ and $R_2$, we can use the same argument to show that
%for $j =1,2$,
%\begin{align*}
% \sum_{T_1, \ldots, T_m,  T'_1, \ldots, T'_m \in R_{j} }
% \bigg| \expect{  \prod_{\ell =1}^m \bar{A}(T_\ell )  \bar{A}(T'_\ell )} \bigg| + \bigg|
%\expect{\prod_{\ell=1}^m \bar{A}(T_\ell )} \expect{\prod_{\ell=1}^m \bar{A}(T'_\ell )} \bigg| \le c n^{-1/2}.
%\end{align*}

\end{proof}

\begin{proof}[Proof of \prettyref{lmm:approximatecentrallimit}]
We show the first claim; the second one follows analogously.   Fix $t\geq 1.$
Since the convergence property to be proved depends only on the sequence of random
empirical distributions  of $(\theta_i^t : t \in C^*)$ indexed by $n.$   We may therefore
assume without loss of generality that all the random variables $(\theta_i^t : t \in C^*)$
for different $n$ are defined on a single underlying probability space; the joint
distribution for different  values of $n$ can be arbitrary.   To show the convergence
in probability, it suffices to show that for any subsequence $\{n_k\}$ there exists a
sub-subsequence $\{n_{k_\ell}\}$ such that
\begin{align}
\lim_{\ell \to \infty}   \dKS\pth{ \frac{1}{K_{k_\ell}} \sum_{i \in C^\ast} \delta_{\theta_i^t}, \calN(\mu_t, \tau_t^2)} = 0, \text{ a.s.}  \label{eq:subsequenceconvergence}
\end{align}
Fix a subsequence $n_k$. In view of Lemmas \ref{lmm:momentmatching} and \ref{lmm:momentcompute}, for any fixed integer $m$,
\begin{align*}
\lim_{k \to \infty} \expect{ (\theta_{i} ^t)^m } = \expect{ (\mu_t + \tau_t Z_t )^m}.
\end{align*}
Combining \prettyref{lmm:momentvariance} with Chebyshev's inequality,
\begin{align}
\lim_{k \to \infty} \frac{1}{K_k} \sum_{i \in C^\ast} \left( \theta_i^t \right)^m \overset{p}{=} \expect{ \left( \mu_t + \tau_t Z_t \right)^m },
\label{eq:subsub}
\end{align}
which further implies, by the well-known property of convergence in probability,
that there exists a sub-subsequence such that \prettyref{eq:subsub} holds almost surely. Using a
standard diagonal argument, one can construct a sub-subsequence $\{n_{k_\ell}\}$ such that for all $m \ge 1$,
\begin{align*}
\lim_{\ell \to \infty} \frac{1}{K_{k_\ell}} \sum_{i \in C^\ast} \left(\theta_i^t\right)^m = \expect{ (\mu_t + \tau_t Z_t )^m}  \text{ a.s.}
\end{align*}
Since Gaussian distribution are determined by its moments, by the method of moments (see, for example, \cite[Theorem 4.5.5]{Chung2001course}),  applied for each outcome $\omega$ in the underlying
probability space (excluding some subset of probability zero),  it follows that the sequence of empirical distribution of
$\theta_i^t$ for $i \in C^\ast$ weakly converges to $\calN(\mu_t, \tau_t^2)$, which, since Gaussian density is bounded, is equivalent to convergence in the Kolmogorov distance,\footnote{This follows from the fact that when one of the distributions has bounded density the L\'evy distance, which metrizes
weak convergence, is equivalent to the Kolmogorov distance (see, e.g. \cite[1.8.32]{petrov}).}
proving the desired \prettyref{eq:subsequenceconvergence}.
\end{proof}

\section{Proofs of algorithm correctness}

Theorems \ref{thm:almost_exactBP_submat}-\ref{thm:weakexactBP_submat}  are proved in this section.
%Once again, the analysis closely follows \cite{Deshpande12}.
\prettyref{lmm:approximatecentrallimit} implies that if $i \in C^\ast$, then
$\theta_i^t \sim \calN(\mu_t, \tau_t^2)$; if $i \notin C^\ast$, then $\theta_i^t \sim \calN(0, \tau_t^2)$.
Ideally, one would pick the optimal $f(x, t)= \eexp^{\mu_t (x- \mu_t) }$ which result in the optimal state evolution
% \prettyref{eq:mu-ideal}
$\mu_{t+1} =  \sqrt{\lambda} \eexp^{\mu_t^2/2} $
and $\tau_t=1$
for all $t \ge 1$. Furthermore, if $\lambda > 1/\eexp$,
then $\mu_t \to \infty$ as $ t \to \infty$, and thus we can hope to estimate $C^\ast$ by selecting the indices $i$
such that $\theta_i^t$ exceeds a certain threshold.
The caveat is that \prettyref{lmm:approximatecentrallimit} needs $f$ to be a polynomial of finite degree.
%Nevertheless, the following  lemma shows that truncating the Taylor expansion of the exponential function
%$f(x, t)= \eexp^{\mu_t (x- \mu_t) }$ suffices for our purpose. The proof of \prettyref{lmm:polynomialapproximation} can be found in \cite[Lemma 2.3]{Deshpande12}.
Next we proceed to find the best degree-$d$ polynomial for iteration $t$, denoted by $f_d(\cdot,t)$, which maximizes the signal to noise ratio.

%\begin{lemma}\label{lmm:polynomialapproximation}
%Assume $\lambda > 1/\eexp $.  Let $P(z,0)\equiv 1$ and $\hat \mu_1=\sqrt{\lambda},$  and for $t\geq 1,$  define recursively
%\begin{align*}
%P(z, t) = \frac{1}{\hat{L}_t} \sum_{k=0}^{d} \frac{ \hat{\mu}^k_t z^k }{k!} , \quad \hat{\mu}_{t+1} = \sqrt{\lambda} \expect{P(\hat{\mu}_t + Z, t) },
%\end{align*}
%where $Z \sim \calN(0,1)$ and $\hat{L}_t^2 = \expect{ \left( \sum_{k=0}^d ( \hat{\mu}_t Z )^k / k! \right)^2}$.
%For any $M \in \reals$, let $t^\ast(M, \lambda)= \inf\{t: \mu_t \ge 2M\}$ with $\mu_t$ defined in \prettyref{eq:mut}.
%Then there exists $d^\ast(M, \lambda)$ such that if  $d= d^\ast$,  then
%$\hat{\mu}_{t^\ast} \ge M$.
%\end{lemma}
%\hrule
%\nbr{YW: below are optimal choice of polynomials that supercede the existential result in \prettyref{lmm:polynomialapproximation}}

Recall that the Hermite polynomials $\{H_k: k\geq 0\}$ are the orthogonal polynomials with respect to the standard normal distribution (cf. \cite[Section 5.5]{orthogonal.poly}), given by
\begin{align}
	H_k(x)
%	= & ~   \frac{(-1)^k}{\varphi(x)} \fracdk{\varphi(x)}{x}{k} 	\label{eq:Hm} \\
%	= & ~
	= (-1)^k\frac{\varphi^{(k)}(x)}{\varphi(x)}  = \sum_{i = 0}^{\lfloor k/2 \rfloor} (-1)^i (2i-1)!! \binom{k}{2i} x^{k-2i}, \label{eq:Hn.2}
\end{align}
where $\varphi$ denotes the standard normal density and $\varphi^{(k)}(x)$ is  the
$k$-th derivative of $\varphi(x)$; in particular, $H_0(x)=1, H_1(x) = x, H_2(x)=x^2-1$, etc.
Furthermore, $\deg(H_k)=k$ and
$\{H_0,\ldots,H_d\}$ span all polynomials of degree at most $d$.
For $Z\sim \calN(0,1)$,  $\Expect[H_m(Z)H_n(Z)]= m! \delta_{m,n}$ and $\Expect[H_k(\mu+Z)]=\mu^k$ for all $\mu\in\reals$; hence the relative density $\frac{\diff \calN(\mu,1)}{\diff \calN(0,1)}(x)= e^{\mu x - \mu^2/2} $ admits the following expansion:
\begin{equation}
e^{\mu x - \mu^2/2} = \sum_{k=0}^\infty H_k(x) \frac{\mu^k}{k!}.
	\label{eq:ortho}
\end{equation}
Truncating and normalizing the series at the first $d+1$ terms immediately yields the solution to \prettyref{eq:optf-poly}
as the best degree-$d$ $L_2$-approximation to the relative density, described as follows:
\begin{lemma}
Fix $d\in \naturals$ and define $\hat \mu_t$ according to the iteration \prettyref{eq:mu-poly} with $\hat \mu_0=0$, namely,
\begin{equation}
\hat\mu_{t+1}^2=\lambda G_d(\hat\mu_t^2).	
	\label{eq:mu-poly1}
\end{equation}
where $G_d(\mu)=\sum_{k=0}^d \frac{\mu^{k}}{k!}$. Define
\begin{equation}
	f_d(x,t) = \sum_{k=0}^d a_k H_k(x),
	\label{eq:fd}
\end{equation}
where $a_k \triangleq \frac{\hat \mu_t^k}{k!} (\sum_{k=0}^d \frac{\hat \mu_t^{2k}}{k!})^{-1/2}  $.
Then $f_d(\cdot,t)$ is the unique maximizer of \prettyref{eq:optf-poly} and
the state evolution \prettyref{eq:state_evolution1} and \prettyref{eq:state_evolution2} with $f=f_d$
coincides with $\tau_t=1$ and $\mu_t=\hat\mu_t$.
Furthermore, for any $d\geq 2$ the equation
\begin{equation}
	G_d(a) = a G_{d-1}(a)
	\label{eq:ga}
\end{equation}
has a unique positive solution, denoted by $a^*_d$. Let $\lambda^*_d = \frac{1}{G_{d-1}(a^*_d)}$ and define $\lambda^*_1=1$. Then
\begin{enumerate}
	\item for any $d\in \naturals$ and any $\lambda > \lambda^*_d$, $\hat\mu_t \to \infty$ as $t\diverge$ and hence for any $M>0$,
	\begin{equation}
	t^*(\lambda,M) = \inf\{t: \hat\mu_t > M\}
	\label{eq:tlambda}
\end{equation}
	is finite;
	\item $\lambda^*_d  \downarrow 1/e$ monotonically as $d \diverge$ according to $\lambda^*_d = 1/e - \frac{1/e^2+o(1)}{(d+1)!}$.
%	 \nbr{JX: should be $\lambda^*_d = 1/e + \frac{1/\eexp^2+o_d(1)}{d!}$}.
\end{enumerate}
	\label{lmm:hermite}
\end{lemma}
\begin{remark}
The best affine update gives $\lambda^*_1=1$; for the best quadratic update, $a_2^*=\sqrt{2}$ and hence $\lambda^*_2= \frac{1}{1+\sqrt{2}} \approx 0.414$.
%	Furthermore, $\lambda_3^*= 0.414, \lambda_4^& 0.376 & 0.369 & 0.368
%	For best cubic, $\lambda_3 \approx 0.376$ and $\lambda_\infty = 1/e \approx 0.368 $.
%best cubic: $\lambda_3 = \frac{8}{2 \sqrt[3]{11+2 \sqrt{30}}+\left(11+2 \sqrt{30}\right)^{2/3}+2 \sqrt[3]{11-2 \sqrt{30}}+\left(11-2 \sqrt{30}\right)^{2/3}+7} \approx 0.376 > 1/e \approx 0.368 $
More values of the threshold are given below, which converges to $1/e\approx 0.368$ rapidly.
\begin{center}
\begin{tabular}{c|c|c|c|c|c}
\hline
	~~$d$~~ & 1 & 2 & 3 & 4 & 5 \\
	\hline
	~~$\lambda_d^*$~~ & 1 & 0.414 & 0.376 & 0.369 & 0.368 \\
	\hline
	\end{tabular}
%	\label{tab:}
%	\caption{}
%\end{table}	
\end{center}
\end{remark}

\begin{remark}\label{rmk:choiced}
Let
\begin{equation}
	d^*(\lambda) = \inf\{d \in \naturals: \lambda^*_d < \lambda\},
	\label{eq:dlambda}
\end{equation}
which is finite for any $\lambda > 1/e$. Then for any $d \geq d^*$,  $\hat\mu_t \to \infty$ as $t\diverge$.	
We note that as $\lambda$ approaches the critical value $1/e$, the degree $d^*(\lambda)$ blows up according to $d^*(\lambda) = \Theta(\log \frac{1}{\lambda e-1} / \log \log \frac{1}{\lambda e-1})$, as a consequence of the last part of \prettyref{lmm:hermite}.
	\label{rmk:deg-lambda}
\end{remark}

\begin{remark}[Best affine message passing]
For $d=1$, the best state evolution is given by
	\[
\hat\mu_{t+1}^2 = \lambda (1 + \hat \mu_t^2)
	\]
and the corresponding optimal update rule is
\[
f_1(x,t) = \frac{1+ \hat \mu_t x }{\sqrt{1+ \hat \mu_t^2}}.	
\]
This is strictly better than $f(x,t)=x$ described in \prettyref{sec:spectral}
	which gives $\hat\mu_{t+1}^2 = \lambda \hat \mu_t^2$; nevertheless, in order to have $\hat\mu_t \diverge$
we still need to assume the spectral limit $\lambda \geq 1$.
	\label{rmk:affine}
\end{remark}

\begin{proof}[Proof of \prettyref{lmm:hermite}]
To solve the maximization problem \prettyref{eq:optf-poly}, note that any degree-$d$ polynomial $g$ can be written in terms of the linear combination \prettyref{eq:fd}, where the coefficients satisfies $\Expect[g^2(Z)] = \sum_{k=0}^d k! a_k^2 = 1$. By a change of measure, $\Expect[g(\hat \mu_t+Z)] = \Expect[g(Z) e^{\hat \mu_t Z- \hat \mu_t^2/2}] = \sum_{k=0}^d a_k \hat \mu_t^k$, in view of the orthogonal expansion \prettyref{eq:ortho}. Thus the optimal coefficients and the optimal polynomial $f_d(\cdot,t)$ are given by \prettyref{eq:fd}, resulting in the following state evolution
\[
\hat \mu_{t+1} = \sqrt{\lambda} 	\max\{\Expect[g(\hat \mu_t + Z)]\colon \Expect[g(Z)^2]=1, \deg(g) \leq d \} = \pth{\lambda  \sum_{k=0}^d \frac{\hat\mu_t^{2k}}{k!} }^{1/2},
\]	
which is equivalent to \prettyref{eq:mu-poly1}.

Next we analyze the behavior of the iteration \prettyref{eq:mu-poly1}.
The case of $d=1$ follows from the obvious fact that $\hat\mu_{t+1}^2=\lambda (\hat\mu_t^2+1)$ diverges if and only if $\lambda \geq 1$. For $d\geq 2$, note that $G_d$ is a strictly convex function with $G_d(0)=1$ and $G_d'=G_{d-1}$.  Also, $(G_d(a) - a G_{d-1}(a))'= - a G_d''(a) < 0$.   Thus, $G_d(a) - a G_{d-1}(a)$ is strictly decreasing on $a>0$ with value $\frac{1}{d!}$ at $a=1$ and limit $-\infty$ as
$a\to \infty,$  so  \prettyref{eq:ga} has a unique positive solution $a^*_d$ and it satisfies $a^*_d >  1.$
Furthermore, $(G_d(a) -  a G_{d-1}(a))'  \big|_{a=1}=  \sum_{k=0}^{d-2} \frac{1}{k!},$ so
by Taylor's theorem,
$$
G_d(a) -  a G_{d-1}(a)   = \frac{1}{d!} - (a-1)  \sum_{k=0}^{d-2} \frac{1}{k!} + O((a-1)^2),
$$
yielding
$$
a_d^* = 1+ \frac{1}{d!\sum_{k=0}^{d-2} \frac{1}{k!} } + O(1/(d!)^2).
$$
Consider next the values of $\lambda$ such that $\hat\mu_t$ diverges.
For very large $\lambda$, $G_d(a)$  dominates
$a/\lambda$ pointwise and $\hat\mu_t$ diverges. The critical value of $\lambda$ is when
$G_d(a)$ and $a/\lambda$ meet tangentially, namely,
$$
\lambda G_{d-1}(a) =1, \quad
\lambda G_d(a) =a,
$$
whose solution is given by $a=a_d^*$ and $\lambda = \lambda^*_d,$  where
\begin{align*}
\lambda_d^* & \triangleq \frac{1}{G_{d-1}(a_d^*)} = \frac{1}{ G_{d-1}(1) + G'_{d-1}(1) (a_d^*-1) + O( (a_d^\ast-1)^2) } \\
& = \frac{1}{\sum_{k=0}^{d} \frac{1}{k!}  + O(1/(d!)^2)}
= 1/\eexp + \frac{\sum_{k=d+1}^\infty 1/k! + O(1/(d!)^2)  }{\eexp\sum_{k=0}^{d} 1/k!  } \\
& = 1/\eexp + \frac{1/\eexp^2+o(1)}{ (d+1)!}.
\end{align*}
Thus, $\lambda_d^*$ is the minimum value such that
for all $\lambda > \lambda^*_d$, $\lambda G_d(a) > a$ for all $a>0,$ so that starting from any
$\hat\mu_t\geq 0$ we have $\hat\mu_t \to \infty$ monotonically.
The fact $\lambda_d^*$ is decreasing in $d$ follows from the fact $G_d$ is pointwise increasing in $d.$
\end{proof}

Lemmas \ref{lmm:approximatecentrallimit} and \ref{lmm:hermite} immediately imply the following
partial recovery results.
\begin{lemma}\label{lmm:approximaterecovery}
Assume that $\lambda > 1/\eexp$ and $\Omega(\sqrt{n}) \le K \le o(n)$.  Fix any $\epsilon \in (0,1)$. Let $M=8 \log (1/\epsilon)$ and run the message passing algorithm for $t$ iterations with $f=f_{d^*}$, $d^*=d^\ast(\lambda)$ as in \prettyref{eq:dlambda}, and $t=t^\ast(\lambda, M)$ as in \prettyref{eq:tlambda}. Let $\tilde{C}=\{ i: \theta_i^{t^\ast} \ge \hat{\mu}_{t^\ast}/2 \}$. Then with probability converging to one as $n\to\infty,$
\begin{align}
 \frac{1}{K} | \tilde{C} \cap C^\ast |  & \ge 1-\epsilon   \label{eq:ccA} \\
% \frac{1}{n} | \tilde{C} \backslash C^\ast | & \le \epsilon   \label{eq:ccB} \\
K(1-\epsilon) \leq  |\tilde{C}|  &\leq  n\epsilon . \label{eq:ccC}
 \end{align}
% \nbr{YW: \prettyref{eq:ccB} is only used to prove \prettyref{eq:ccC} and never used in the subsequent lemma. So I vote we kill it from the statement.}
% \nb{JX: OK. I agree to kill \prettyref{eq:ccB}. I also suggest put the assumption $K=o(n)$ in the statement.}
\end{lemma}
\begin{proof}
Notice that
\begin{align*}
| \tilde{C} \cap C^\ast | = \sum_{i \in C^\ast} \indc{  \theta_i^{t^\ast} \ge \hat{\mu}_{t^\ast}/2 }.
\end{align*}
%Since the indicator function can be approximated from above and below pointwise by
%bounded Lipschitz continuous function,
By the choice of $f=f_{d}$ in \prettyref{eq:fd}, we have $\tau_t = 1$ for all $t \geq 1$.
It follows from \prettyref{lmm:approximatecentrallimit} that
\begin{align}
\lim_{n \to \infty}  \frac{1}{K} | \tilde{C} \cap C^\ast | = \prob{  \hat \mu_{t^\ast} + Z \ge \hat{\mu}_{t^\ast}/2 },
\label{eq:ccAa}
\end{align}
where the convergence is in probability.
Notice that we have used $d=d^\ast(\lambda)$ and $t=t^\ast(\lambda,M)$
defined by \prettyref{eq:dlambda} and \prettyref{eq:tlambda} in \prettyref{lmm:hermite}.
Thus $\hat{\mu}_{t^\ast}  \ge M = 8 \log (1/\epsilon)$ and
\begin{align*}
\prob{  \mu_{t^\ast} + Z \le \hat{\mu}_{t^\ast}/2 } = Q(  \hat{\mu}_{t^\ast}/2 ) <  \eexp^{-\hat{\mu}^2_{t^\ast}/8} < \epsilon,
\end{align*}
%Hence
%\begin{align*}
%\lim_{n \to \infty}  \frac{1}{K} | \tilde{C} \cap C^\ast |  >  1-  \eexp^{-\hat{\mu}^2_{t^\ast}/8} \ge 1-\epsilon,
%\end{align*}
which, in view of \prettyref{eq:ccAa}, implies \prettyref{eq:ccA} with probability converging to one as $n\to\infty.$  Similarly, \prettyref{lmm:approximatecentrallimit} implies that in probability
\begin{align*}
\lim_{n \to \infty} \frac{1}{n} | \tilde{C} \backslash C^\ast | =\prob{Z \ge \hat{\mu}_{t^\ast}/2 } = Q(  \hat{\mu}_{t^\ast}/2 ) .
\end{align*}
Thus in probability, $\lim_{n \to \infty}  \frac{1}{n}| \tilde{C} \backslash C^\ast |  \le  \epsilon$. Since $K=o(n)$, we have $\pprob{K(1-\epsilon) \leq  | \tilde{C}| \leq n\epsilon }\to 1.$
\end{proof}

Although $\tilde{C}$ contains a large portion of $C^\ast$,
since $| \tilde{C}|$ is linear in $n$ with high probability, \ie, $|\tilde C|/n \to Q( \hat{\mu}_{t^\ast}/2)$ by \prettyref{lmm:approximatecentrallimit},
it is bound to contain a large number of outlier indices.
The next lemma, closely
following \cite[Lemma 2.4]{Deshpande12}, shows that given the conclusion of
\prettyref{lmm:approximaterecovery}, the power iteration
in Algorithm \ref{alg:MP} can remove most of
%a large portion of
the outlier indices in $\tilde{C}.$

\begin{lemma} \label{lmm:power_clean}   Suppose $\lambda = \frac{\mu^2K^2}{n} \ge 1/\eexp$,\footnote{The proof uses the lower bound  $\lambda \ge 1/\eexp$ to get $\epsilon<10^{-3}$. If instead $\lambda \ge \lambda_0$ for some $\lambda_0>0$, then the lemma holds with $10^{-3}$ replaced by
some $\epsilon_0>0$ depending on $\lambda_0$.}
%\nbr{YW: The proof used the lower bound on $\lambda$, namely, $1/e$ to get $\epsilon<10^{-3}$. If this is a different constant, then $10^{-3}$ will change.}
%$\mu/\sqrt{n}=o(1)$, $2 \le K=o(n),$
$K\diverge$,
%$K=o(n),$ \nbr{YW: since we no longer have zero diagonal, it seems this can be dropped now.}
$\frac{|C^*|}{K}\rightarrow 1$ in probability, and
$\tilde{C}$ is a set (possibly depending on $A$) such that \prettyref{eq:ccA} -  \prettyref{eq:ccC} hold
for some $0 < \epsilon < 10^{-3}$.
Let
%$t^*=t^*(\lambda,8\log(1/\epsilon))$ be defined in \prettyref{eq:tlambda} and
\begin{equation}
s^*= \frac{2}{  \log( \sqrt{\lambda} (1-\epsilon) / (16 \sqrt{ h(\epsilon)  +\epsilon  } )  )},
%s^*= \frac{2}{  \log \left( \sqrt{\lambda} (1-\epsilon) / (  4\sqrt{\epsilon} + 4\sqrt{2} \epsilon^{1/4} ) \right)   }  ,
	\label{eq:sstar}
\end{equation}
where $h(\epsilon)\triangleq \epsilon\log \frac{1}{\epsilon} + (1-\epsilon)\log \frac{1}{1-\epsilon}$ is the binary entropy function.
Then $\hat{C}$ produced by Algorithm \ref{alg:MP} returns
$ | \hat{C} \cap C^\ast| \ge (1- \eta(\epsilon, \lambda) )K $,
with probability converging to one as $n \to \infty,$ where
\begin{equation}
%\eta(\epsilon, \lambda) = 2 \epsilon + \frac{288 (\sqrt{\epsilon} + \epsilon^{1/4})^2 }{\lambda(1-\epsilon)^2}.
\eta(\epsilon, \lambda) = 2 \epsilon + \frac{5000 (h(\epsilon) + \epsilon) }{\lambda(1-\epsilon)^2}.
	\label{eq:etaeps}
\end{equation}
\end{lemma}
\begin{proof}
Fix a $\tilde C$ that satisfies \prettyref{eq:ccA} -  \prettyref{eq:ccC}.
We remind the reader that in this paper we let $A=W/\sqrt{n}$ so
that $\var(A_{ij})=1/n$ for $i,j \in [n]$ and $\expect{A_{ij}}=\mu/\sqrt{n}$ for  $i,j \in C^*.$
Let $m=|\tilde C|$ and abbreviate the restricted matrix $A_{\tilde{C} } \in \reals^{ \tilde C \times \tilde C}$  by $\tilde A$.
%be the matrix $A$ restricted to the rows and columns in $\tilde{C}$.
Let $\ones_{\tilde C \cap C^*} \in \reals^{\tilde C} $ denote the indicator vector of $\tilde{C}\cap C^*$.
Then the mean of $\tilde A$ is the rank-one matrix
$\eexpect{\tilde A} = \frac{\mu}{\sqrt{n}} \ones_{\tilde C \cap C^*} \ones_{\tilde C \cap C^*}^\top$, whose largest eigenvalue  is
$\frac{ \mu |\tilde C \cap C^*|}{\sqrt{n}}$ with the corresponding eigenvector $v \triangleq \frac{1}{\sqrt{|\tilde C \cap C^*|}}\ones_{\tilde C \cap C^*}$.
Let $ Z = \tilde A - \eexpect{\tilde A}$, and let $u$ denote the principal eigenvector of $\tilde A$.
Using a simple variant of the Davis-Kahan's sin-$\theta$ theorem \cite[Proposition 1]{CMW13},
%\cite{DavisKahan70},
we obtain
%\footnote{There is a slight problem in \cite[Proof of Lemma 2.4]{Deshpande12} which bounds $\|u-v\|$. This in general is impossible because both $v$ and $-v$ are eigenvectors and we can only bound the difference between the eigenspaces.}
\begin{equation}
\|uu^\top - vv^\top\|  \leq \frac{2 \|Z\|}{ \mu |\tilde C \cap C^*| / \sqrt{n}} \leq \frac{2 \|Z\|}{ \sqrt{\lambda} (1-\epsilon)  },
%\frac{\|Z\|}{ \frac{ \mu |\tilde C \cap C^*|  }{ \sqrt{n}} - \| Z\| }.
	\label{eq:sintheta}
\end{equation}
where the last inequality follows from \prettyref{eq:ccA}.
Observe that $Z$
 is a symmetric matrix such that $\{ Z_{ij}\}_{i \leq j} \iiddistr N(0,1/n).$
To bound $\|Z\|$, note that $\|Z\| = \max\{\lambda_{\max}(Z), -\lambda_{\min}(Z)\}$ and 
$\lambda_{\min}(Z)$ has the same distribution as $-\lambda_{\max}(Z)$.
%consider $D = \diag{d_i}$, where $d_i \iiddistr N(0,1/n)$ are independent of $Z$. Then
%$G \triangleq \sqrt{n} (Z + D)$ is distributed according to the $m\times m$ Gaussian orthogonal ensemble.
%such that $G_{ii} ~\iiddistr \calN(0,2/n)$ and $\{G_{ij}\}_{i<j} \iiddistr \calN(0,1/n)$;
By union bound and the Davidson-Szarek bound \cite[Theorem 2.11]{Davidson01}, for any $t>0$, 
%$\prob{\|G\| \geq 2 \sqrt{m} + \sqrt{2t}} \leq \exp(-t/4)$.
\begin{align}
\prob{\|Z\| \geq 2 \sqrt{m/n} + \sqrt{2t/n}} \leq 2 \eexp^{-t/2 }  \label{eq:Delta},
\end{align}
%Choosing $t=\sqrt{m} \tau$,
%we get
%\begin{align*}
%\prob{\|\tilde W\| \geq 2 \sqrt{m} + \sqrt{2 \tau} m^{1/4} } \leq \eexp^{-\sqrt{m} \tau/4}.
%\end{align*}
% Moreover, for any $s>0$, $\|D\| = \max_{i\in \tilde C} |d_i| \leq \sqrt{\frac{s}{n}}$ with probability at least $1 - 2me^{-s/2}$.
% Hence,
%\begin{align}
%\prob{\| Z \| \geq 2 \sqrt{m/n} + \sqrt{\frac{s}{n}} +  \sqrt{\frac{2 t}{n}} } \leq  e^{- t/4} +  2 me^{-s/2}. \label{eq:Delta}
%\end{align}
By assumption we have $K(1-\epsilon) \leq m \leq \epsilon n$.
% Recall the binary entropy function $h(\epsilon)\triangleq-\epsilon\log \epsilon - (1-\epsilon)\log (1-\epsilon)$.
%Setting $t = 2s = 8 n \sqrt{\epsilon} $ and $\beta= 8 \sqrt{h(\epsilon) + \epsilon}$, we have for any fixed $\tilde C$,
%\nb{JX: $\prob{\|Z\| \geq 2 \sqrt{m/n} + \sqrt{2t/n}} \leq \exp(-t/2)$.
Setting $t=4n h(\epsilon) $ and $\beta= 8 \sqrt{h(\epsilon) + \epsilon}
\ge 2\sqrt{\epsilon} + 2 \sqrt{2} \sqrt{ h(\epsilon) } $, we have for any fixed $\tilde C$,
\begin{equation}
\prob{      \|Z\| \geq \beta }  \leq   2 \eexp^{- 2 n h(\epsilon)}.
%5 \epsilon^{1/4} + {\mu}/{\sqrt{n}}
	\label{eq:Atilde}
\end{equation}
%where the last inequality holds because $\sqrt{\epsilon} \ge h(\epsilon)$ for $\epsilon \in [0, 0.08]$.
%if $ \frac{ |\tilde{C}|}{ n  }  \leq \frac{\lambda (1-\epsilon)^2\beta^2}{256},$  which is true if $\epsilon \leq \frac{\lambda (1-\epsilon)^2\beta^2}{256}.$
The number   of possible choices of $\tilde C$ that fulfills  \prettyref{eq:ccC} so that $|\tilde C| \le \epsilon n$
 is  at most $ \sum_{ k \leq  n \epsilon} \binom{n}{k}$ which is further upper bounded by $ \eexp^{n h(\epsilon)}$ (see, \eg, \cite[Lemma 4.7.2]{ash-itbook}).
In view of \prettyref{eq:Atilde},
%and $\epsilon \le 1$ so that $h(\epsilon) < \sqrt{\epsilon}/4$,
the union bound yields $\|Z\| \leq \beta$ with high probability as $n\to \infty$.

Throughout the reminder of this proof we assume $A$ and $\tilde C$ are
fixed with  $  \|Z\| \leq \beta $.
%Since $\epsilon \le 1$ and $\lambda >1/\eexp$, we have that
%$ \beta \le  \sqrt{\lambda} (1-\epsilon) /4.$
Note that the rank of $uu^\top - vv^\top$ is at most two.
Combining with \prettyref{eq:sintheta},
we have,
\begin{equation}	
\|uu^\top - vv^\top\|_{\rm F}  \leq  \frac{ 2 \sqrt{2} \beta}{ \sqrt{\lambda}  (1-\epsilon)}.
% \le \frac{4 \sqrt{2} \beta}{3  \sqrt{\lambda}  (1-\epsilon)}.	
	\label{eq:uv}
\end{equation}
Next, we argue that $\hat{u}$ is close to $u$, and hence, close to $v$ by the triangle inequality.
By the choice of the initial vector $u^0$,
we can write $u^0 = z / \|z\|$ for a standard normal vector $z \in \reals^m$.
By the tail bounds for Chi-squared distributions, it follows that $\|z\| \le 2 \sqrt{m}$
with high probability. For any fixed $u$,
the random variable $\langle u, z \rangle \sim \calN(0,1)$ and thus
with high probability, $ |\langle u, z \rangle|^2 \geq 1/\log n$, and hence
\begin{equation}
|\langle u, u^0 \rangle|   =|\langle u, z \rangle| / \|z\|  \geq   (2 \sqrt{n  \log n } )^{-1} .
	\label{eq:u0}
\end{equation}
%A classic result in the theory of perturbations of symmetric matrices is that an additive symmetric perturbation of
%a symmetric matrix changes the eigenvalues by no more than the spectral norm of the perturbation matrix.

By Weyl's inequality, the maximal singular value satisfies
$\sigma_1(\tilde A) \geq \frac{\mu K(1-\epsilon)}{\sqrt n} - \beta$  and the
other singular values are at most $\beta$. Let $r= \frac{\sigma_2}{\sigma_1}(\tilde A)$.
%the maximal eigenvalue of  $A|_{\tilde{C}}$ is at least $\frac{\mu K(1-\epsilon)}{\sqrt n} - \beta = \sqrt{\lambda} (1-\epsilon) - \beta $  and the
%other eigenvalues have magnitude at most $\beta + \mu/\sqrt{n} \le 2\beta$ due to $\mu/\sqrt{n}=o(1)$.  Let $\lambda_1$ denote the principal eigenvalue of  $A|_{\tilde{C}}$, $\lambda_2$ denote the second largest eigenvalue
%of $A|_{\tilde{C}}$ in absolute value, and $r=\lambda_2/\lambda_1$.
By the assumption that
%\footnote{ This is the place where the assumption $\epsilon < 10^{-3}$ is used. If $\lambda \geq \lambda_0$ for some $\lambda_0>0$, we need $\epsilon<\epsilon_0$ for some $\epsilon_0>0$ depending
%on $\lambda_0$.}, 
$\epsilon < 10^{-3}$ and $\lambda \geq 1/e$, we have $\sqrt{\lambda} (1-\epsilon) \ge 2 \beta$.
As a consequence, $r \le \frac{ 2 \beta }{  \sqrt{\lambda} (1-\epsilon)}$.
%, where we used $\epsilon < 10^{-3}$ and $\lambda \geq 1/e$
%\nbr{YW: this is the only place this is used.}.
%\nbr{YW: please double-check this. }\nb{JX: I check it. It is OK.}
%\nb{JX: I do not know how to get the following directly: from which we find
%$$
%\langle u^t , u \rangle^2   \geq  1 - \frac{1}{\langle u, u^0 \rangle ^2}\left( \frac{8 \beta}{3 \sqrt{\lambda} (1-\epsilon)}\right)^{2t}.
%$$
%}
%\nb{JX: Here is my argument from scratch.}
Since $u^t= \tilde A^t u^0/ \| \tilde A^t u^0 \|$, it follows that
$$
u^t=  \frac{ \langle u, u^0 \rangle u + y }{ \| \langle u, u^0 \rangle u + y \| }
$$
for some $y \in \reals^m$ such that $\|y\| \le r^t$. Hence,
\begin{align}
\langle u^t , u \rangle^2 & = \frac{  \left( \langle u, u^0 \rangle + \langle y, u \rangle \right)^2 }{ \| \langle u, u^0 \rangle u + y \|^2 }  =  1 + \frac{ \langle y, u \rangle^2 - \|y \|^2  }{\| \langle u, u^0 \rangle u + y \|^2  } \\
& \ge 1- \frac{  \|y \|^2  }{ \langle u, u^0 \rangle^2 -2 | \langle u, u^0 \rangle| \| y\|  } \ge 1- \frac{ r^{2t}  }{ \langle u, u^0 \rangle^2 -2 | \langle u, u^0 \rangle| r^t }.
\label{eq:power}
\end{align}
Recall that $\hat{u} =u^{\lceil s^* \log n \rceil}$. Thus, choosing $s^*= \frac{2}{  \log( \sqrt{\lambda} (1-\epsilon) /(2 \beta ) ) }$
as in \prettyref{eq:sstar}, we obtain
$r^{\lceil s^* \log n \rceil} \le n^{-2}$ and consequently in view of \prettyref{eq:u0}, we get that $\langle \hat{u} , u \rangle^2  \ge 1 -  n^{- 1}$, or equivalently,
\[
\|uu^\top - \hat{u} (\hat{u})^\top \|_{\rm F}^2 = 2 - 2 \iprod{u}{\hat{u}}^2 \le n^{ -1}.
\]
Notice that
$$
\min\{\|\hat{u} - v\|^2, \|\hat{u} + v\|^2\} = 2 - 2 |\iprod{\hat{u}}{v}|  \leq \|\hat{u} (\hat{u})^\top - vv^\top \|_{\rm F}^2.
$$
Applying \prettyref{eq:uv} and the triangle inequality, we obtain
\begin{equation}
	\min\{\|\hat{u} - v\| , \|\hat{u} + v\| \} \leq \|\hat{u} (\hat{u})^\top - vv^\top \|_{\rm F} \leq \frac{2 \sqrt{2} \beta}{ \sqrt{\lambda}  (1-\epsilon)} + n^{-1/2}
\overset{(a)}{\le} \frac{3 \beta}{\sqrt{\lambda}  (1-\epsilon)} \triangleq \beta_o,
	\label{eq:proj}
\end{equation}
where $(a)$ holds for sufficiently large $n$.
% because $4\sqrt{2}/3 <2$.
%\nb{JX: The above is not quite right for $\lambda$ so large that $1/\sqrt{\lambda}$ becomes smaller than $ n^{-\Omega(1)}$.}
%\st{The algorithm only uses the absolute values of the coordinates of $\hat{u},$  but for the purposes of the proof,
%we can multiply $\hat{u}$ by $-1$ if necessary to assume that $\langle \hat{u},u\rangle > 0.$
%We then have $\| u - \hat{u}\| \leq (0.1)  \beta$ for large enough $n$, so that, by the triangle inequality,  $\|\hat{u} - v \| \leq 2\beta.$}
Let $\hat C_o$ be defined by using a threshold test to estimate $C^*$ based on
$\hat{u}$:
\[
\hat C_o = \{ i \in \tilde C : | \hat{u}_i| \geq  \tau \}
\]
where $\tau = 1/ (2\sqrt{| \tilde C \cap C^*| })$.    Note that $v_i  = 2\tau \indc{i \in \tilde C \cap C^*}.$
%For any index $i \in \tilde C$ that is incorrectly classified by $\hat C_o,$ it must be that $|\hat{u}_i - v_i|^2 \geq \tau^2.$
For any $i \in C_o \backslash (\tilde C \cap C^*)$, we have $|\hat{u}_i| \geq \tau$ and $v_i=0$;
For any $i \in (\tilde C \cap C^*)\backslash C_o$, we have $|\hat{u}_i| < \tau$ and $v_i=2\tau$.
Therefore $| |\hat{u}_i| - |v_i| | \geq \tau$ for all $i \in \hat C_o \triangle   (\tilde C \cap C^*)$ and
\[
\min\{\|\hat{u} - v\|^2, \|\hat{u} + v\|^2\} \geq |C_o \triangle   (\tilde C \cap C^*)| \tau^2.
\]
In view of \prettyref{eq:proj}, the number of indices in $\tilde C$ incorrectly classified by $\hat C_o$ satisfies
\[
| \hat C_o \triangle   (\tilde C \cap C^*)  | \leq  4 \beta_o^2 |\tilde C \cap C^*|  \leq  4 \beta_o^2 |C^*|.
%\frac{12 \sqrt{2} \beta}{ \sqrt{\lambda}  (1-\epsilon)} |\tilde C \cap C^*| \leq \frac{12 \sqrt{2} \beta |C^*|}{ \sqrt{\lambda}  (1-\epsilon)}.
\]
%$| \hat C_o \triangle   (\tilde C \cap C^*)  |  \tau^2 \leq   4\beta^2,$  or
%$| \hat C_o \triangle   (\tilde C \cap C^*)  |   \leq   16\beta^2 | \tilde{C}\cap C^*| \leq  16\beta^2 |C^*|.$
%Since $[n]\backslash \tilde C$ contains at most $\epsilon K$ of the indices in $C^*,$
Since $|C^*\backslash \tilde C| \leq \epsilon K$,
we conclude  that $|C^* \triangle \hat C_o |  \leq \epsilon K + 4 \beta_o^2 |C^*|.$
Thus, if the algorithm were to output $\hat C _o $ (instead of $\hat C$) the lemma would be proved.

Rather than using a threshold test in the cleanup step, Algorithm \ref{alg:MP} selects the $K$ indices in
$\tilde C$ with the largest values of $|\hat{u}_i|.$    Consequently, with probability one, either
$\hat C_o \subset \hat C$
or $\hat C \subset \hat C_o.$  Therefore, it follows that
$$
|C^* \triangle \hat C | \leq  2   | C^*\triangle \hat C_o |   +  \big| |C^*| - K \big|.
$$
By assumption, $|C^*| /K$ converges to one in probability, so that, in probability,
\begin{equation}
\limsup_{n\rightarrow\infty} \frac{  |C^* \triangle \hat C |   }{K } \leq  2\epsilon + 8 \beta_o^2 \leq
%2 \epsilon + \frac{5000 h(\epsilon) }{\lambda(1-\epsilon)^2}
%\triangleq
\eta(\epsilon, \lambda),	
	\label{eq:CCeta}
\end{equation}
where $\eta$ is defined in \prettyref{eq:etaeps}, completing the proof.
\end{proof}

\begin{proof}[Proof of   \prettyref{thm:almost_exactBP_submat}]
Given $\eta \in (0,1)$, choose an arbitrary $\epsilon \in (0, 10^{-3}) $ such that $\eta(\epsilon,\lambda)$ defined in \prettyref{eq:etaeps} is at most $\eta$.
With $t^*$ specified in \prettyref{lmm:approximaterecovery} and $s^*$ specified in
\prettyref{lmm:power_clean}, the probabilistic performance guarantee in
\prettyref{thm:almost_exactBP_submat} readily follows by combining \prettyref{lmm:approximaterecovery} and \prettyref{lmm:power_clean}.    The time complexity of \prettyref{alg:MP}
follows from the fact that for both the BP algorithm and the power method each iteration have complexity $O(n^2)$
  and \prettyref{alg:MP} entails running BP and power method for $t^*$ and $s^*$ iterations respectively; both $t^*$ and $s^*$ are constants depending only on $\eta$ and $\lambda$.
\end{proof}

\begin{proof}[Proof of   \prettyref{thm:weakexactBP_submat}]
(Weak recovery)
Fix $k \in [1/\delta]$ and let $C^*_k = C^* \cap S_k^c.$
Define the $n(1-\delta)\times n(1-\delta)$ matrix  $A_k \triangleq A_{S_k^c}$, which
%(defined in \prettyref{alg:exactrecovery})
corresponds
to the submatrix localization problem for a planted community
 $C^*_k$ whose size has a hypergeometric distribution, resulting from sampling
 without replacement, with parameters $(n, K, (1-\delta)n)$ and mean $(1-\delta)K.$
By a result of Hoeffding \cite{Hoeffding63}, the distribution of
$|C^*_k|$ is convex order dominated by the distribution that would result from sampling with replacement, namely, the
$\Binom\left(n(1-\delta), \frac{K}{n}\right)$ distribution.
In particular, Chernoff bounds for  $\Binom(n(1-\delta), \frac{K}{n})  $ also hold for $|C_k^{*}|,$  so
$|C_k^*|/((1-\delta)K) \to 1$ in probability  as $n \to \infty.$   Note that
 $\frac{ ((1-\delta)K)^2 \mu^2}{n(1-\delta)} \rightarrow \lambda (1-\delta)$ and $\lambda(1-\delta)\eexp > 1$ by
 the choice of $\delta.$ Let  $d^\ast( \lambda (1-\delta) )$ be given in \prettyref{eq:dlambda}, \ie,
 $$
 d^*(\lambda(1-\delta) ) = \inf\{d \in \naturals: \lambda^*_d < \lambda (1-\delta) \}.
$$
Choose an arbitrary $\epsilon \in (0, 10^{-3})$ to satisfy $ \eta(\epsilon, \lambda(1-\delta) ) \le \delta$, \ie,
$$
2 \epsilon + \frac{5000 h(\epsilon) }{\lambda(1-\delta) (1-\epsilon)^2 }  \le \delta.
$$
Define $\hat{\mu}_t $ recursively according to \prettyref{eq:mu-poly} with $\lambda$ replaced by $\lambda(1-\delta)$
and $\hat{\mu}_0=0$, \ie,
$$
\hat\mu_{t+1}^2=\lambda (1-\delta) \sum_{k=0}^d \frac{\hat\mu_t^{2k}}{k!}.	
$$
Define $t^\ast(\delta, \lambda)$ according to \prettyref{eq:tlambda} with $M= 8 \log (1/\epsilon)$,
%\ie,
%$$
%t^\ast (\delta, \lambda) = \inf \{t: \hat{\mu}_t > 8 \log (1/\epsilon) \},
%$$
and $s^*(\delta, \lambda)$ according to \prettyref{eq:sstar} with $\lambda$ replaced by $\lambda(1-\delta)$.
%\ie,
%$$
%s^*(\delta, \lambda)= \frac{2}{  \log( \sqrt{\lambda(1-\delta) } (1-\epsilon) / (12 \epsilon^{1/4}))}.
%$$
Then \prettyref{thm:almost_exactBP_submat} with $n$ and $K$
replaced by $n(1-\delta)$ and $\lceil K (1-\delta) \rceil$  implies that  as $n\to \infty,$
$$
\prob{| \hat{C}_k \triangle C^*_k | \leq  \delta K\mbox{ for }  1\leq k \leq 1/\delta }   \to 1. $$
Given  $(C_k^*,  \hat{C}_k)$,  each of the random variables $r_i \sqrt{n} $ for $i  \in S_k$
is conditionally Gaussian with variance $\lceil (1-\delta)K\rceil,$ which is smaller than $K.$   Furthermore, on the event,
${\cal E}_k =\{    | \hat{C}_k   \triangle C_k^*    |   \leq  \delta K  \},$
$$
 |\hat{C}_k  \cap C_k^*  |  \geq  |\hat{C}_k| -  |\hat{C}_k   \triangle C^*_k | =
  \lceil K(1-\delta)\rceil -     |\hat{C}_k   \triangle C^*_k |      \geq K(1-2\delta).
$$
Therefore, on the event ${\cal E}_k ,$
 for $i \in S_k \cap C^*$,  $r_i\sqrt{n}$ has mean greater than or equal to $K(1-2 \delta)\mu$, and for $i \in S_k \backslash C^*$,
$r_i$ has mean zero.

Define the following set by thresholding
\[
 C'_o = \{ i \in  [n] : r_i \geq  (1-2\delta)  \sqrt{\lambda} /2  \}
\]
%In view of \prettyref{eq:proj},
The number of indices in $S_k$ incorrectly classified by $ C'_o\cap S_k$ satisfies (use $|S_k|=\delta n$):
\[
\expect{|  (C'_o\cap S_k)  \Delta C_k^*  |} \le
 \delta n Q \left( (1-2\delta)  \sqrt{\lambda n/K } /2    \right ) \le  \delta n \eexp^{ - \Omega( n/K)  }.
\]
Summing over $k \in  [1/\delta]$ yields $\expect{|  C'_o \Delta C^*  |} \le  n \eexp^{ - \Omega( n/K)  }.$
By Markov's inequality,
\[
\prob{ | C'_o  \Delta C^*  | \ge  K^2/n  }  \le  \frac{n^2}{K^2} \eexp^{ - \Omega( n/K)  } \overset{K=o(n)}{=}  o(1).
\]
Instead of $C'_o$, \prettyref{alg:exactrecovery} outputs $C'$ which selects the $K$ indices in
$[n]$ with the largest values of $r_i.$  Applying the same argument as that at the end of the proof of \prettyref{lmm:power_clean}, we get
$
|C^* \triangle C' | \leq  2   | C^*\triangle  C'_o |   + | |C^*| - K|,
$
and hence $|C^* \triangle C' | /K \to 0$ in probability.
%\end{proof}
%
%
%\begin{proof}[Proof of   \prettyref{thm:weakexactBP_submat}]
%The discussion in \prettyref{sec:compare} shows that  \prettyref{eq:submat-mle-suffnew}  is implied by the condition
%$\lambda > 1/\eexp$ if $K\geq \frac{n}{8\eexp \log n}$ for all $n.$
%Thus, as noted in \prettyref{sec:compare},   \prettyref{thm:weakexactBP_submat} can be deduced from
%\prettyref{thm:weakexactBP_submat} and a general result in  \cite{HajekWuXu_one_info_lim15}.

(Exact recovery)
As noted in \prettyref{rmk:two_step}, the second part of \prettyref{thm:weakexactBP_submat} readily follows from
\prettyref{thm:almost_exactBP_submat} and the general result in  \cite[Theorem 7]{HajekWuXu_one_info_lim15}.
Here, we give an alternative, more direct proof based on the weak recovery proof given above.
Recall the fact that the maximum of $m$ independent standard normal random variables is at
most $\sqrt{2 \log m} + o_P(1)$ as $m\to \infty$, with equality if they are independent \cite{DN70}.
Also, for  $k\in [1/\delta]$,
$|S_k \cap C^*| \leq |C^*|=K$ and  $| S_k  \backslash C^* | \leq |[n]\backslash C^*| =n-K.$
Therefore,
\begin{eqnarray}
\min_{i\in S_k\cap C^*} r_i \sqrt{n} & \geq & K(1-2\delta) \mu  - \sqrt{2K\log K } + o_P(\sqrt{K})  \label{eq:rmin}\\
\max_{j \in S_k \backslash C^*} r_i \sqrt{n} & \leq &  \sqrt{2K\log(n-K)} + o_P(\sqrt{K}).\label{eq:rmax}
\end{eqnarray}
Since $k$ ranges over a finite  number of values, namely, $[1/\delta]$,
\prettyref{eq:rmin} and \prettyref{eq:rmax} continue to hold with left-hand sides replaced by $\min_{i\in C^*} r_i \sqrt{n} $ and
$\max_{j \in [n]  \backslash C^*} r_i \sqrt{n}$, respectively.
%the same conditions hold for extrema over $C^*$ and $[n]\backslash C^*$, respectively:
%\begin{eqnarray*}
%\min_{i\in C^*} r_i \sqrt{n} & \geq  & K(1-2\delta) \mu  - \sqrt{2K\log K } + o_P(\sqrt{K})  \\
%\max_{j \in [n]  \backslash C^*} r_i \sqrt{n} & \leq &   \sqrt{2K\log( n-K)} + o_P(\sqrt{K}).
%\end{eqnarray*}
Therefore,  by the choice of $\delta,$
 $\min_{i\in C^*} r_i \sqrt{n}  >  \max_{j \in [n] \backslash C^*} r_i \sqrt{n}$ with
probability converging to one as $n\to \infty$ and so $\check{C}=C^\ast$ with probability converging to one as well.
% This establishes the performance  guarantee of  \prettyref{thm:weakexactBP_submat}.  The total complexity of  Algorithm \ref{alg:exactrecovery}
%has already been shown in the proof of   \prettyref{thm:weakexactBP_submat}.

(Time complexity)
The running time of
Algorithm \ref{alg:exactrecovery} is dominated by invoking Algorithm \ref{alg:MP} for a constant number, $1/\delta$,
of times, and the number of iterations within Algorithm \ref{alg:MP} is $(t^*+s^*\log n )n^2$, with both $t^*$ and $s^* \to \infty$ as  either $\delta \to 0$ or $\lambda \to 1/\eexp$.  In particular, the threshold comparisons require $O(n^2)$ computations.
Thus, the total complexity of  Algorithm \ref{alg:exactrecovery} is as stated in the theorem.
\end{proof}

 \begin{remark} \label{rmk:differences}
 Versions of  Theorems \ref{thm:almost_exactBP_submat} and \ref{thm:weakexactBP_submat}
 are given in  \cite{Deshpande12} for the case
 $K= \Theta(\sqrt{n})$ and $\mu= \Theta(1)$; here we extend the range of $K$ to
 $\Omega(\sqrt{n}) \leq K \leq o(n).$    The algorithms and proofs are nearly
 the same; we  comment here on the main differences we encountered
 by allowing $K/\sqrt{n} \to \infty$  and $\mu \to 0.$
First, a larger $K$ requires modification of bounds used in calculating the means and
variances of messages in Lemmas \ref{lmm:momentmatching} -   \ref{lmm:momentvariance}.
The larger $K$ means a larger portion of messages are sent between
vertices in $C^\ast$. That effect is offset by $\mu$ being smaller.   Our approach is to balance these
two effects by  accounting separately for the contributions of singly covered edges with both endpoints in $C^*.$
See $R_{1,\alpha,k}$ in \prettyref{lmm:momentmatching},
$R_{3, \alpha, k}$ in \prettyref{lmm:momentbound}, and $R_{\alpha,k}$ in \prettyref{lmm:momentvariance}.

Secondly, after the message passing algorithm and spectral cleanup are applied in   \prettyref{alg:MP},
a final cleanup procedure is applied to obtain weak recovery or exact recovery (when possible).
As in \cite{Deshpande12}, we consider a threshold estimator for each vertex $i$ based on a sum over
$\hat C.$    If $K = \Theta(\sqrt{n})$ as considered in \cite{Deshpande12}, then $\lambda$ being a constant implies that
 the mean $\mu$ does not converge to zero. In this case if $| \hat C  \triangle C^* |  = o(K),$
 the error incurred by  summing over $\hat C$ instead of over $C^*$  could be bounded by truncating $A_{ij}$ to a large magnitude
% the effect of summing
%over $\hat C$ instead of over $C^*$  could be bounded in   \cite{Deshpande12}  by truncating $A_{ij}$
%the edge random variables at a large magnitude
$\bar{\rho}$ and bounding the difference of sums
%over $\hat C \triangle C^*$
by $\bar{\rho}\big|  C^*\triangle \hat C  \big|  = o(K) \ll \mu K.$
However, for $K \gg \sqrt{n}$ with vanishing $\mu$ this approach fails.
Instead, we rely on the cleanup procedure in  \prettyref{alg:exactrecovery}
which entails running  \prettyref{alg:MP} for $1/\delta$ times on subsampled vertices.   A related difference we encounter is that
if $K$ is large enough then the condition $\lambda > 1/\eexp$ alone is not sufficient for exact recovery, but
adding the information-theoretic condition \prettyref{eq:submat-mle-suff} suffices.

Lastly, the method of moment requires $f(\cdot, t)$ to be a polynomial so that the exponential function \prettyref{eq:fexp}, which results in the ideal state evolution \prettyref{eq:mu-ideal}, cannot be directly applied. It is shown in \cite[Lemma 2.3]{Deshpande12} that for any $\lambda>1/e$ and any threshold $M$
there exists $d^*=d^*(\lambda,M)$ so that taking $f$ to be the truncated Taylor series of \prettyref{eq:fexp} up to degree $d^\ast$ results in the state evolution $\hat \mu_t$ which exceeds $M$ after some finite time $t^*(\lambda,M)$;
%suffices for our purpose to obtain weak or exact recovery.
however, no explicit formula of $d^\ast$, which is needed to instantiate \prettyref{alg:MP}, is provided.
Although in principle this does not pose any algorithmic problem as $d^*$ can be found by exhaustive search in $O(1)$ time independent of $n$, it is more satisfactory to find the best polynomial message passing rule explicitly which maximizes the signal-to-noise ratio subject to degree constraints (\prettyref{lmm:hermite}) and provides an explicit formula of $d^*$ as a function of $\lambda$ only
%leads to the optimal polynomial construction with an explicit formula of $d^\ast$,
(\prettyref{rmk:choiced}).

%we use the optimal polynomial of a given degree...

%\cite[Lemma 2.3]{Deshpande12}
%existential. Of course one can find the degree by exhaustive search which can be done in $O(1)$ time which only depends on $\lambda$.
%It is more satisfactory to have an procedure which is both explicit and optimal.
 \end{remark}

\section{The Gaussian biclustering problem}   \label{sec:Gaussian_bi_cluster}

%\nbr{to be rewritten}

We return to the biclustering problem where the goal is to locate a submatrix whose row and column support need not coincide. Consider
the model \prettyref{eq:model} parameterized by $(n_1, n_2, K_1, K_2, \mu)$ indexed by a common $n$ with $n\to \infty.$
%A problem related to the recovery of a principal submatrix from an observed symmetric
%Gaussian matrix is the Gaussian biclustering
%problem, for which an $n_1 \times n_2$ matrix of independent, Gaussian
%random variables is observed.   For $C^*_1 \subset [n_1]$ and $C^*_2 \subset [n_2]$ with $|C^*_1|=K_1$
%and $|C^*_2|=K_2$, it is assumed that $W_{ij} \sim {\cal N}(\mu, 1)$ if $(i,j)\in C^*_1\times C^*_2$
%and $W_{ij} \sim {\cal N}(0, 1)$ else.   We consider a sequence of problems, so
%$(n_1, n_2, K_1, K_2, \mu)$ depends on an index $n$ with $n\to \infty.$
In \prettyref{sec:bi_cluster_info_limits} we present the information limits for weak
and exact recovery for the Gaussian bicluster model. The sharp conditions given
for exact recovery are from  Butucea et al.\  \cite{Butucea2013sharp}, and
calculations from  \cite{Butucea2013sharp} with minor adjustment provide conditions for weak recovery
as well.
\prettyref{sec:bi_clusterMP} shows how the optimized message passing algorithm
and its analysis can be extended from the symmetric case to the asymmetric case for biclustering and compares its performance to the fundamental limits.
%from the principal submatrix recovery problem to the bicluster
%recovery problem.
As originally observed in  \cite{HajekWuXu_one_info_lim15} for recovering the principal submatrix, the connection between weak and exact recovery via the voting procedure extends to the biclustering problem as well.

\subsection{Information-theoretic limits for Gaussian biclustering}   \label{sec:bi_cluster_info_limits}

Information-theoretic conditions ensuring exact recovery of  both $C_1^*$ and $C_2^*$
by the maximal likelihood estimator (MLE), \ie,
\[
(\hat C_1^{\sf MLE}, \hat C_2^{\sf MLE}) = \argmax_{\substack{|C_1|=K_1\\|C_2|=K_2}}\sum_{\substack{i\in C_1\\j\in C_2}} W_{ij}
\]
are obtained in Butucea et al.~\cite{Butucea2013sharp}.
While  \cite{Butucea2013sharp} does not focus on conditions for weak recovery, the calculations therein
combined with the voting procedure for exact recovery described in \cite{HajekWuXu_one_info_lim15} in fact resolve the information limits
% gives a
%picture of the information requirements
for both weak and exact recovery in the bicluster Gaussian model.   Throughout this section
we assume that $K_i = o(n_i)$ for $i=1,2$.   For the converse results we assume
$C_i^*$ is a subset of $[n_i]$ of cardinality $K_i$
selected uniformly at random  for $i =1,2$, with $C^*_1$ independent of $C_2^*.$
Let $\lambda_i = \frac{K_i^2 \mu^2}{n_i}$ for $i =1,2.$
The voting procedure mentioned in the theorems below is the cleanup procedure described
in \prettyref{alg:exactrecovery}; it uses the method of successive withholding.

\begin{theorem}[Weak recovery thresholds for Gaussian biclustering]   \label{thm:bi_cluster_weak}~~\\
(i) If
\begin{align}
\liminf_{n\to \infty}  \frac{   \mu \sqrt{K_1K_2}  }{\sqrt{2( K_1 \log (  n_1/K_1  ) + K_2 \log( n_2 /K_2 )) }}   > 1,     \label{eq:B_suff0}
\end{align}
then both $C_1^*$ and $C_2^*$ can be weakly recovered by the MLE.  Conversely, if both $C_1^*$ and $C_2^*$ can be
weakly recovered by some estimator,  then
\begin{align}
\liminf_{n\to \infty}  \frac{   \mu \sqrt{K_1K_2}  }{\sqrt{2( K_1 \log (  n_1/K_1  ) + K_2 \log( n_2 /K_2 ) }}   \geq  1.    \label{eq.B_nec}
\end{align}
(ii) If
\begin{align}   \label{eq:weak_C2_line_sum}
\liminf_{n\to\infty}    \frac{K_1^2\mu^2}{2n_1\log(n_2/K_2)}  > 1,
\end{align}
or, equivalently,     $\liminf_{n\to\infty}    \frac{\lambda_1}{2 \log(n_2/K_2)}  > 1,$  then $C_2^*$ can be weakly recovered by
column sum thresholding.   Similarly, if
\begin{align}   \label{eq:weak_C1_line_sum}
\liminf_{n\to\infty}    \frac{K_2^2\mu^2}{2n_2\log(n_1/K_1)}  > 1 ,
\end{align}
then $C_1^*$ can be weakly recovered by row sum thresholding.  \\
(iii)    Suppose for some small $\delta > 0$ that
$C_2^*$ can be weakly recovered even if a  fraction $\delta$ of the rows of the matrix are
hidden.    Then $C^*_1$ can be weakly recovered by the voting procedure if
\begin{align}  \label{eq:suff_cond_voting_weak}
\liminf_{n\to\infty}  \frac{ K_2\mu^2 }{2\log(n_1/K_1)} > 1.
\end{align}
\end{theorem}

\begin{theorem}[Exact recovery thresholds for Gaussian biclustering]    \label{thm:bi_cluster_exact} ~~\\
(i)  If for some small $\delta > 0,$  $C_2^*$ can be weakly recovered even if a fraction $\delta$ of the
rows of the matrix are hidden, and if
\begin{align}
\liminf_{n\to \infty}   \frac{ \sqrt{K_2}\mu }{ \sqrt{ 2 \log K_1} + \sqrt{ 2  \log n_1 } } >1,  \label{eq:submat_vote_1}
\end{align}
then $C_1^*$ can be exactly recovered by the voting procedure.
Similarly, if for some small $\delta > 0,$  $C_1^*$ can be weakly recovered even if a fraction $\delta$ of the
columns of the matrix are hidden, and if
\begin{align}
\liminf_{n\to \infty}   \frac{ \sqrt{K_1}\mu }{ \sqrt{ 2 \log K_2} + \sqrt{ 2  \log n_2 } } >1,  \label{eq:submat_vote_2}
\end{align}
then $C_2^*$ can be exactly recovered by the voting procedure. \\
(ii) The set  $C_2^*$ can be exactly recovered  by column sum thresholding if
\begin{align}   \label{eq:exact_C2_line_sum}
\liminf_{n\to\infty}    \frac{K_1\mu}{\sqrt{2n_1}(\sqrt{\log K_2} +  \sqrt{\log n_2 })  }   > 1,
\end{align}
or, equivalently,   $\liminf_{n\to\infty}    \frac{\lambda_1}{(\sqrt{\log K_2} +  \sqrt{\log n_2})^2}  > 2.$
%\nb{JX: the constant should be $2$ instead of $4$. Please double check it.}
(A similar condition holds for exact recovery of $C_1^*.$)
\end{theorem}

The proofs of Theorems \ref{thm:bi_cluster_weak} and \ref{thm:bi_cluster_exact} are given in \prettyref{app:bicluter-mle}.
The condition involving $\delta$ in \prettyref{thm:bi_cluster_weak}(iii) and \prettyref{thm:bi_cluster_exact}(i)
requires a certain robustness of the estimator for weak recovery.   If  the rows indexed by a set $S,$ with $S\subset [n_1]$
and $|S|=\delta n_1,$  are hidden, then the observed matrix has dimensions $n_1(1-\delta) \times n_2$ and the planted
submatrix has $K_1-|S_1\cap C_1^*| \approx K_1(1-\delta)$ rows and $K_2$ columns.   It is shown in
\cite[Section 3.3]{HajekWuXu_one_info_lim15} that the MLE  has this robustness property for weak recovery
of a principal submatrix, and a similar extension can be established for weak recovery for biclustering.
The estimator used is the MLE based on the assumption that the submatrix to be found has shape
$K_1(1-\delta) \times K_2.$
With that extension in hand,  the following corollary is a consequence of the two theorems, and it recovers
the main result of \cite{Butucea2013sharp}.

\begin{corollary}   \label{cor:exact_bicluster}
If \eqref{eq:B_suff0}, \eqref{eq:submat_vote_1}, and \eqref{eq:submat_vote_2} hold, then $C_1^*$ and $C_2^*$ can
both be exactly recovered by the MLE.    Conversely, if exact recovery is possible, then  \eqref{eq.B_nec} holds,
and  both  \eqref{eq:submat_vote_1}  and \eqref{eq:submat_vote_2} hold with ``$>$'' replaced by ``$\geq$''.
%strict inequalities replaced by weak (i.e. $\geq$) inequalities.
\end{corollary}

We conclude this subsection with a few remarks on Theorems \ref{thm:bi_cluster_weak} and \ref{thm:bi_cluster_exact}:
\begin{enumerate}
\item
As one might expect from the theorems themselves, the following implications hold:
any of  \eqref{eq:B_suff0}, \eqref{eq:weak_C1_line_sum}, or  \eqref{eq:submat_vote_1}  implies \eqref{eq:suff_cond_voting_weak} (dropping the second term in the denominator
on the left-hand side of \eqref{eq:B_suff0} yields \eqref{eq:suff_cond_voting_weak});
 \eqref{eq:exact_C2_line_sum} implies \eqref{eq:weak_C2_line_sum}.

\item
If we let $K_1/n_1 = K_2/n_2$, then $\mu$ can be selected so that: \eqref{eq:weak_C2_line_sum} holds (so $C^*_2$ can
be weakly recovered) but  \eqref{eq:suff_cond_voting_weak} fails\footnote{In this paragraph,
by ``\eqref{eq:suff_cond_voting_weak} fails" we mean $\limsup < 1$)}  if and only if
$K_1^2/(n_1K_2) > 1,$  or equivalently,  $K_1/n_2 > 1.$   This condition implies $n_2 < K_1 = o(n_1)$;
$A$ is a tall thin matrix.
Even if $C_2^*$ were exactly recovered, voting does not provide weak recovery of $C_1^*$ if
\eqref{eq:suff_cond_voting_weak} fails.
If $C_2^*$ is given exactly (for example, by a genie)  the optimal way to recover $C_1^*$ is by voting, which fails if
\eqref{eq:suff_cond_voting_weak} fails.   Thus, in this regime, weak recovery of $C_2^*$ is possible
while weak recovery of $C_1^*$ is impossible.

\item
If   $n_1=n_2$ and $K_1=K_2,$  the sufficient conditions and the necessary conditions for weak
and for exact recovery, respectively,  are identical to those in  \cite{HajekWuXu_one_info_lim15}
for the recovery of a $K\times K$ principal submatrix with elevated mean, in a symmetric
$n\times n$ Gaussian matrix.  Basically, in the bicluster problem the data matrix provides roughly
twice the information (because the matrix is not symmetric) and there is
twice the information to be learned, namely $C_1^*$ and $C_2^*$ instead of only $C^*,$
and the factors of two cancel to yield the same conditions.   It therefore follows from
\cite[Remark 7]{HajekWuXu_one_info_lim15},  that if  $n_1=n_2$ and $K_1=K_2\le n_1^{1/9},$
then \eqref{eq:B_suff0} implies \eqref{eq:submat_vote_1} and \eqref{eq:submat_vote_2};  in this
regime, \eqref{eq:B_suff0} alone is the sharp condition for both weak and exact recovery.

\item If $\frac{K_i^2 \mu^2}{n_i}\equiv \lambda_i$ for positive constants $\lambda_1$ and $\lambda_2$ and if $K_1 \asymp K_2,$
then \eqref{eq:B_suff0} holds for all sufficiently large $n$, so weak recovery is information theoretically possible.
In contrast, our proof that the optimized message passing algorithm provides weak recovery in this regime
requires $(\lambda_1, \lambda_2) \in {\cal G}.$

\item
Either \eqref{eq:B_suff0} or  \eqref{eq:weak_C2_line_sum} suffices
for the weak recovery of $C_2^*.$   We leave it as an open problem to determine whether
there is a sharp converse for these conditions, or whether there is yet another sufficient
condition for weakly recovering $C_2^*$ only.
\end{enumerate}

\subsection{Message passing algorithm for the Gaussian biclustering model}  \label{sec:bi_clusterMP}

Suppose $n_i \to \infty$ and  $\Omega(\sqrt{n_i})  \leq K_i \leq o(n_i)$ for $i \in \{0,1\},$ as $n\to \infty.$
The belief propagation algorithm and our analysis of it for recovery of a single set of indices
can be naturally adapted to the biclustering model.

Let $f(\cdot, t): \reals \to \reals$ be a scalar function for each iteration $t$. To be definite, we shall describe the algorithm such
that at each iteration, the messages are passed either from the row indices to the column indices, or vice-versa, but
not both.   The messages are defined as follows for $t\geq 0:$
\begin{align}
\mbox{($t$ even)}~~~~~& \theta^{t+1}_{i \to j} =  \frac{1}{\sqrt{n_2}}  \sum_{\ell \in [n_2] \backslash \{j \} } W_{\ell i} f( \theta^t_{\ell \to i }, t), \quad \forall i\in [n_1], j\in [n_2]
    \label{eq:theta_update_ij_even}    \\
\mbox{($t$ odd)}~~~~~&  \theta^{t+1}_{j\to i} =  \frac{1}{\sqrt{n_1} }  \sum_{\ell \in [n_1] \backslash \{i\} } W_{\ell j} f( \theta^t_{\ell \to j}, t), \quad \forall  j \in [n_2], i \in [n_1],
    \label{eq:theta_update_ij_odd}
\end{align}
with the initial condition $\theta^0_{\ell \to i}= 0$ for $(\ell,i)\in [n_2]\times [n_1].$
Moreover, let the aggregated beliefs be given by
\begin{align}
\mbox{($t$ even)}~~~~~& \theta^{t+1}_{i} =  \frac{1}{\sqrt{n_2}}  \sum_{\ell \in [n_2] } W_{\ell i} f( \theta^t_{\ell \to i }, t), \quad \forall i\in [n_1]
  \label{eq:theta_last_update_even}   \\
\mbox{($t$ odd)}~~~~~&  \theta^{t+1}_{j} =  \frac{1}{\sqrt{n_1} }  \sum_{\ell \in [n_1] } W_{\ell j} f( \theta^t_{\ell \to j}, t), \quad \forall  j \in [n_2].
    \label{eq:theta_last_update_odd}
\end{align}

Let $\lambda_i = \frac{K_i^2 \mu^2}{n_i}$ for $i =1,2.$
Suppose as $n \to \infty$,
for $t$ even (odd), $\theta_i^t$  is approximately $\calN(\mu_t, \tau_t)$ for $i \in C^*_1 $ ($i \in C^*_2$)
and $\calN(0, \tau_t)$ for $i \in [n_1]\backslash C^*_1$ ($i \in [n_2] \backslash C^*_2$).
Then similar to the symmetric case, the update equations of  message passing and the fact that $\theta_{i \to j}^t= \theta_i^t$ for all $i,j$
suggest the following state evolution equations for $t \ge 0$:
\begin{align}
\mu_{t+1}^2 & =
\begin{cases}
\sqrt{\lambda_1} \expect{ f( \mu_t + \sqrt{\tau_t} Z, t) } & \mbox{$t$ even} \\
\sqrt{\lambda_1} \expect{ f( \mu_t + \sqrt{\tau_t} Z, t) } & \mbox{$t$ odd}
\end{cases}\\
%%\mbox{($t$ even)}~~~~~&   \mu_{t+1} =\sqrt{\lambda_1} \expect{ f( \mu_t + \sqrt{\tau_t} Z, t) },   \\
% \mbox{($t$ odd)}~~~~~& \mu_{t+1}=\sqrt{\lambda_2} \expect{ f( \mu_t + \sqrt{\tau_t} Z, t) }, \\
\tau_{t+1} & = \expect{ f( \sqrt{\tau_t} Z, t )^2}.
\end{align}
The optimal choice of $f$ for maximizing the signal-to-noise ratio $\frac{\mu_{t+1}}{\sqrt{\tau_{t+1} } }$ is again
$f(x, t)= \eexp^{x \mu_t - \mu_t^2}$. With this optimized $f$, we have $\tau_{t+1}=1$ and the state evolution equations reduce to
\begin{align}
%\mbox{($t$ even)}~~~~~\mu_{t+1}^2=\lambda_1 \eexp^{\mu_t^2}  \\
%\mbox{($t$ odd)}~~~~~\mu_{t+1}^2=\lambda_2 \eexp^{\mu_t^2}
\mu_{t+1}^2=
\begin{cases}
\lambda_1 \eexp^{\mu_t^2} & \mbox{$t$ even} \\
\lambda_2 \eexp^{\mu_t^2} & \mbox{$t$ odd}
\end{cases}
\end{align}
with $\mu_0=0$.

To justify the state evolution equations,
we rely on the method of moments, requiring $f$ to be polynomial. Thus, we choose
$f=f_d( \cdot, t)$ as per  \prettyref{lmm:hermite}, which maximizes the signal-to-noise ratio
among all polynomials with degree up to $d$.
With $f=f_d$, we have $\tau_{t+1}=1$ and  the state evolution equations reduce to
\begin{align}
%\mbox{($t$ even)}~~~~~ \hat\mu_{t+1}^2=\lambda_1 G_d(\hat\mu_t^2) , \label{eq:mu-poly_bicluster_even}\\
%\mbox{($t$ odd)}~~~~~\hat \mu_{t+1}^2=\lambda_2 G_d(\hat\mu_t^2)\label{eq:mu-poly_bicluster_odd},
\mu_{t+1}^2=
\begin{cases}
\lambda_1 G_d(\mu_t^2) & \mbox{$t$ even} \\
\lambda_2 G_d(\mu_t^2) & \mbox{$t$ odd}
\end{cases}
\label{eq:mu-poly_bicluster}
\end{align}
where $G_d(\mu)=\sum_{k=0}^d \frac{\mu^{k}}{k!}$.

Combining message passing with spectral cleanup, we obtain the following algorithm for estimating $C_1^*$ and $C_2^*$.

\begin{algorithm}[htb]
\caption{Message passing for biclustering}\label{alg:bicluster_MP}
\begin{algorithmic}[1]
\STATE Input: $n_1, n_2, K_1, K_2 \in \naturals$, $\mu>0$, $W \in \reals^{n_1\times n_2}$, $d^\ast \in \naturals,$  $t^*\in 2\naturals,$ and $s^* > 0.$
\STATE Initialize: $\theta_{\ell \to i}^0=0$ for $(\ell,i) \in [n_2]\times [n_1].$ %$\theta_{i}^0=0$.
For $t \geq 0$, define the sequence of degree-$d^\ast$ polynomials $f_{d^*}(\cdot, t)$
as per  \prettyref{lmm:hermite} and  $\hat{\mu}_t$  according to \prettyref{eq:mu-poly_bicluster}.
\STATE Run $t^\ast$ iterations of message passing as in  \eqref{eq:theta_update_ij_even} and \eqref{eq:theta_update_ij_odd}
with $f=f_{d^*}$ and compute $\theta_{i}^{t^{\ast} }$ for all $i \in [n_1]$ as per \eqref{eq:theta_last_update_even}
and $\theta_{j}^{t^{\ast}+1 }$ for all $j  \in [n_2]$ as per \eqref{eq:theta_last_update_odd}.
\STATE Find the sets $\tilde{C}_1=\{ i \in [n_1]: \theta_{i}^{t^\ast} \ge \hat{\mu}_{t^\ast}/2 \}$ and  $\tilde{C}_2=\{ j \in [n_2]: \theta_{j}^{t^\ast+1} \ge \hat{\mu}_{t^\ast+1}/2 \}$.
\STATE (Cleanup via power method)
%Let $A=W/( \sqrt{n_1}+\sqrt{n_2} ) $.
% and let $\tilde{A}$ denote
%$A$ restricted to  entries in $\tilde{C}_1 \tilde C_2$.
Denote the restricted matrix $W_{\tilde{C}_1 \tilde C_2}$ by $\tilde W$.
Sample $u^0$ uniformly from the unit sphere in $\reals^{\tilde{C}_1}$ and compute
$u^{t+2}= \tilde W \tilde{W}^\top u^t / \|\tilde{W} \tilde{W}^\top u^t\|$,  for $t$ even
%where $\gamma_t$ is selected to make $\| u^{t+2} \|=1$,
and
$ 0\leq t \leq   2  \lceil s^*\log (n_1 n_2) \rceil -2.$  Let $\hat{u}=u^{2\lceil s^*\log (n_1 n_2) \rceil }.$
Return  $\hat{C}_1,$  the set of  $K_1$ indices $i$ in $\tilde C_1$ with the largest values of $|\hat{u}_i|.$
Compute the power iteration with $ \tilde{W}^\top\tilde{W}$
 for odd values of $t$ and return $\hat{C}_2$ similarly.
%\STATE Let $A|_{\tilde{C}}$ denote the restriction of $A$ to the rows and columns with index in $\tilde{C}$, and estimate
%its principal eigenvector by running the power method for $s^\ast \log n $ iterations; denote the
%estimate by $u \in \reals^{|\tilde{C}|}.$
%\STATE Return  $\hat{C} \subset [n]$ as the set of the top $K$ entries of $u$ in absolute value.
\end{algorithmic}
\end{algorithm}

We now turn to the performance of \prettyref{alg:bicluster_MP}. Let
\begin{align}
	{\cal G} &= \{ (\lambda_1, \lambda_2) :  \mu_t \to \infty\}, \\
	\label{eq:calG}
{\cal G}_d &= \{ (\lambda_1, \lambda_2) :  \hat \mu_t \to \infty\}.
\end{align}
As $d \to \infty$, $G_d(\mu) \to \eexp^{\mu}$ uniformly over bounded intervals.
It suggests that if $(\lambda_1, \lambda_2) \in \calG$, then there exists a $d^\ast(\lambda_1, \lambda_2)$
such that $(\lambda_1, \lambda_2) \in \calG_{d^*}$ and hence $\hat \mu_t \to \infty$ as $t \to \infty.$
The following theorem confirms this intuition, showing that the bicluster message passing algorithm
(\prettyref{alg:bicluster_MP})  approximately recovers  $C_1^*$ and $C_2^*$, provided that $(\lambda_1, \lambda_2) \in {\cal G}.$
\begin{theorem}  \label{thm:weakMP_submat_bicluster}
Fix $\lambda_1, \lambda_2 >0.$   Suppose
$ \frac{ K_i^2\mu^2}{n_i} \to  \lambda_i$,  $K_1 \asymp K_2$,  and $\Omega(\sqrt{n_i}) \leq K_i \leq o(n_i)$ as $n\to \infty,$  for
$i=1,2$.
Consider the model \prettyref{eq:model} with $|C^*_i|/K_i \to 1$ in probability as $n\to \infty.$
Suppose $(\lambda_1, \lambda_2) \in {\cal G}$ and define $d^\ast(\lambda_1, \lambda_2)$ as in \prettyref{eq:defd_bicluster}.
For every $\eta \in (0,1),$  there exist explicit positive constants $ t^\ast, s^\ast$ depending on  $(\lambda_1, \lambda_2,\eta)$
such that \prettyref{alg:bicluster_MP} returns
$ | \hat{C}_i \cap C_i^\ast| \ge (1- \eta )K_i$ for $i =1,2$ with probability converging to $1$ as $n \to \infty$, and the
total running time is bounded by $c (\eta, \lambda_1, \lambda_2) n_1n_2 \log (n_1n_2)$, where
$c(\eta, \lambda_1,\lambda_2) \to \infty$ as either $\eta \to 0$ or $(\lambda_1, \lambda_2)$ approaches $\partial \cal G$.
\end{theorem}

\begin{figure}[ht]
	\centering
	\includegraphics[width=0.4\textwidth]{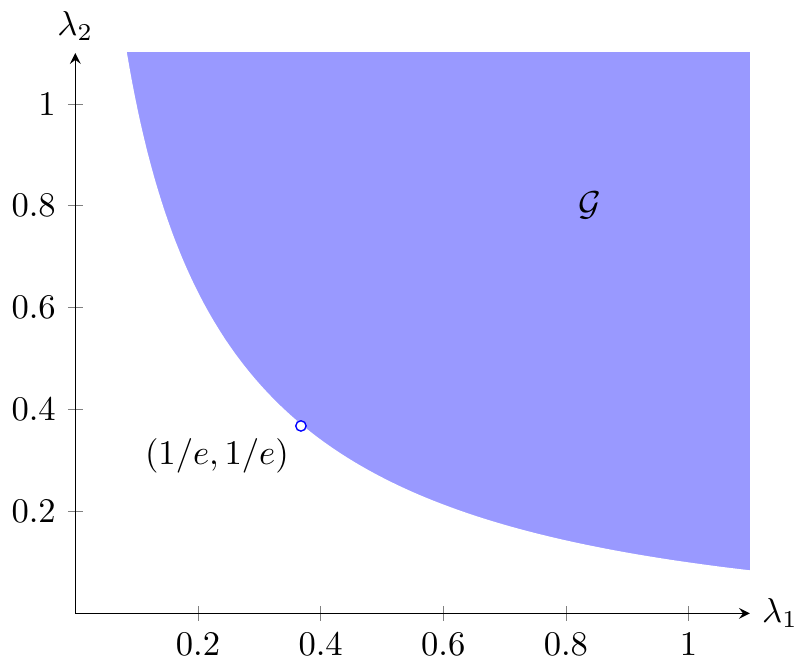}
	\caption{Required signal-to-noise ratios by \prettyref{alg:bicluster_MP} for biclustering.}
	\label{fig:calG}
\end{figure}

\begin{remark}[Exact biclustering via message passing]
If  the assumptions  of   \prettyref{thm:weakMP_submat_bicluster} hold
and the voting condition \eqref{eq:submat_vote_1} (respectively, \eqref{eq:submat_vote_2})  holds, then $C^*_1$ (respectively, $C^*_2$)
can be exactly recovered by a voting procedure similar to the one in \prettyref{alg:exactrecovery}.
Similar to the analysis in the symmetric case (cf. \prettyref{fig:MP_phase_plot}), whenever
\eqref{eq:submat_vote_1} -- \eqref{eq:submat_vote_2} imply the sufficient condition for message passing, \ie,
$(\lambda_1,\lambda_2) \in \calG$ defined in \prettyref{eq:calG},
there is no computational gap for exact recovery.

To be more precise, consider $K_i = \frac{\rho_i n}{\log n}$ for $i=1,2$. Then
	\eqref{eq:submat_vote_1} and \eqref{eq:submat_vote_2} are equivalent to $\lambda_i > 8 \rho_i$.
Thus, whenever $K_1$ and $K_2$ are large enough so that
$(8 \rho_1, 8 \rho_2)$ lies in the closure $\textbf{cl}(\calG)$,
	or more generally,
		\begin{equation}
%	(\rho_1,\rho_2)
\pth{\liminf_{n \to \infty} \frac{K_1 \log n_1}{n_1}, \;  \liminf_{n\to \infty} \frac{K_2 \log n_2}{n_2}}
	 \in \frac{1}{8} \textbf{cl}(\calG)	
	\label{eq:K1K2exact}
\end{equation}	
then  \prettyref{alg:bicluster_MP} plus voting achieves information-theoretically exact recovery threshold with optimal constants (i.e. it is successful if \eqref{eq:submat_vote_1} and \eqref{eq:submat_vote_2} hold).
%(PREVIOUS VERSION WAS NOT CORRECT--CONDITION {eq:K1K2exact} DOES NOT INVOLVE $\mu.$)
% implies that $(\lambda_1,\lambda_2) \in \calG$ and hence \prettyref{alg:bicluster_MP} plus voting achieves information-theoretically exact recovery threshold with optimal constants.
This result can be viewed as a two-dimensional counterpart of \prettyref{eq:Kexact} obtained for the symmetric case.
		
\label{rmk:K12exact}
\end{remark}

\begin{remark}   \label{rmk:partialG}
Clearly $\cal G$ is an open subset of $\reals_+^2$ and ${\cal G}$ is an upper closed set.
Let $\partial \cal G$ denote its boundary and let $\phi(x)   \triangleq \lambda_2\eexp^{\lambda_1 \eexp^x}$, so
that $\mu_{t+2}^2=\phi(\mu_t^2)$ for $t$ even.
Note that $(\lambda_1, \lambda_2) \in \partial \cal G$ if and only if the function
 is such that for some $x > 0$,  $\phi(x)=x$ and
$\phi'(x)=1.$     Since $\phi'(x)=\phi(x)y$,  where $y=\lambda_1 \eexp^x,$ it follows that
$xy=1$ where $y=\lambda_1 \eexp^x$ and $x=\lambda_2 \eexp^y.$
Therefore, it is convenient to express the boundary of $\calG$ in the parametric form
\[
\partial {\cal G}= \{  (y\eexp^{-1/y}, y^{-1} \eexp^{-y} ) : y>0 \}.
\]     It follows
that $(1/\eexp, 1/\eexp) \in \partial G$ and
$\{ (\lambda_1, \lambda_2) \in \reals_+^2:  \lambda_1\lambda_2 \geq \eexp^{-2}\}\backslash \{(1/\eexp,1/\eexp)\}  \subset  {\cal G}.$
Boundaries of $\calG_d$ can be determined similar to \prettyref{eq:ga} (see \prettyref{fig:Gd} for plots).
\begin{figure}
	\centering
	\includegraphics[width=0.35\textwidth]{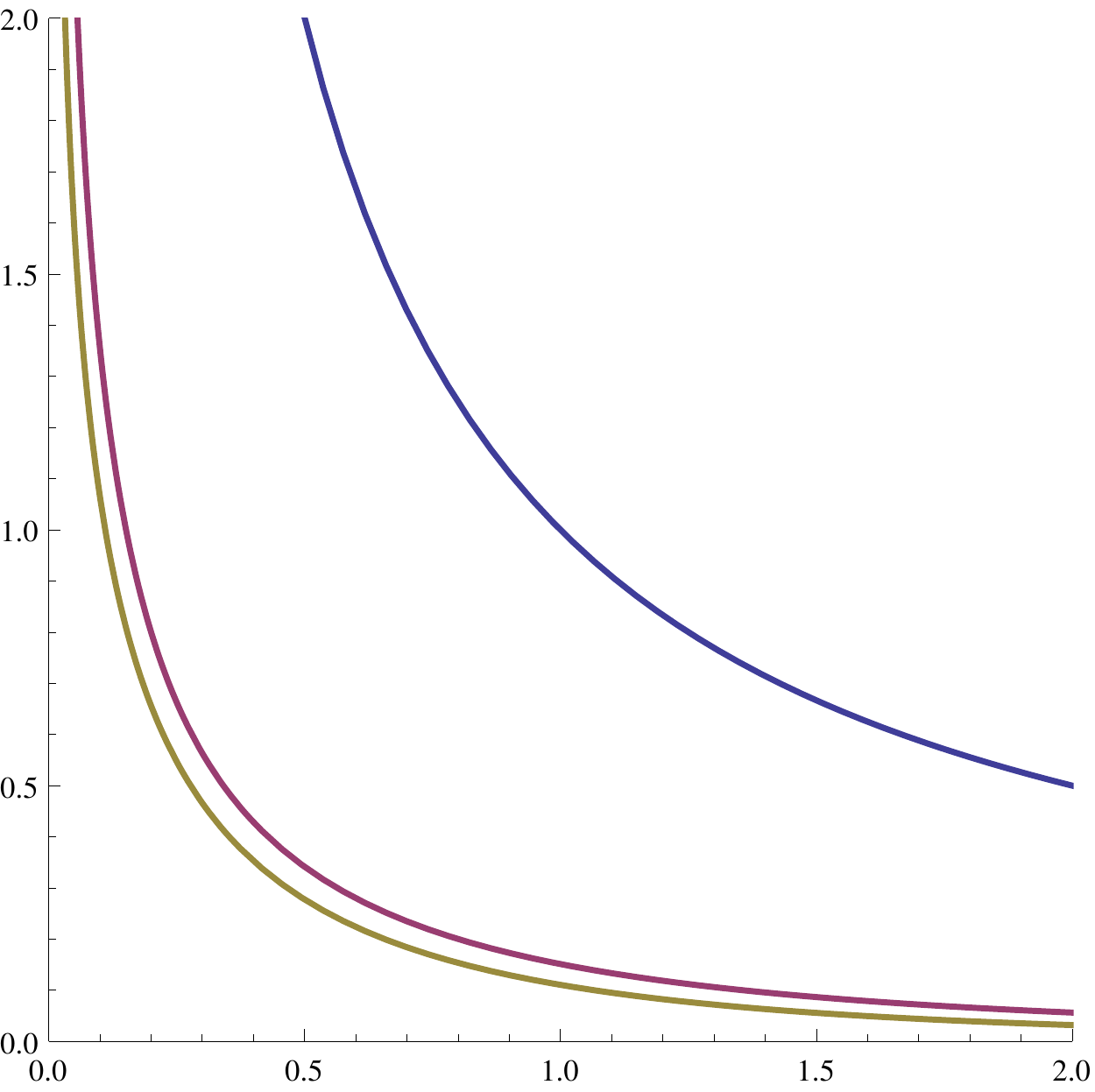}
\caption{Boundaries of the regions $\calG_d$ for $d=1,2,3$; as $d$ increases, $\calG_d$ converges to $\calG$ in \prettyref{fig:calG}.} \label{fig:Gd}
\end{figure}
\end{remark}

\begin{proof}[Proof of  \prettyref{thm:weakMP_submat_bicluster}]
The proof  follows step-by-step that of  \prettyref{thm:almost_exactBP_submat}; we shall point
out the minor differences. Given $\lambda_1$ and $\lambda_2$, define
\begin{align}
d^\ast (\lambda_1, \lambda_2) = \inf\{ d \in \naturals: (\lambda_1, \lambda_2) \in \calG_d\},
\label{eq:defd_bicluster}
\end{align}
and choose $c_0>0$ so that \prettyref{eq:c0} holds.
Given any $\eta \in (0,1)$, choose an
arbitrary $\epsilon \in (0, \epsilon_0)$ such that $\eta(\epsilon)$ defined in \prettyref{eq:defeta_bicluster}
is at most $\eta$. Notice that $\epsilon_0$ is determined by $c_0$.
Let $M=8 \log (1/\epsilon)$ and choose
\begin{equation}
	t^*(\lambda_1, \lambda_2,M) = \inf\left \{t: \min\{ \hat\mu_t , \hat\mu_{t+1} \} > M\right \}. \label{eq:deft_bicluster}
\end{equation}
In view of \prettyref{lmm:hermitebicluster} and the assumption that $(\lambda_1, \lambda_2) \in \calG$,
$d^\ast$ is finite. Since $  (\lambda_1, \lambda_2) \in \calG_{d^*}$, it follows that $\hat \mu_t \to \infty$ and thus $t^*(\lambda_1, \lambda_2, M)$ is finite.

The assumptions of  \prettyref{thm:weakMP_submat_bicluster} imply that $n_1 \asymp n_2.$
Lemmas \ref{lmm:momentmatching} -  \ref{lmm:momentvariance} therefore go through as before,
with $n$ in the upper bounds taken to be $\min\{n_1, n_2\}$, so that $\frac{1}{\sqrt{n_i}} \leq \frac{1}{\sqrt{n}}.$
This modification then implies that  \prettyref{lmm:approximatecentrallimit}, justifying the state evolution
equations, goes through as before.

The correctness proof for the spectral clean-up procedure in \prettyref{alg:bicluster_MP} is given by \prettyref{lmm:power_clean_bicluster} below with  $s^\ast$ defined by \prettyref{eq:sstar_bicluster}; it
 is similar to \prettyref{lmm:power_clean} used in \prettyref{thm:almost_exactBP_submat} but applies to rectangular matrices and uses singular
value decomposition.  To ensure that $c_0$ is well-defined, the following condition is used for \prettyref{lmm:power_clean_bicluster}:
\begin{align}
\frac{ \mu\sqrt{K_1K_2} }{ \sqrt{n_1} + \sqrt{n_2} } =\Omega(1),   \label{eq:spectral_condition}
\end{align}
which is equivalent to $\min\{\frac{\lambda_1 K_2}{K_1},  \frac{\lambda_2 K_1}{K_2}\} = \Omega(1)$ and implied by
%$\left(\frac{ \mu\sqrt{K_1K_2} }{ \sqrt{n_1}} =\Omega(1)\mbox{ and }
%\frac{ \mu\sqrt{K_1K_2} }{ \sqrt{n_2}} =\Omega(1)\right),$  or equivalently, \\
%$\left(   \frac{\lambda_1 K_2}{K_1} = \Omega(1) \mbox{ and } \frac{\lambda_2 K_1}{K_2} = \Omega(1)\right).$
the assumptions of \prettyref{thm:weakMP_submat_bicluster}, completing the proof of the theorem.
 (In fact, given the
first condition of \prettyref{thm:weakMP_submat_bicluster}, \ie, $\lambda_1, \lambda_2$ are fixed, \eqref{eq:spectral_condition} is equivalent to $K_1\asymp K_2.$)
%The correctness of the spectral cleanup procedure is justified in \prettyref{lmm:power_clean_bicluster} below, completing
%the proof of the theorem.
\end{proof}

\begin{lemma}   \label{lmm:hermitebicluster}
For $d\ge1$, $\calG_d \subset \calG_{d+1}$ with $\calG_1=\{ (\lambda_1, \lambda_2): \lambda_1 \lambda_2 \ge 1\}$,
and $\cup_{d=1}^{\infty} \calG_d = \calG$.	
\end{lemma}
\begin{proof}
By definition, $G_1(x) = 1+ x$ and thus for $t$ even, $\hat{\mu}_{t+2}^2= \lambda_2 (1+ \lambda_1 (1+ \hat{\mu}_t^2) )$.
As a consequence, $\hat{\mu_t} \to \infty$ if and only if $\lambda_1 \lambda_2 \ge 1$, proving the claim
for $\calG_1$.
Let $\phi_d(x)   \triangleq \lambda_2 G_d (\lambda_1 G_d (x) )$ so that $\hat{\mu}^2_{t+2}=\phi_d(\hat{\mu}^2_{t})$ for $t$
even .   The fact
$\calG_d \subset \calG_{d+1}\subset \calG$ follows from the fact $\phi_d(x) $ is increasing in $d$ and $\phi_d(x) < \phi(x),$
where $\phi$ is defined in \prettyref{rmk:partialG}.
To prove $\cup_{d=1}^{\infty} \calG_d = \calG,$  fix $(\lambda_1,\lambda_2) \in \calG.$   It suffices to show that
 $(\lambda_1,\lambda_2) \in \calG_d$ for $d$ sufficiently large.    Since $\phi_2(x)/x^4 \to \infty$ as $x\diverge$,  there exists an absolute constant $x_0>1$
such that $\phi_d(x) \geq x^2$  whenever $x\geq x_0$ and $d\geq 2.$     Let $t_0$ be an even number such
 that $\mu_{t_0}^2  > x_0.$   Since
$\phi_d(x)$ converges to $\phi(x)$ uniformly on bounded intervals, it follows that the first $t_0/2$ iterates using
$\phi_d$ converge to the corresponding iterates using $\phi.$   So, for $d$ large enough, $\hat{\mu}_{t_0}^2 > x_0,$
and hence, for such $d$, $\hat{\mu}_{t}^2  \to \infty$ as $t\to \infty,$ so $(\lambda_1,\lambda_2) \in \calG_d.$
\end{proof}

\begin{lemma} \label{lmm:power_clean_bicluster}
Suppose \eqref{eq:spectral_condition} holds, \ie,
\begin{equation}
\frac{ \mu\sqrt{K_1K_2} }{ \sqrt{n_1} + \sqrt{n_2} } \geq \frac{1}{c_0}	\label{eq:c0}
\end{equation}
for some $c_0>0$.
For $i=1,2$, suppose that
$\frac{|C_i^*|}{K_i}\rightarrow 1$ in probability and
$\tilde{C}_i$ is a set (possibly depending on $W$) such that
%\nbr{Bruce:  Given the other assumptions, we only need to assume \eqref{eq:ccA} and the second part of  \eqref{eq:ccC}, though other parts help in proof.}
\begin{align}
 \frac{1}{K_i} | \tilde{C}_i \cap C_i^\ast |  & \ge 1-\epsilon   \label{eq:ccA1}   \\
% \frac{1}{n_i} | \tilde{C}_i \backslash C_i^\ast | & \le \epsilon   \label{eq:ccB1} \\
K_i(1-\epsilon) \leq  |\tilde{C}_i| & \leq  n_i \epsilon  \label{eq:ccC1}
 \end{align}
hold for some $0< \epsilon < \epsilon_0$, where $\epsilon_0$ depends only on $c_0$.
Let
\begin{equation}
s^*= \pth{\log \frac{ 1 - \epsilon - 3 c_0 \sqrt{ h(\epsilon) + \epsilon}  }{  3 c_0 \sqrt{ h(\epsilon) + \epsilon} } }^{-1}
	\label{eq:sstar_bicluster}
\end{equation}
where $h(\epsilon)\triangleq \epsilon\log \frac{1}{\epsilon} + (1-\epsilon)\log \frac{1}{1-\epsilon}$ is the binary entropy function. Then $\hat C_i$ returned by \prettyref{alg:bicluster_MP} satisfies
$ | \hat{C}_i \cap C^\ast_i | \ge (1- \eta(\epsilon) )K_i $ for $i=1, 2$,
with probability converging to one as $n \to \infty,$ where
%$\eta(\epsilon)\to  0$ as $\epsilon \to 0.$
\begin{align}
\eta(\epsilon)= 2\epsilon + 650 c_0^2 \; \frac{h(\epsilon) + \epsilon} {(1-\epsilon)^2 }.
\label{eq:defeta_bicluster}
\end{align}
\end{lemma}

\begin{proof}  (Similar to proof of \prettyref{lmm:power_clean}.)  We prove the lemma for $\hat{C}_1$;
the proof for $\hat{C}_2$ is identical.
 For the first part of the proof we
assume that for $i=1,2$, $\tilde{C}_i$ is fixed, and later use a union bound over all possible choices of $\tilde{C}_i.$
%Let $W=W/( \sqrt{n_1}+\sqrt{n_2} ) $ so that $\var(W_{ij})=1/( \sqrt{n_1}+\sqrt{n_2} )^2$ and $\expect{W_{ij}}=\mu/( \sqrt{n_1}+\sqrt{n_2} ) $ for $(i,j) \in C_1^* \times C_2^*$ and zero otherwise.
Recall that $W_{\tilde{C}_1 \tilde C_2}$, which we abbreviate henceforth as $\tilde W$, is the matrix $W$ restricted to  entries in $\tilde{C}_1 \times \tilde C_2$.
Let $Z = \tilde W -\eexpect{\tilde W}$
and note that
\begin{equation}
\eexpect{ \tilde{W}} =   \mu \sqrt{ |\tilde C_1 \cap C_1^*| |\tilde C_2 \cap C_2^*|}  v_1v_2^\top
	\label{eq:ETW}
\end{equation}
is a rank-one matrix, where $v_i$ is the unit vector in
$\reals^{| \tilde{C}_i |  }$ obtained by normalizing the indicator vector of $\tilde{C}_i\cap C^*_i.$
Thus, thanks to \prettyref{eq:ccA1}, the leading singular value of
$\eexpect{ \tilde W} $ is at least $\mu \sqrt{K_1 K_2} (1-\epsilon) $  with  left singular vector $v_1$ and right singular vector $v_2$.
%,  and
%all other singular values equal to zero.

It is well-known (see, \eg, \cite[Corollary 5.35]{vershynin2010nonasym}) that if $M$ is an $m_1 \times m_2$  matrix with i.i.d.\ \emph{standard
normal} entries, then
$
\prob{ \| M \|  \geq \sqrt{m_1} + \sqrt{m_2} +  t } \leq   2 \eexp^{-t^2/2}.
$
Applying this result for $m_i=|\tilde C_i|$,
which satisfies $m_i \le \epsilon n_i$ by \prettyref{eq:ccC1},
 and $t= 2 \sqrt{h(\epsilon) (n_1+n_2)}$, we have for fixed $(\tilde C_1,\tilde C_2)$,
$$
\prob{      \|Z\|  \geq  (\sqrt{n_1} + \sqrt{n_2}) \beta           } \leq  2 \eexp^{- 2 (n_1+n_2)h(\epsilon)},
$$
where $\beta \triangleq 3 \sqrt{\epsilon+h(\epsilon)})$.
Similar to the proof of \prettyref{lmm:power_clean}, the number of $(\tilde C_1,\tilde C_2)$ that satisfies \prettyref{eq:ccC1} is at most $e^{(n_1+n_2)h(\epsilon)}$. By union bound,
if we drop the assumption that $\tilde C_i$ is fixed for $i=1,2$, we still have that with
high probability,
$  \|Z\|  \leq (\sqrt{n_1} + \sqrt{n_2}) \beta $.
%In the reminder of this proof we condition on this event.
%we shall assume $W$ and $\tilde C$ are
%fixed with  $  \|Z\|  \leq (\sqrt{n_1} + \sqrt{n_2}) \beta $.

Denote by $u$ the leading left singular vector of $W_{\tilde{C}_1 \tilde C_2}$.
Then
 \begin{align*}
 \|uu^\top -v_1 v_1^\top \|_{\rm F}
 = &~\sqrt{2} \|(\identity - uu^\top) v_1 v_1^\top \|_{\rm F} \overset{(a)}{=} \sqrt{2} \|(\identity - uu^\top) v_1 v_1^\top \|  \\
  \overset{(b)}{\leq} &~  \sqrt{2} \min\sth{\frac{\|Z\| }{|\sigma_1(\eexpect{ \tilde{W}}) - \sigma_2(\tilde{W})|}, 1 }
 \overset{(c)}{\leq} \frac{ 2 \sqrt{2} \|Z\| }{\sigma_1(\eexpect{ \tilde{W}})},
% \overset{(b)}{\leq} &~\sqrt{2} \min\sth{\frac{ \|Z\| }{|\sigma_1(\eexpect{ \tilde{W}}) - \sigma_2(\tilde{W})|} , 1} \\
% \overset{(c)}{\leq} &~\frac{ 2 \sqrt{2} \|Z\| }{\sigma_1(\eexpect{ \tilde{W}})},
 \end{align*}
 where (a) is because $\text{rank}((\identity - uu^\top) v_1 v^\top_1) \leq 1$,
% \nb{JX: I think $(\identity - uu^\top) v_1 v^\top_1$ has rank 1, and thus $(a)$ holds with constant $\sqrt{2}$.}   --Yes, I agree.  -Bruce
 (b) follows from Wedin's sin-$\theta$ theorem for SVD \cite{wedin72},
 (c) follows from Weyl's inequality $\sigma_2(\tilde{W}) \leq \sigma_2(\eexpect{ \tilde{W}}) + \|Z\|=\|Z\|$.
% the same argument as in \cite[Proposition 1]{CMW13}.\footnote{Namely, if $\sigma_1(\eexpect{ \tilde{W}}) - \sigma_2(\tilde{W}) \leq 2 \|Z\|$, then there is nothing to prove; if  $\sigma_1(\eexpect{ \tilde{W}}) - \sigma_2(\tilde{W}) > 2 \|Z\|$, then $\sigma_1(\eexpect{ \tilde{W}}) - \sigma_2(\tilde{W}) \geq \sigma_1(\eexpect{ \tilde{W}}) - \|Z\| \geq \sigma_1(\eexpect{ \tilde{W}})/2$, where we used Weyl's inequality $\sigma_2(\tilde{W}) \leq \sigma_2(\eexpect{ \tilde{W}}) + \|Z\|=\|Z\|$.}
%which is an extension of Davis-Kahan sin-$\theta$ theorem \cite{DavisKahan70} to asymmetric matrices,
In view of \prettyref{eq:ETW}, conditioning on the high-probability event that
$  \|Z\|  \leq (\sqrt{n_1} + \sqrt{n_2}) \beta $, we have
%since we assumed that $\|\tilde{W}\|_2 \leq \frac{\mu \sqrt{K_1K_2}(1-\epsilon)  \beta}{\sqrt{n_1} + \sqrt{n}_2 }, $ we get that
 \begin{equation}
	\|uu^\top -v_1 v_1^\top \|_{\rm F}  \leq \frac{ 2 \sqrt{2} \beta (\sqrt{n_1} + \sqrt{n_2})}{\mu(1-\epsilon) \sqrt{K_1K_2}} \leq
	\frac{ 2 \sqrt{2} c_0 \beta}{1-\epsilon},
	\label{eq:kahanuv1}
\end{equation}
where the last inequality follows from the standing assumption \prettyref{eq:c0}.

% $\sin \theta(u ,v_1) \leq  \frac{C^{''} \beta}{1-\beta}$ for some universal constant $C^{''}>0$, where $\theta(u_1, v_1)$
% denotes the angle between $u$ and $v_1$.
% Therefore,
% \begin{align}
% \|uu^\top -v_1 v_1^\top \|_{\rm F} \leq \sqrt{2} \|uu^\top -v_1 v_1^\top \|_2 \leq  \sqrt{2} \sin \theta(u ,v_1)  \leq  \frac{\sqrt{2} C^{''} \beta}{1-\beta}. \label{eq:kahanuv1}.
% \end{align}

% \sqrt{2}\sin \theta(u_1,v_1) \leq \frac{\sqrt{2} \beta}{1-\beta}.$

 Next, we argue that $\hat{u}$ is close to $u$, and hence, close to $v_1$ by the triangle inequality.
 By \prettyref{eq:u0}, the initial value $u^0 \in \reals^{\tilde C_1}$ satisfies $|\langle u, u^0 \rangle| \geq  ( 2 \sqrt{n_1 \log n_1} )^{-1}$ with high probability.
% By the choice of the initial vector $u^0$,
%we can write $u^0 = z / \|z\|_2$ for a standard normal vector $z \in \reals^m$.
%By the tail bounds for Chi-squared distributions, it follows that $\|z\|_2 \le 2 \sqrt{m}$
%with high probability. For $u$ fixed,
%the random variable $\langle u, z \rangle \sim \calN(0,1)$ and thus
%with high probability, $ |\langle u, z \rangle|^2 \geq 1/\log n$, and hence $|\langle u, u^0 \rangle| ^2 \geq  n^{-c}$ for any $c >1.$
%By the choice of the initial vector $u^0$,  for $u$ fixed, the random variable $\langle u, u^0\rangle$
%has an approximately Gaussian density with mean zero and variance $1/|\tilde{C}_1|  \ge n_1^{-1}$.
%Thus, with high probability,  $|\langle u, u_0 \rangle| ^2 \geq  n_1^{-c}$ for any $c >1.$
By Weyl's inequality,
%the singular values are changed by no more than the spectral norm of the perturbation matrix.
%Therefore,
the largest singular value of  $\tilde{W}$  is at least $\mu \sqrt{K_1K_2}(1-\epsilon) - (\sqrt{n_1} + \sqrt{n}_2) \beta$, and the
other singular values are at most $ (\sqrt{n_1} + \sqrt{n}_2 )\beta $.
In view of \prettyref{eq:c0}, $\frac{1-\epsilon}{c_0 \beta} - 1 > 1$ for all $\epsilon < \epsilon_0$, where $\epsilon_0>0$ depends only on $c_0$.
Let $\lambda_1$ and $\lambda_2$ denote the first and second eigenvalue of  $\tilde{W} \tilde{W}^\top$ in absolute value, respectively.
Let $r=\lambda_2/\lambda_1$. Then  $r \le (\frac{c_0 \beta}{1 - \epsilon - c_0 \beta})^2$.
%Thus, the ratio between the leading
%singular value and any other singular value is at least $\frac{1-\beta}{\beta},$
%from which we find for positive even values of $t:$
%$$
%\langle u^t , u \rangle^2   \geq  1 - \frac{1}{\langle u, u^0 \rangle ^2}\left( \frac{\beta}{1-\beta}\right)^{2t}.
%$$
Since for even $t$, $u^t= ( \tilde{W} \tilde{W} ^\top)^{t/2} u^0/ \|  ( \tilde{W} \tilde{W}^\top )^{t/2} u^0 \|$, the same analysis of power iteration that leads to \prettyref{eq:power} yields
\begin{align*}
\langle u^t , u \rangle^2 \ge 1- \frac{ r^{t}  }{ \langle u, u^0 \rangle^2 -2 | \langle u, u^0 \rangle| r^{t/2} }.
\end{align*}
%it follows that
%$$
%u^t=  \frac{ \langle u, u^0 \rangle u + y }{ \| \langle u, u^0 \rangle u + y \| },
%$$
%where $y$ is a vector in $\reals^m$ such that $\|y\|  \le r^{t/2}$. Hence,
%\begin{align*}
%\langle u^t , u \rangle^2 & = \frac{  \left( \langle u, u^0 \rangle + \langle y, u \rangle \right)^2 }{ \| \langle u, u^0 \rangle u + y \| ^2 }  =  1 + \frac{ \langle y, u \rangle^2 - \|y \| ^2  }{\| \langle u, u^0 \rangle u + y \|^2  } \\
%& \ge 1- \frac{  \|y \|_2^2  }{ \langle u, u^0 \rangle^2 -2 | \langle u, u^0 \rangle| \| y\|  } \ge 1- \frac{ r^{t}  }{ \langle u, u^0 \rangle^2 -2 | \langle u, u^0 \rangle| r^{t/2} }.
%\end{align*}
Since $\hat{u} =u^{2 \lceil s^* \log n \rceil}$ and
$
s^*=(\log\frac{ 1 - \epsilon - c_0 \beta}{ c_0 \beta})^{-1}
$,
we have $r^{\lceil s^* \log n_1 \rceil } \le n_1^{-2}$ and thus $|\langle \hat{u} , u \rangle^2 |  \geq 1 -  n_1^{-1}$ and consequently,
$\|uu^\top - \hat{u} (\hat{u})^\top \|_{\rm F}^2 = 2 - 2 \iprod{u}{\hat{u}}^2 \leq n_1^{-1}$.
%Notice that
%$$
%\min\{\|\hat{u} - v_1 \|^2, \|\hat{u} + v_1 \|^2\} = 2 - 2 |\iprod{\hat{u}}{v_1}|  \leq \|\hat{u} (\hat{u})^\top - v_1v_1^\top \|_{\rm F}^2.
%$$
Similar to \prettyref{eq:proj}, applying \prettyref{eq:kahanuv1} and the triangle inequality, we obtain
\begin{equation}
	\min\{\|\hat{u} - v_1\| , \|\hat{u} + v_1\| \} \leq \|\hat{u} (\hat{u})^\top - v_1v_1^\top \|_{\rm F} \leq \frac{ 2 \sqrt{2} c_0 \beta}{1-\epsilon} + n_1^{-1/2} \le \frac{3 c_0 \beta}{1-\epsilon} \triangleq \beta_0.
\end{equation}
By the same argument that proves \prettyref{eq:CCeta}, we have
$
\limsup_{n\rightarrow\infty} |C_1^* \triangle \hat C_1 |   /{K_1 } \leq  2\epsilon + 8\beta_0^2 \le \eta(\epsilon)
$
with $\eta$ defined in \prettyref{eq:defeta_bicluster}, completing the proof.
\end{proof}

\begin{remark}
Condition \eqref{eq:spectral_condition} implies that
$\mu^2 K_1 K_2/n_1 =\Omega(1)$, which in turn implies
that \eqref{eq:suff_cond_voting_weak} holds in the regime $K_1=o(n_1)$.
Hence, under \eqref{eq:spectral_condition}, either
both $C^*_1$ and $C^*_2$ can be weakly recovered or neither of them can be weakly recovered.
\end{remark}

 \appendix
\section{Row-wise thresholding}\label{app:degreethreshold}
We describe a simple thresholding procedure for recovering $C^*$.
Let $R_i = \sum_j  W_{i,j}$ for $i\in [n].$
Then $R_i  \sim {\cal N} ( K\mu, n)$ if
$i\in C^*$ and $R_i  \sim {\cal N} (0, n)$ if  $i \notin C^\ast$.
Let  $\hat{C} = \left\{ i \in [n] : R_i \geq \frac{K \mu}{2}\right\}.$
Then  $\expect{ |\hat C  \triangle C^*| } = n Q\left(\frac{K \mu}{2\sqrt{n}}\right) .$
Recall that
$\lambda= \frac{K^2 \mu^2}{n}$.
Hence, if
\begin{equation}
\lambda =\omega\pth{ \log \frac{n}{K}},
	\label{eq:thresholding-weak}
\end{equation}
 then we have
$\eexpect{ |\hat C  \triangle C^*| } = o(K)$ and hence achieved weak recovery.
In the regime $K\asymp n \asymp (n-K), $
$\lambda =\omega( \log \frac{n}{K})$ is equivalent to $\lambda \to \infty$,
which is also equivalent to $K\mu^2 \to \infty$ and coincides with the necessary and sufficient condition for the information-theoretic possibility of weak
recovery in this regime \cite[Theorem 2]{HajekWuXu_one_info_lim15}.
(If instead $n-K = o(n)$, weak recovery is trivially provided by $\hat C = [n]$.)
Thus, row-wise thresholding provides weak recovery in the regime $K\asymp n \asymp (n-K)$ whenever information
theoretically  possible.  Under the information-theoretic condition \eqref{eq:submat-mle-suff},
an algorithm attaining exact recovery can be built using row-wise thresholding
for weak recovery followed by voting,  as in  \prettyref{alg:exactrecovery}
(see  \cite[Theorem 4]{HajekWuXu_one_info_lim15} and its proof).
In the regime $\frac{n}{K} \log \frac{n}{K} = o(\log n)$, or equivalently $ K=\omega( n \log \log n/ \log n)$,
condition  \eqref{eq:submat-mle-suff} implies that $\lambda =\omega( \log \frac{n}{K})$, and
hence in this regime  exact recovery can be attained in linear time $O(n^2)$
whenever information theoretically  possible.

\section{Proofs of Theorems \ref{thm:bi_cluster_weak} and \ref{thm:bi_cluster_exact}}
	\label{app:bicluter-mle}
	
	In the proofs below we use the following notation.  We write
$\textsf{p}_e(\pi_1, s^2)$ to denote the
minimal average error probability for testing ${\cal N}(\mu_1, \sigma^2)$ versus ${\cal N}(\mu_0, \sigma^2)$
with priors $\pi_1$ and $1-\pi_1$, where $\mu_1 \geq \mu_0$
%average error probability for the
%MAP decision rule for a binary hypothesis testing problem with priors $\pi_1$ and $1-\pi_1$,
%distributions ${\cal N}(\mu_1, \sigma^2)$ vs. ${\cal N}(\mu_0, \sigma^2)$
%with $\mu_1 \geq \mu_0$,
and $s^2 = \frac{(\mu_0-\mu_1)^2}{\sigma^2}.$  That is,
$$\textsf{p}_e(\pi_1, s^2) \triangleq \min_\gamma  \{  \pi_1 Q(s-\gamma) + (1-\pi_1) Q(\gamma)  \}.$$

\begin{proof}[Proof of  \prettyref{thm:bi_cluster_weak}]
We defer the proof of (i) to the end and begin with the proof of (ii).
Column sum thresholding for recovery of $C_2^*$ consists of comparing $\sum_i  W_{i,j}$ to
a threshold for each $j \in [n_2]$ to estimate whether $j\in C_2^*.$
This sum has the ${\cal N}(K_1 \mu, n_1)$ distribution if $j \in C_2^*$, which has prior probability $K_2/n_2$,
and the sum has the ${\cal N}(0, n_1)$ distribution otherwise.   The mean number of classification
errors divided by $K_2$ is given by $(n_2/K_2)\textsf{p}_e(K_2/n_2, K_1^2\mu^2   /n_1),$  which
converges to zero under \eqref{eq:weak_C2_line_sum}.   This proves (ii).

The proof of (iii) is similar, although it involves the method of successive withholding in a way
similar to that in \prettyref{alg:exactrecovery}.    The set $[n_1]$ is partitioned into
sets,  $S_1, \ldots  , S_{1/\delta}$  of size $n_1\delta.$     There are $1/\delta$ rounds of the algorithm,
and indices in $S_{\ell}$ are classified in the $\ell^{th}$ round.    For the
$\ell^{th}$ round,   by assumption, given $\epsilon > 0,$   there exists an estimator $\hat{C}_2$
based on observation of $W$ with the rows indexed by $S_{\ell}$ hidden such that
$|\hat{C}_2 \Delta C^*_2| \le \epsilon K_2$
with high probability. Then the voting procedure estimates whether $i \in C_1^*$ for
each $i \in S_{\ell}$  by comparing $\sum_{j \in \hat{C}_2} W_{i,j}$ to a threshold for each $i \in [n_1]$.
This sum has approximately the $\calN(K_2 \mu, K_2)$ distribution if $i \in C_1^*$ and $\calN(0,K_2)$
distribution otherwise ; the discrepancy can be made sufficiently small by choosing $\epsilon_2$ to be small (See
\cite[Lemma 9]{HajekWuXu_one_info_lim15} for a proof).
Thus, the mean number of classification
errors divided by $K_1$ is well approximated by $(n_1/K_1) \textsf{p}_e(K_1/n_1, K_2 \mu^2  )$, which converges to zero
under \eqref{eq:suff_cond_voting_weak}, completing the proof of (iii).

Now to the proof of (i).    The proof of sufficiency for weak recovery is closely based on the proof of
sufficiency for exact recovery by the MLE given in \cite{Butucea2013sharp}; the reader is referred to
\cite{Butucea2013sharp} for the notation used in this paragraph.
The proof  in  \cite{Butucea2013sharp} is divided into two sections.   In our terminology,
\cite[Section 3.1]{Butucea2013sharp} establishes the weak recovery of $C_1^*$ and $C_2^*$ by the MLE
under the assumptions  \eqref{eq:B_suff0}, \eqref{eq:submat_vote_1}, and \eqref{eq:submat_vote_2}.   However,
 the assumption  \eqref{eq:submat_vote_1}  (and similarly,  \eqref{eq:submat_vote_2}) is used
 in only one place in the proof, namely for bounding the terms $T_{1,km}$ defined therein.
 We explain here why  \eqref{eq:B_suff0} alone is sufficient for the proof  of weak recovery.
 Condition  \eqref{eq:B_suff0} implies condition  \eqref{eq:suff_cond_voting_weak}, which, in the notation\footnote{The notation of  \cite{Butucea2013sharp} is mapped to ours as $N \to n_1$, $M \to n_2$, $n \to K_1$, $m \to K_2$, and $a \to \mu.$} of
 \cite{Butucea2013sharp}, implies that there exists some sufficiently small $\alpha>0$ such that
 $$
 \frac{a^2 m }{2\log(N/n)} \geq 1+ \alpha.
 $$
So  \cite[(3.4)]{Butucea2013sharp} can be replaced as: there exist some sufficiently small $\delta_1>0$ and $\alpha_1>0$ such that
 $$
  \frac{(1-\delta_1)^2}{2}a^2 m \geq (1+\alpha_1)\log(N/n) \geq (1+\alpha_1)\log\left( \frac{\delta(N-n)}{n-k}\right) ,
 $$
 where we use the assumption $0\leq k < (1-\delta)n$, or $n-k > \delta n.$    Thus, for large enough $n,$
\begin{align*}
 T_{1,km}&  \leq \exp\left(  - \frac{\delta n \alpha_1}{2}   \left(      \log\left(\frac{N-n}{n-k}\right)   \right) \right)
 \leq \exp\left(  - \frac{\delta n \alpha_1}{2}      \log\left(\frac{N-n}{n}\right)    \right) = o(1/n),
\end{align*}
from which the desired conclusion,  $\sum_{k:(n-k) > \delta n} T_{1,km} = o(1),$  follows.
This completes the proof of sufficiency of  \eqref{eq:B_suff0} for weak recovery of both $C_1^*$ and $C_2^*$,
and marks the end of our use of notation from  \cite{Butucea2013sharp}.

The rate distortion argument  used in the proof of  \cite[Theorem 5]{HajekWuXu_one_info_lim15} shows
that \eqref{eq.B_nec} must hold if $C_1^*$ and $C_2^*$ are both weakly recoverable.
\end{proof}

\begin{proof}[Proof of  \prettyref{thm:bi_cluster_exact}]
The proof follows along the lines of the proofs of \prettyref{thm:bi_cluster_weak} parts (ii) and (iii).   The key
calculation for part (i) is that  \eqref{eq:submat_vote_1} implies that  $n_1\textsf{p}_e(K_1/n_1, K_2 \mu^2  )\to 0;$
and the key calculation for  part (ii) is that \eqref{eq:exact_C2_line_sum} implies
that $n_2\textsf{p}_e(K_2/n_2, K_1^2\mu^2   /n_1)\to 0.$
\end{proof}

\bibliographystyle{abbrv}
\bibliography{../graphical_combined}

\end{document}